\documentclass[twoside]{article}
\usepackage[accepted]{aistats2022}

\usepackage[round]{natbib}

\usepackage[utf8]{inputenc} 
\usepackage[T1]{fontenc}    
\usepackage{hyperref, url}       
\usepackage{booktabs}       
\usepackage{amsfonts}       
\usepackage{nicefrac}
\usepackage{microtype}
\usepackage{xcolor}
\usepackage{array}
\usepackage{mdwmath}
\usepackage{mdwtab}
\usepackage{eqparbox}
\usepackage{float}
\usepackage{amsmath}
\usepackage{amssymb}
\usepackage{float}
\usepackage{enumitem}
\usepackage{multirow}
\usepackage{graphicx}
\usepackage{caption}
\usepackage{subcaption}
\usepackage{float}
\usepackage{wrapfig}
\usepackage{stfloats}

\usepackage[section]{algorithm}
\usepackage{algorithmic}
\usepackage{amsthm}

\makeatletter
\newtheorem*{rep@theorem}{\rep@title}
\newcommand{\newreptheorem}[2]{%
\newenvironment{rep#1}[1]{%
 \def\rep@title{#2 \ref{##1}}%
 \begin{rep@theorem}}%
 {\end{rep@theorem}}}
\makeatother

\newtheorem*{D1}{Definition~\ref{def:D1}}
\newtheorem{thm}{Theorem}
\newreptheorem{thm}{Theorem}
\newtheorem{rmk}{Remark}
\newreptheorem{rmk}{Remark}
\newtheorem{lma}{Lemma}
\newreptheorem{lma}{Lemma}
\newtheorem{prop}{Proposition}

\newtheorem{definition}{Definition}

\DeclareMathOperator*{\argmin}{arg\,min}

\newcommand{\norm}[1]{\left\lVert#1\right\rVert}

\numberwithin{equation}{section}

\begin{document}

\twocolumn[

\aistatstitle{Low-Pass Filtering SGD for Recovering Flat Optima in the Deep Learning Optimization Landscape}

\aistatsauthor{Devansh Bisla \And Jing Wang \And  Anna Choromanska }

\aistatsaddress{ \texttt{db3484@nyu.edu} \And  \texttt{jw5665@nyu.edu} \And \texttt{ac5455@nyu.edu} } ]

\begin{abstract}
In this paper, we study the sharpness of a deep learning (DL) loss landscape around local minima in order to reveal systematic mechanisms underlying the generalization abilities of DL models. Our analysis is performed across varying network and optimizer hyper-parameters, and involves a rich family of different sharpness measures. We compare these measures and show that the low-pass filter based measure exhibits the highest correlation with the generalization abilities of DL models, has high robustness to both data and label noise, and furthermore can track the double descent behavior for neural networks. We next derive the optimization algorithm, relying on the low-pass filter (LPF), that actively searches the flat regions in the DL optimization landscape using SGD-like procedure. The update of the proposed algorithm, that we call LPF-SGD, is determined by the gradient of the convolution of the filter kernel with the loss function and can be efficiently computed using MC sampling. We empirically show that our algorithm achieves superior generalization performance compared to the common DL training strategies. On the theoretical front we prove that LPF-SGD converges to a better optimal point with smaller generalization error than SGD.
\end{abstract}

\section{INTRODUCTION}
\label{sec:introduction}

Training a deep network requires finding network parameters that minimize the loss function that is defined as the sum of discrepancies between target data labels and their estimates obtained by the network. This sum is computed over the entire training data set. Due to the multi-layer structure of the network and non-linear nature of the network activation functions, DL loss function is a non-convex function of network parameters. Increasing the number of training data points complicates this function since it increases the number of its summands. Increasing the number of parameters (by adding more layers or expanding the existing ones) increases the complexity of each summand and results in the growth, which can be exponential~\citep{WNonConv}, of the number of critical points of the loss function. A typical use case for DL involves massive (i.e., high-dimensional and/or large) data sets and very large networks (i.e., with billions of parameters), which results in an optimization problem that is heavily non-convex and difficult to analyze.

In order to design efficient optimization algorithms for DL we need to understand which regions of the DL loss landscape lead to good generalization and how to reach them. Existing work ~\citep{feng2020neural,DBLP:journals/corr/ChaudhariCSL17,doi:10.1162/neco.1997.9.1.1,JKABFBS2018ICLRW,DBLP:journals/corr/abs-1805-07898,keskar2016largebatch,DBLP:journals/corr/SagunEGDB17,DBLP:journals/corr/abs-1901-06053,jiang2019fantastic,DBLP:conf/uai/DziugaiteR17,DBLP:journals/corr/abs-1712-05055,61aa9e9cc965421e82d7b7042c61abc8,NIPS2017_7176} shows that properly regularized SGD recovers solutions that generalize well and provides an evidence that these solutions lie in wide valleys of the landscape. In the prior studies~\citep{SAM,DBLP:journals/corr/ChaudhariCSL17}, it was demonstrated that the spectrum of the Hessian at a well-generalizing local minimum of the loss function is often almost flat, i.e., the majority of eigenvalues are close to zero. An intuitive illustration of why flatness potentially leads to better generalization is presented in Figure~\ref{fig:flat}. However, existing studies focus on simplified DL frameworks, small data sets, and/or limited optimization settings and thus are inconclusive. There exists no \textit{comprehensive study} confirming that good generalization abilities of the DL model correlate with the properties of the loss landscape around recovered solutions and some research argues against the existence of this relationship~\citep{SB2017}. The lack of more comprehensive approaches has kept the field from moving away from purely generic DL training methodologies, which are not equipped with any mechanisms allowing them to take advantage of the properties of the non-convex loss landscape underlying the DL optimization problem. 

\begin{figure}
    \centering
    \includegraphics[width=0.3\textwidth]{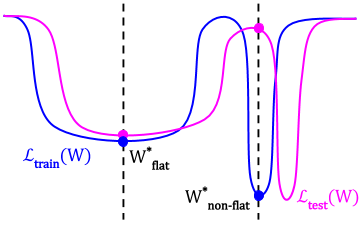}
    \vspace{-0.15in}
    \caption{Consider the train and test loss ($\mathcal{L}_{\text{train}}(W)$ and $\mathcal{L}_{\text{test}}(W)$ resp, where $W$ are model parameters), which have similar shape but are shifted with respect to each other. The local minimum that lies in the flat region ($W^{*}_{\text{flat}}$) admits a similar value of the train and test loss (generalizes well), despite the shift. The local minimum that lies in the narrow region of the landscape ($W^{*}_{\text{non-flat}}$) admits a significantly larger value of the test loss compared to the train loss (generalizes very poorly), due to the shift between them.}
    \label{fig:flat}
    \vspace{-0.31in}
\end{figure}

 In this paper we first empirically confirm that the sharpness of a DL loss landscape around local minima indeed correlates with the generalization abilities of a DL model. In our work we consider network architectures with and without skip connections and batch norm, and with varying widths, and DL optimizer with varying settings of the hyper-parameters, such as the batch size, learning rate, momentum coefficient, and weight decay regularization. To the best of our knowledge, our study is the most comprehensive research study to date of the connection between DL landscape geometry and the generalization abilities of DL models. We then perform comparative studies of different existing sharpness measures in order to identify the most promising one -- the LPF based measure -- that best promotes good performance of the model and that can be efficiently computed. The identified measure is robust to the presence of the data and label noise and furthermore mirrors the double descent behavior of the test accuracy~\citep{Nakkiran2020Deep, belkin2019reconciling}. We incorporate the LPF based sharpness measure to SGD in order to guide the optimization process towards the well-generalizing regions of the loss landscape, and we show that the resulting algorithm is more accurate than the common baselines, including recently proposed state-of-the-art method SAM~\citep{SAM}. Finally, we provide convergence guarantee showing that LPF-SGD recovers a better optimal point with smaller generalization error than SGD.
 
 \textit{What is not new in this paper?} The evidence that well-generalizing solutions lie in wide valleys of the DL loss landscape is not new as well as investigated flatness measures.
 
 \textit{What is new in this paper?} The contributions of our work compared to prior works include: a) the comprehensiveness of the study, b) the direct comparison of several flatness measures and their levels of correlation to the generalization abilities of the model, c) the novel algorithm superior to SGD and prior flatness-aware techniques, and d) its theoretical and empirical analysis.
 
 This paper is organized as follows: Section~\ref{sec:related_work} reviews the literature, Section~\ref{sec:sharp_v_gen} studies the sharpness and generalization properties of the local minima of the DL loss function, Section~\ref{sec:LPF-SGD} introduces the LPF-SGD algorithm that is analyzed in Section~\ref{sec:theory}, and Section~\ref{sec:experiments} shows experiments. Finally, Section~\ref{sec:con} concludes the paper. Supplement contains additional material.

\vspace{-0.1in}
\section{RELATED WORK}\label{sec:related_work}
\vspace{-0.1in}

The loss function in DL is typically optimized using standard or variant forms of SGD~\citep{bottou-98x}, which iteratively updates parameters by taking steps in the anti-gradient direction of the loss function computed for randomly sampled data mini-batches. Various optimization approaches, including adaptive learning rates~\citep{duchi2011adaptive,DBLP:journals/corr/KingmaB14,DBLP:journals/corr/abs-1212-5701,DBLP:journals/corr/abs-1902-09843,DBLP:journals/corr/abs-1804-04235,DBLP:journals/corr/abs-1905-11286}, momentum-based strategies~\citep{journals/nn/Qian99,Nesterov2005SmoothMO,tieleman2012lecture,hardt2016train}, approximate steepest descent methods~\citep{NIPS2015_5797}, continuation methods~\citep{allgower2012numerical,mobahi2015link,gulcehre2016mollifying}, alternatives to standard backpropagation algorithm based on alternating minimization~\citep{carreira2014distributed,zhang2017convergent,Zhang2017,Askari2018,zeng2018block,lau2018proximal,gotmare2018decoupling,DBLP:conf/icml/ChoromanskaCKLR19,taylor2016training,zhang2016efficient} or TargetProp~\citep{le1986learning,yann1987modeles,lecun1988theoretical,lee2015difference,bartunov2018assessing}, as well as architectural modifications~\citep{srivastava2014dropout,ioffe2015batch,cooijmans2016recurrent,salimans2016weight,NIPS2015_5953,Cowen2019}, were proposed to accelerate the training process and improve the generalization performance of the DL models. However there exists no strong evidence that these techniques indeed improve the training of DL models. Moreover, it was shown that well-tuned vanilla SGD often can achieve superior performance compared to mentioned modified techniques~\citep{conf/nips/WilsonRSSR17}. In addition, some prior work~\citep{DBLP:journals/corr/GoodfellowV14,sutskever2013importance} emphasizes the sensitivity of aforementioned optimizers to initial conditions and learning rates. Finally, theoretical approaches to DL optimization have focused on the convergence properties of vanilla SGD and consider only shallow models~\citep{pmlr-v70-brutzkus17a,NIPS2017_6662,DBLP:journals/corr/ZhongS0BD17,Tian2017AnAF,DBLP:journals/corr/abs-1710-10174,Li2018AlgorithmicRI} or deep, but linear ones~\citep{DBLP:journals/corr/abs-1810-02281}. These results do not hold for DL models used in practice. 

The shape of the DL loss function is poorly understood. Consequently, the aforementioned generic, rather than landscape-adaptive, optimization strategies are commonly chosen to train DL models, despite their potentially poor fit to the underlying non-convex DL optimization problem. Most of what we know about the DL loss landscape is either based on unrealistic assumptions and/or holds only for shallow (two-layer) networks. The existing literature emphasizes i) the proliferation of saddle points~\citep{NIPS2014_5486,Baldi:1989:NNP:70359.70362,DBLP:journals/corr/SaxeMG13} (including degenerate or hard to escape ones~\citep{pmlr-v40-Ge15,anandkumar2016efficient,NIPS2014_5486}), ii) the equivalency of some local minima~\citep{chaudhari2015trivializing,haeffele2015global,janzamin2015beating,NIPS2016_6112,soudry2016no,JoanTopology,Nguyen2017TheLS,DBLP:journals/corr/abs-1712-04741,haeffele2015global,pmlr-v80-du18a,Safran:2016:QIB:3045390.3045473,DBLP:journals/corr/abs-1712-08968,DBLP:journals/corr/HardtM16,pmlr-v70-ge17a,DBLP:journals/corr/YunSJ17,pmlr-v80-draxler18a,JoanTopology,JoanPS,DBLP:conf/aistats/ChoromanskaHMAL15, DBLP:conf/colt/ChoromanskaLA15}, iii) the existence of a large number of isolated minima and few dense regions with lots of minima close to each other~\citep{baldassi2015subdominant,baldassi2016unreasonable,baldassi2016multilevel} (shown for shallow networks), and iv) the existence of global and local minima yielding models with different generalization performance~\citep{keskar2016largebatch}. Furthermore, it was demonstrated that becoming stuck in poor minima is a major problem only for smaller networks~\citep{DBLP:journals/corr/GoodfellowV14}. 

There exists only a few approaches that aim at adapting the DL optimization strategy to the properties of the loss landscape and encouraging the recovery of flat optima. They can be summarized as i) regularization methods that regularize gradient descent strategy with sharpness measures such as the Minimum Description Length~\citep{doi:10.1162/neco.1997.9.1.1}, local entropy~\citep{DBLP:journals/corr/ChaudhariCSL17}, or variants of $\epsilon$-sharpness~\citep{SAM,keskar2016largebatch}, ii) surrogate methods that evolve the objective function according to the diffusion equation~\citep{DBLP:journals/corr/Mobahi16}, iii) averaging strategies that average model weights across training epochs~\citep{61aa9e9cc965421e82d7b7042c61abc8, cha2021swad}, and iv) smoothing strategies that smoothens the loss landscape by introducing noise in the model weights and average model parameters across multiple workers run in parallel~\citep{DBLP:journals/corr/abs-1805-07898, haruki2019gradient, lin2020extrapolation} (such methods focus on distributed training of DL models with an extremely large batch size - a setting where SGD and its variants struggle).

\vspace{-0.1in}
\section{INTERPLAY BETWEEN SHARPNESS AND GENERALIZATION}
\label{sec:sharp_v_gen}
\vspace{-0.1in}

In this section we study the correlation between network generalization and the sharpness of the local minima of the DL loss function. Without loss of generality we consider balanced networks~\citep{NIPS2015_5797}, where the norms of incoming weights to different units are roughly the same (for details see Supplement, Section~\ref{sec:bal}). This assumption is necessary since the properties of the loss landscape of two equivalent (realizing the same function) imbalanced networks can be radically different (e.g., their landscapes can be stretched along different directions), which would lead to inconsistencies in the analysis. The importance of balancing the network is overlooked in existing studies.

\textbf{All codes for experiments presented in this section as well as Section~\ref{sec:experiments} are available at} \def\UrlFont{\bfseries}\url{https://github.com/devansh20la/LPF-SGD}.
\def\UrlFont{\normalseries}

We consider a broad family of different measures of sharpness: $\epsilon$-sharpness~\citep{keskar2016largebatch}, PAC-Bayes measure~\citep{jiang2019fantastic} ($\mu_{PAC-Bayes}$), Fisher Rao Norm~\citep{liang2019fisher} (FRN), gradient of the local entropy~\citep{DBLP:journals/corr/ChaudhariCSL17} ($\mu_{\text{LE}}$), Shannon entropy~\citep{pereyra2017regularizing} ($\mu_{entropy}$), Hessian-based measures~\citep{maddox2020rethinking, mackay1992bayesian} (the Frobenius norm of the Hessian ($||H||_{F}$), trace of the Hessian (Trace ($H$), the largest eigenvalue of the Hessian ($\lambda_{max}(H))$, and the effective dimensionality of the Hessian), and the low-pass filter based measure (LPF). Decreasing the value of any of these measures leads to flatter optima. The definitions of these measures as well as the algorithms numerically computing them are provided in the Supplement (Section~\ref{sec:FM}). Similarly, the Supplement (Section~\ref{sec:Sensitivity}) contains the analysis of the sensitivity of these measures to the changes in the curvature of the synthetically generated landscapes (i.e., landscapes, where the curvature of the loss function is known). Below we only provide the definition of the LPF measure as this one constitutes the foundation of our algorithm. 

\begin{definition}[LPF]\label{def:lpf}
Let $K\sim \mathcal{N}(0,\sigma^2 I)$ be a kernel of a Gaussian filter. LPF based sharpness measure at solution $\theta^*$ is defined as the convolution of the loss function with the Gaussian filter computed at $\theta^*$:

\vspace{-0.3in}
\begin{align}
    \label{eq:sharp_meas_9}
    (L \circledast K) (\theta^*) &= \int L(\theta^*-\tau) K(\tau) d\tau.
\end{align}
\vspace{-0.3in}
\end{definition}

Next, we aim at revealing whether there exists any correlation between sharpness of local minima and the performance of the deep networks and, if so, which flatness measure exhibits the highest correlation with the model generalization.

\vspace{-0.1in}
\subsection{Sharpness versus generalization}\label{sec:sharpness_vs_generalization}
\vspace{-0.05in}

\begin{table}
\centering
    \begin{tabular}{l|l}
        \hline
        Hyper-parameter & Settings  \\ \hline
        Momentum & [0.0, 0.5, 0.9] \\
        Width & [32, 48, 64] \\
        Weight decay & [$0.0$, $1e^{-4}$, $5e^{-4}$] \\
        Learning rate & [$5e^{-3}$, $7.5e^{-3}$, $1e^{-2}$] \\
        Batch size & [32, 128, 512] \\
        Skip connection & [False, True] \\
        Batch norm & [False, True] \\\hline
    \end{tabular}
    \vspace{-0.12in}
    \caption{Choices of hyper-parameters. }
    \label{tab:hyper_resnet}
    \vspace{-0.25in}
\end{table}

We train $2916$ ResNet18~\citep{he2016deep} models on publicly available CIFAR-10~\citep{CIFAR10} data set by varying network architecture (width, skip connections (skip), and batch norm (bn)), optimizer hyper-parameters (momentum coefficient(mom), learning rate (lr), batch size (batch)), and regularization (weight decay (wd)) according to the Table~\ref{tab:hyper_resnet}. For each set of hyper-parameters, we train three models with different seeds. We follow \citep{dziugaite2020search, jiang2019fantastic} in terms of stopping criteria, as we focus our analysis on practical models, and train each model until the cross entropy loss reached the value of $\approx 0.01$, which corresponds to $\approx 99\%$ training accuracy. Any model that has not reached this threshold was discarded from further analysis. Note that the other stopping criteria, for example these based on \#epochs/\#iterations, lead to under/over trained models as some models train faster than others. More details of training can be found in the Supplement (Section~\ref{sec:training_details_sharp_vs_gen}). At convergence, when the model approaches the final weights, we compute the Kendall rank correlation coefficient between generalization gap and sharpness measures (averaged over the seeds). Table~\ref{tab:gen_corr} captures the results along with $95\%$ confidence interval. Note that Table~\ref{tab:gen_corr} is extremely compute intensive ($\approx\!\!9000$ GPU hours), which makes it prohibitive to obtain similar results across multiple data sets and models. In Figure~\ref{fig:gen_corr} we show a scatter plot of normalized LPF based measure (with the highest overall correlation) and normalized LE based measure (with the lowest overall correlation) with the corresponding generalization gap. Note that our analysis includes models with wide-ranging generalization gap.

\begin{figure}[!ht]
    \vspace{-0.15in}
    \centering
    \includegraphics[width=0.24\textwidth]{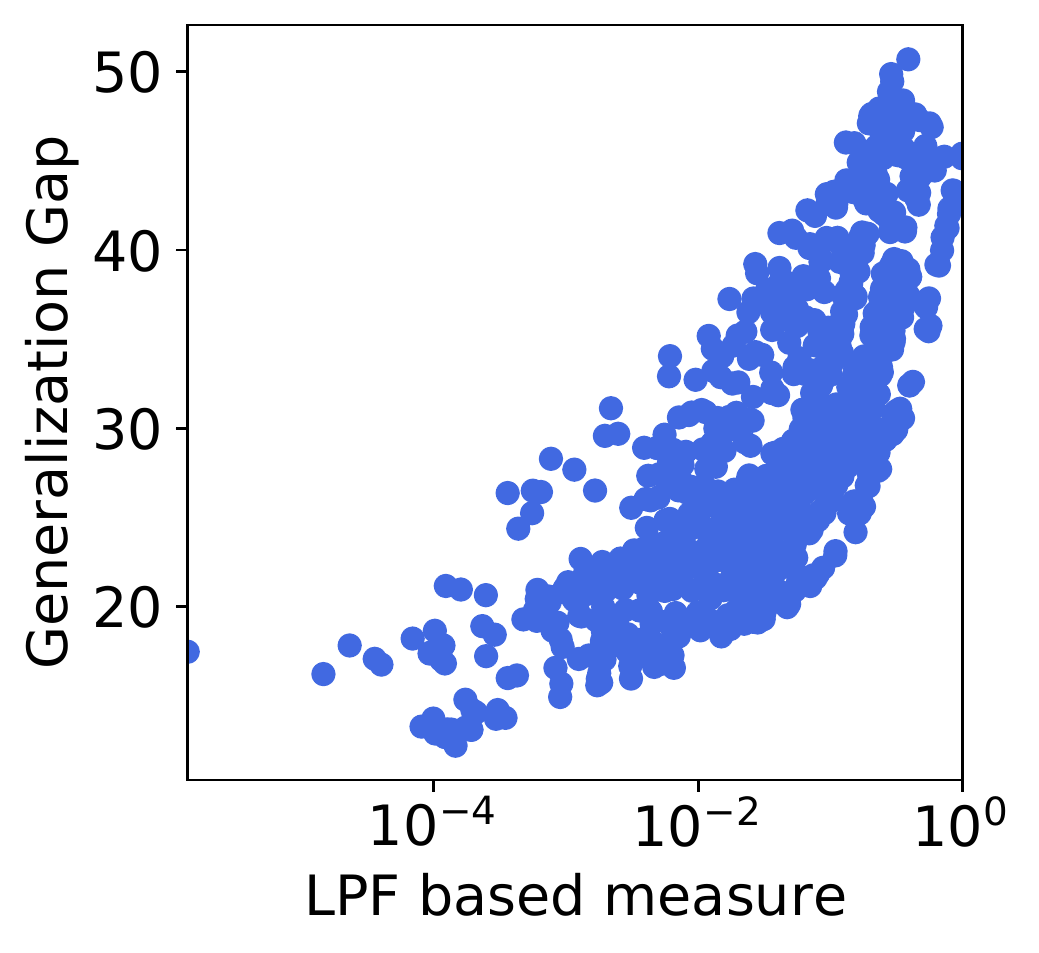}
    \hspace{-0.05in}\includegraphics[width=0.24\textwidth]{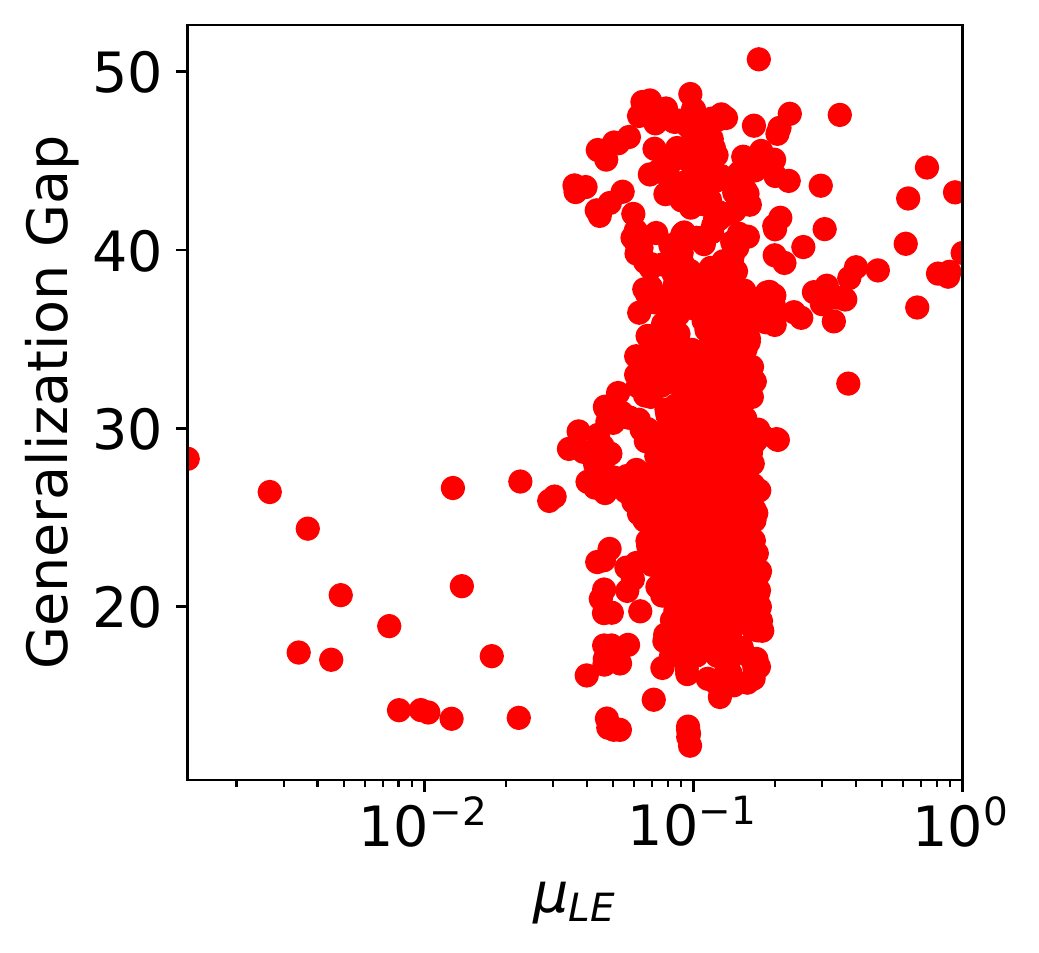}
    \vspace{-0.35in}
    \caption{Normalized LPF based measure (with the highest overall correlation) (left) and normalized LE based measure (with the lowest overall correlation) (right) versus generalization gap.}
    \label{fig:gen_corr}
    \vspace{-0.1in}
\end{figure}

Table~\ref{tab:gen_corr} and Figure~\ref{fig:gen_corr} clearly reveal that there exist sharpness measures for which one can observe the correlation between the generalization gap and the sharpness of the local minima (Kendall rank correlation coefficient above $0.35$ is considered to indicate strong correlation~\citep{CC}). We observe that LPF based measure outperforms other sharpness measures in terms of tracking the generalization behavior of the network when we vary weight decay, batch size and skip connections. It also shows strong positive correlation for both momentum and learning rate. It is however negatively correlated with generalization gap when varying the width of the network and batch normalization. Note that when varying batch normalization we observe only a weak correlation of all flatness measures with the generalization gap. In terms of varying the width, we observe in certain cases, including LPF based measure, strong negative correlation. This can be intuitively explained by the fact that network with higher capacity (higher width) and sharp minima can still outperform smaller network, even if its landscape is more flat. Finally, note that the table reporting the Kendall rank correlation coefficient between hyper-parameters (columns) and various sharpness measures (rows) (see Table~\ref{tab:sharp_v_hyperparam}) is deferred to the Supplement (Section~\ref{sec:correxp}). This table provides guidelines on how to design or modify deep architectures to populate the DL loss landscape with flat minima - this research direction lies beyond the scope of this work and will be investigated in the future.
\begin{table*}[t]
    \centering
    \setlength\tabcolsep{0.1pt}
    \renewcommand{\arraystretch}{1.0}
    \begin{tabular}{|p{2cm}|c|c|c|c|c|c|c|c|}
        \hline
        Measure                & mom & width & wd & lr & batch & skip & bn & Overall \\ \hline
        $\lambda_{\max}(H)$ & $0.882_{\pm0.03}$ & $0.070_{\pm0.07}$ & $0.299_{\pm0.07}$ & $0.598_{\pm0.06}$ & $0.975_{\pm0.01}$ & $-0.148_{\pm0.09}$ & $-0.089_{\pm0.09}$ & $0.415_{\pm0.00}$   \\
        $\|H\|_F$ & $0.925_{\pm0.02}$ & $0.004_{\pm0.07}$ & $0.304_{\pm0.07}$ & $0.763_{\pm0.05}$ & $0.992_{\pm0.01}$ & $0.077_{\pm0.09}$ & $-0.094_{\pm0.09}$ & $0.481_{\pm0.04}$  \\
        Trace (H) & $0.938_{\pm0.02}$ & $0.158_{\pm0.08}$ & $0.328_{\pm0.07}$ & $0.665_{\pm0.05}$ & $0.981_{\pm0.01}$ & $0.229_{\pm0.09}$ & $-0.081_{\pm0.09}$ & $0.538_{\pm0.00}$ \\
        $d_{eff}$ & $0.352_{\pm0.07}$ & $0.080_{\pm0.07}$ & $0.188_{\pm0.07}$ & $0.145_{\pm0.07}$ & $0.342_{\pm0.07}$ & $0.195_{\pm0.09}$ & $\mathbf{0.043_{\pm0.09}}$ & $0.326_{\pm0.00}$ \\
        $\epsilon$-sharpness & $0.776_{\pm0.04}$ & $-0.130_{\pm0.07}$ & $0.304_{\pm0.07}$ & $0.684_{\pm0.05}$ & $0.961_{\pm0.02}$ & $-0.239_{\pm0.09}$ & $-0.098_{\pm0.09}$ & $0.383_{\pm0.02}$  \\ 
        $\mu_{PAC-Bayes}$ & $\mathbf{0.994_{\pm0.01}}$ & $-0.836_{\pm0.04}$ & $0.408_{\pm0.07}$ & $\mathbf{0.900_{\pm0.03}}$ & $0.994_{\pm0.01}$ & $0.389_{\pm0.08}$ & $-0.094_{\pm0.09}$ & $0.466_{\pm0.03}$ \\
        FRN & $0.812_{\pm0.04}$ & $\mathbf{0.238_{\pm0.07}}$ & $0.161_{\pm0.07}$ & $0.495_{\pm0.06}$ & $0.849_{\pm0.04}$ & $-0.073_{\pm0.09}$ & $-0.098_{\pm0.09}$ & $0.270_{\pm0.02}$   \\ 
        $\mu_{entropy}$ & $0.718_{\pm0.05}$ & $0.182_{\pm0.07}$ & $-0.072_{\pm0.07}$ & $0.327_{\pm0.07}$ & $0.716_{\pm0.05}$ & $-0.326_{\pm0.08}$ & $-0.123_{\pm0.09}$ & $0.296_{\pm0.00}$  \\ 
        $\mu_{LE}$ & $0.160_{\pm0.08}$ & $-0.817_{\pm0.04}$ & $0.138_{\pm0.07}$ & $0.106_{\pm0.07}$ & $0.115_{\pm0.07}$ & $-0.050_{\pm0.09}$ & $-0.076_{\pm0.09}$ & $-0.064_{\pm0.04}$  \\ 
        LPF & $0.990_{\pm0.01}$ & $-0.762_{\pm0.05}$ & $\mathbf{0.428_{\pm0.07}}$ & $0.867_{\pm0.04}$ & $\mathbf{0.996_{\pm0.01}}$ & $\mathbf{0.655_{\pm0.07}}$ & $-0.085_{\pm0.09}$ & $\mathbf{0.606_{\pm 0.03}}$  \\ \hline
    \end{tabular}
    \vspace{-0.1in}
    \caption{Kendall rank correlation coefficient between generalization gap and sharpness measure (rows) averaged over set of models trained by varying only a single hyper-parameter (cols), along with $95\%$ confidence interval. Larger value indicates better correlation.}
    \label{tab:gen_corr}
    \vspace{-0.2in}
\end{table*}

\begin{figure*}[!ht]
    \centering
    \includegraphics[width=0.95\textwidth]{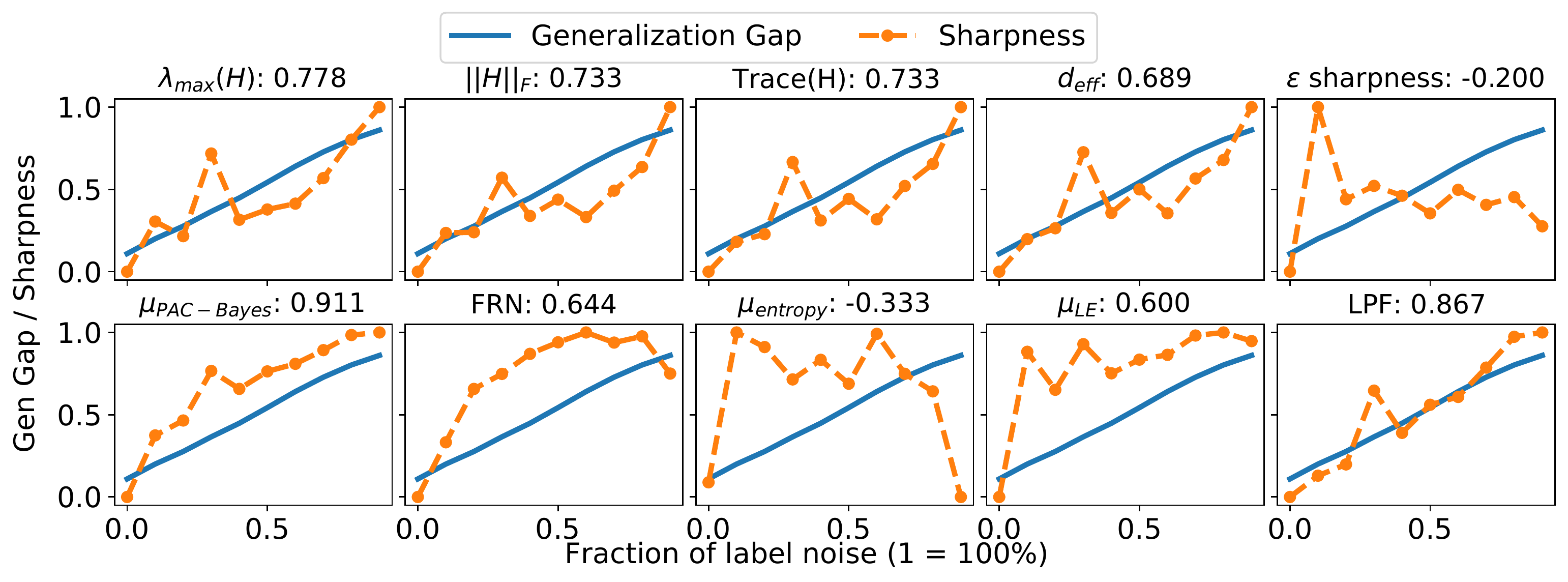}    
    \vspace{-0.15in}
    \caption{Normalized sharpness measures and generalization gap for varying levels of label noise. Kendall rank correlation coefficient between generalization gap and sharpness is provided in the titles above the figures.}
    \label{fig:noise_corr2}
    \vspace{-0.2in}
\end{figure*}

Next we will investigate how robust different flatness measures are to the data and label noise, and whether they track the double descent phenomenon.

\vspace{-0.1in}
\subsection{Sharpness versus generalization under data and label noise}
\label{sec:sharpness_vs_generalization_under_noise}
\vspace{-0.1in}

Large neural networks can easily achieve zero training error~\citep{zhang2016understanding} in the presence of data or label noise while generalizing very poorly. In this section, we evaluate the capability of sharpness measures to explain generalization performance of deep networks in presence of such noise. For experiments with noisy data, we corrupt each input image from publicly available CIFAR-10 data set with Gaussian noise with $0$ mean and $\sigma I$ standard deviation and vary $\sigma$ in the range $[0,4]$. For experiments with noisy labels, we flip the true label of each input sample with probability $\alpha$ and vary $\alpha$ in the range $[0,1]$. The flipped label is chosen uniformly at random from the set of all classes, excluding the true class. We train $10$ ResNet18 models with different levels of label noise and $20$ ResNet18 models with different levels of data noise. Each experiment was repeated $5$ times using different seeds. The experimental details are moved to the Supplement (Section~\ref{sec:noiseexp}).

In Figure~\ref{fig:noise_corr2} we plot the values of normalized sharpness measures and generalization gap for varying level of label noise (experiments for data noise are in Figure~\ref{fig:noise_corr}, Section~\ref{sec:noiseexp} in the Supplement). We also report the Kendall rank correlation coefficient between sharpness measures and generalization gap (averaged over the seeds). We observe that LPF based measure is robust to both the data and label noise (it is the second best measure) and aligns well with the generalization behavior of the network. Finally, Section~\ref{sec:sharpness_vs_generalization_under_double_decent} in the Supplement shows that LPF-SGD based measure tracks the double descent phenomenon most closely from among all the considered sharpness measures.

\vspace{-0.1in}
\section{LPF-SGD}
\label{sec:LPF-SGD}
\vspace{-0.1in}

Based on the results from the previous section, we identify LPF based sharpness measure as the most suitable one to incorporate into the DL optimization strategy to actively search for the flat regions in the DL loss landscape. Guiding the optimization process towards the well-generalizing flat regions of the DL loss landscape using LPF based measure can be done by solving the following optimization problem

\vspace{-0.3in}
\begin{equation}
    \underset{\theta}{\min} (L \circledast K) (\theta) = \underset{\theta}{\min}\int L(\theta - \tau)K(\tau)d\tau,
    \label{eq:main}
\end{equation}
\vspace{-0.25in}

where $K$ is a Gaussian kernel and $L(\theta)$ is the training loss function. We solve the problem given in Equation~\ref{eq:main} using SGD. The gradient of the convolution between the loss function and the Gaussian kernel is

\vspace{-0.3in}
\begin{align}
     \nabla_{\theta}(L \circledast K) (\theta) =& \nabla_{\theta} \int_{-\infty}^{\infty} L(\theta - \tau)K(\tau)d\tau \label{eq:main2} \\ 
     &\approx \frac{1}{M}\sum_{i=0}^{M} \nabla_{\theta}(L(\theta - \tau_{i})), 
     \label{eq:main3}
\end{align}
\vspace{-0.25in}

where $\tau_i$ comes from the same distribution as entries in the filter kernel $K$, and the integral is approximated using the MC method. Thus, to recap, we smooth the loss surface by performing a convolution between the loss function and the Gaussian low pass filter. The multiplicative term in Equation~\ref{eq:main} and Equation~\ref{eq:main2} is the Gaussian probability density function. The action of the convolution is realized by sampling from the filter kernel and accumulating results (Equation~\ref{eq:main3}). The resulting algorithm, that we call LPF-SGD, is in fact extremely simple and its pseudo-code is provided in Algorithm~\ref{alg:lpf-sgd} (for the ease of notation, $\mathcal{N}(0,\gamma\Sigma)$ in the algorithm stands for a sample from a Gaussian distribution). The balancing of model weights in our algorithm is incorporated into the construction of the LPF kernel $K$ (thus there is no need to explicitly balance the network). To be more specific, $K \sim \mathcal{N}(0,\gamma\Sigma)$ and we set the matrix $\Sigma$ to be proportional to the norm of the parameters in each filter of the network, i.e., we set $\Sigma = diag(||\theta_{1}^{t}||, ||\theta_{2}^{t}|| \cdots ||\theta_{k}^{t}||)$, where $\theta_{k}^{t}$ is the weight matrix of the $k^{th}$ filter at iteration $t$ and $\gamma$ is the LPF radius. 

\begin{rmk}
For imbalanced networks, the norm of the weights in one filter can be much larger than the norm of the weights in another filter (Figure~\ref{fig:norm}, Supplement). In such cases isotropic covariance would lead to large perturbations in one filter while only minor in the other one. The parameter-dependent strategy (anisotropic, as in LPF-SGD) is equivalent to using isotropic covariance on a balanced network. We compare isotropic vs anistropic co-variance matrix in section ~\ref{sec:AS} of the Supplement and show that anisotropic approach is superior (Table~\ref{tab:iso_aniso}) in terms of final performance.
\end{rmk}

Finally, we increase $\gamma$ during network training to progressively increase the area of the loss surface that is explored at each gradient update. This is done according to the following rule:

\vspace{-0.2in}
\begin{equation}
    \gamma_{t} = \gamma_{0}(\alpha/2 (-\cos(t \pi/T) + 1) + 1),
    \label{eq:gamma_policy}
\end{equation}
\vspace{-0.28in}

where $T$ is total number of iterations (epochs * gradient updates per epoch) and $\alpha$ is set such that $\gamma_{T} = (\alpha+1)*\gamma_0$. This policy can be justified by the fact that the more you progress with the training, the more you care to recover flat regions in the loss landscape. 

\setlength{\intextsep}{0pt} 
\setlength{\textfloatsep}{0.1cm}
\setlength{\floatsep}{0cm}
\begin{algorithm}
    \centering
    \caption{LPF-SGD}\label{alg:lpf-sgd}
    \begin{algorithmic}
    \STATE \textbf{Inputs:} $\theta^{t}$: weights \\ 
    \STATE \textbf{Hyperpar:} $\gamma$: filter radius, $M$: $\#$ MC iterations \\
    \WHILE{not converged}
        \STATE Sample data batch $B = (x_{1}, y_{1} ) \cdots (x_{b}, y_{b})$ \\
        \STATE Split batch into M splits $B = \{ B_{1} \cdots B_{M} \}$ \\
        \STATE $g \leftarrow 0$ \\
        \FOR{i=1 to M}
            \STATE $\Sigma = diag(||\theta_{i}^{t}||_{i=1}^{k})$
            \STATE $g = g + \frac{1}{M}\nabla_{\theta}L(B_{i}, \theta^{t} + \mathcal{N}(0,\gamma\Sigma))$
        \ENDFOR
        \STATE $\theta^{t+1} = \theta^{t} - \eta*g$ // Update weights
    \ENDWHILE
    \end{algorithmic}
\end{algorithm}

\section{THEORY}
\label{sec:theory}
\vspace{-0.1in}

In this section we theoretically analyze LPF-SGD and show that it converges to the optimal point with smaller generalization gap than SGD. All proofs are deferred to the Supplement (Section~\ref{sec:proofs}). Existing work~\citep{bousquet2001algorithmic,bousquet2002stability,hardt2016train,bassily2020stability,DBLP:journals/corr/abs-1809-04564} shows that the generalization error is highly correlated with the Lipschitz continuous and smoothing properties of the objective function. Inspired by the analysis in~\citep{lakshmanan2008decentralized,duchi2012randomized}, we first formally confirm that indeed Gaussian LPF leads to a smoother objective function and then show that SGD run on this smoother function recovers solution with smaller generalization error than in case of the original objective. We use classical approach to analyze the generalization performance and consider standard definition of generalization error that measures how far the empirical loss is from the true loss (for LPF-SGD, the true loss is the smoothed original loss).

Below we analyze the case when $\Sigma=\sigma^2 I$ and defer the analysis for non-scalar $\Sigma$ (i.e., $\Sigma = \gamma diag(||\theta_{1}||, ||\theta_{2}|| \cdots ||\theta_{k}||)$) to the Supplement (Section~\ref{sec:proofs}). For the purpose of our analysis we assume $\Sigma$ is kept fixed during the duration of the algorithm. 

\vspace{-0.1in}
\subsection{Properties of Gaussian LPF}\label{subsec:gau}
\vspace{-0.05in}
Let $\mathcal{S}=\{\xi_1,\cdots,\xi_m\}$ denote the set of $m$ samples drawn i.i.d. from an unknown distribution $\mathcal{D}$. Let $l_o(\theta;\xi)$ denote the loss of the model parametrized by $\theta$ for a specific example $\xi$. Assume $K \sim \mathcal{N}(0,\sigma^2 I)$. Denote the distribution $\mathcal{N}(0,\sigma^2 I)$ as $\mu$. 
Define the convolution of $l_o(\theta;\xi)$ with the Gaussian kernel K as

\vspace{-0.3in}
\begin{align}
    l_{\mu}(\theta;\xi)&=(l_o(\cdot;\xi)\circledast K)(\theta)=\int_{\mathbb{R}^d}l_o(\theta-\tau;\xi)\mu(\tau)d\tau \nonumber \\ 
    &=\mathbb{E}_{Z\sim\mu}[l_o(\theta+Z;\xi)]
    \label{eq:f_mu}
\end{align}
\vspace{-0.35in}

where $Z$ is a random variable satisfying distribution $\mu$. 
The loss function smoothed by the Gaussian LPF, that we denote as $l_{\mu}$, satisfies the following theorem. 
\begin{thm}
\label{thm:gau}
Let $\mu$ be the $\mathcal{N}(0,\sigma^2I_{d\times d})$ distribution. Assume the differentiable loss function $l_o(\theta;\xi):\mathbb{R}^d\rightarrow \mathbb{R}$ is $\alpha$-Lipschitz continuous and $\beta$-smooth with respect to $l_2$-norm. The smoothed loss function $l_\mu(\theta;\xi)$ is defined as (\ref{eq:f_mu}). Then the following properties hold:
\vspace{-0.18in}
\begin{itemize}
  \item[i)]$l_\mu$ is $\alpha$-Lipschitz continuous.
  \vspace{-0.1in}
  \item[ii)]$l_\mu$ is continuously differentiable; moreover, its gradient is $\min\{\frac{\alpha}{\sigma},\beta\}$-Lipschitz continuous, i.e., $f_\mu$ is $\min\{\frac{\alpha}{\sigma},\beta\}$-smooth.
  \vspace{-0.1in}
  \item[iii)] If $l_o$ is convex, $l_o(\theta;\xi)\leq l_\mu(\theta;\xi)\leq l_o(\theta;\xi)+\alpha\sigma\sqrt{d}$.
\end{itemize}
\vspace{-0.14in}
In addition, for each bound i)-iii), there exists a function $l_o$ such that the bound cannot be improved by more than a constant factor.
\end{thm}

Theorem~\ref{thm:gau} confirms that indeed $l_\mu$ is smoother than the original objective $l_o$. At the same time, if $\frac{\alpha}{\sigma}<\beta$, increasing $\sigma$ leads to an increasingly smoother objective function, which is consistent with our intuition.

\vspace{-0.1in}
\subsection{Generalization error and stability}\label{subsec:gen}
\vspace{-0.05in}

In this section, we consider the generalization error of LPF-SGD algorithm and confront it with the generalization error of SGD. We first define the true loss as $L^{true}(\theta):=\mathbb{E}_{\xi\sim D}l(\theta;\xi)$, where $l$ is an arbitrary loss function (i.e., $l_o$ for SGD case and $l_{\mu}$ for LPF-SGD case). Note that the smoothed loss $l_{\mu}$ is by construction an estimator of the original loss $l_o$ with bounded bias and variance, especially in a reasonable range of $\sigma$, thus the comparison captures relevant insights regarding LPF-SGD properties compared to SGD. Since the distribution $\mathcal{D}$ is unknown, we replace the true loss by the empirical loss given as $L^{\mathcal{S}}(\theta):= \frac{1}{m}\sum_{i=1}^ml(\theta;\xi)$.
 
We assume $\theta=A(\mathcal{S})$ for a potentially randomized algorithm $A$. The generalization error is given as:
\vspace{-0.1in}
\begin{align}
    \epsilon_g:=\mathbb{E}_{\mathcal{S},A}[L^{true}(A(\mathcal{S}))-L^{\mathcal{S}}(A(\mathcal{S}))].
\end{align}
\vspace{-0.35in}

In order to bound $\epsilon_g$, we consider the following bound.
\begin{definition}[$\epsilon_s$-uniform stability \citep{hardt2016train}]\label{def:D1}
Let $\mathcal{S}$ and $\mathcal{S}'$ denote two data sets from input data distribution $\mathcal{D}$ such that $\mathcal{S}$ and $\mathcal{S}'$ differ in at most one example. Algorithm $A$ is $\epsilon_s$-uniformly stable if and only if for all data sets $\mathcal{S}$ and $\mathcal{S}'$ we have

\vspace{-0.35in}
\begin{align}
    \sup_\xi\mathbb{E}[l(A(\mathcal{S});\xi)-l(A(\mathcal{S}');\xi)]\leq\epsilon_s.
\end{align}
\vspace{-0.31in}
\end{definition}
Next theorem implies that the generalization error could be bounded using the uniform stability. 

\begin{thm}[\citealt{hardt2016train}]\label{thm:gen1}
     If $A$ is an $\epsilon_s$-uniformly stable algorithm, then the generalization error (the gap between the true risk and the empirical risk) of $A$ is upper-bounded by the stability factor $\epsilon_s$:
     
     \vspace{-0.35in}
    \begin{align}
        \epsilon_g:=\mathbb{E}_{\mathcal{S},A}[L^{true}(A(\mathcal{S}))-L^{\mathcal{S}}(A(\mathcal{S}))]\leq\epsilon_s
    \end{align}
    \vspace{-0.35in}
\end{thm}

 Theorems~\ref{thm:gau},~\ref{thm:gen1}, and uniform stability bound for SGD (Theorem~\ref{thm:gen2} in the Supplement) can be combined to analyze LPF-SGD. Intuitively, Theorem~\ref{thm:gau} explains how the Lipschitz continuous and smoothness properties of the objective function change after the function is smoothed with Gaussian LPF. This change can be propagated through Theorems~\ref{thm:gen1}. The details of the analysis are deferred to the Supplement (Section~\ref{sec:proofs}) and below we show the final result, Theorem~\ref{prop:gen}. A justification of the theoretical approach we took when deriving it is also in the Supplement (Section~\ref{sec:genjust}).

\begin{thm}[Generalization error (GE) bound of LPF-SGD]\label{prop:gen}
    Assume that $l_o(\theta;\xi)\in[0,1]$ is a $\alpha$-Lipschitz and $\beta$-smooth loss function for every example $\xi$. 
    Suppose that we run SGD and LPF-SGD for $T$ steps with non-increasing learning rate $\eta_t\leq c/t$. Denote the GE bound of SGD and LPF-SGD as $\hat{\epsilon}_g^o$ and $\hat{\epsilon}_g^{\mu}$, respectively. Then their ratio is
    \vspace{-0.15in}
    \begin{equation}
\rho=\frac{\hat{\epsilon}_g^{\mu}}{\hat{\epsilon}_g^o}=\frac{1-p}{1-\hat{p}}\left(\frac{2c\alpha}{T}\right)^{\hat{p}-p}\!\!{=O\left(\frac{1}{T^{\hat{p}-p}}\right)},
    \end{equation}
    \vspace{-0.25in}
    
    where $p=\frac{1}{\beta c+1}$, $\hat{p}=\frac{1}{\min\{\frac{\alpha}{\sigma},\beta\}c+1}$. Finally, the following two properties hold:
    \vspace{-0.15in}
    \begin{itemize}
        \item[i)] If  $\sigma>\frac{\alpha}{\beta}$ and $T{\gg}2c\alpha^2\left(\frac{1-p}{1-\hat{p}}\right)^{\frac{1}{\hat{p}-p}}$, $\rho{\ll} 1$.
        \vspace{-0.1in}
        \item[ii)] If $\sigma>\frac{\alpha}{\beta}$ and $T>2c\alpha^2e^{-p}$, increasing $\sigma$ leads to a smaller $\rho$.
    \end{itemize}
    \vspace{-0.2in}
\end{thm}

By point i) in Theorem~\ref{prop:gen}, when number of iterations is large enough, GE bound of LPF-SGD is much smaller than that of SGD, which implies LPF-SGD converges to a better optimal point with lower generalization error than SGD. Moreover, point ii) in Theorem~\ref{prop:gen} indicates that increasing $\sigma$ leads to a smaller generalization error.

\vspace{-0.1in}
\section{EXPERIMENTS}
\label{sec:experiments}
\vspace{-0.07in}

\begin{table}[t]
    \centering
    \setlength\tabcolsep{0.2pt}
    \renewcommand{\arraystretch}{1.1}
    \begin{tabular}{|c|c|c|c|c|c|c|}
    \hline
    Data & Mo & mSGD & E-SGD & ASO & SAM  & LPF- \\
    set & -del& & & & & SGD \\\hline
    \multirow{3}{1.2cm}{CIFAR 10}    & 18         & $11.5_{\pm 0.3}$  &   $10.9_{\pm 0.3}$ & $10.8_{\pm 0.5}$ &	$10.0_{\pm 0.1}$	&	$\mathbf{9.0_{\pm 0.2}}$  \\ 
                                & 50         & $10.2_{\pm 0.4}$  &   $10.4_{\pm 0.1}$ & $10.0_{\pm 0.2}$ &	$8.8_{\pm 0.3}$		&	$\mathbf{8.6_{\pm 0.1}}$ \\
                                & 101        & $9.5_{\pm 0.4}$   &   $9.9_{\pm 0.3}$	 & $9.3_{\pm 0.2}$	 &	$\mathbf{8.3_{\pm 0.3}}$	&	$8.7_{\pm 0.1}$  \\ \hline
    \multirow{3}{1.2cm}{CIFAR 100}   & 18         & $38.3_{\pm 0.3}$  &   $38.2_{\pm 0.2}$ & $37.3_{\pm 0.4}$ &	$36.2_{\pm 0.2}$	&	$\mathbf{30.0_{\pm 0.2}}$ \\ 
                                & 50         & $35.6_{\pm 1.0}$  &   $34.5_{\pm 0.3}$ & $34.3_{\pm 0.8}$ &	$33.1_{\pm 0.9}$	&	$\mathbf{30.6_{\pm 0.4}}$   \\ 
                                & 101        & $32.7_{\pm 0.5}$  &   $33.5_{\pm 0.5}$ & $31.9_{\pm 0.2}$ &	$30.7_{\pm 0.4}$	&	$\mathbf{29.9_{\pm 0.4}}$  \\ \hline
    tImgNet                & 18         & $64.1_{\pm0.1}$   &   $64.7_{\pm 0.1}$ & $63.7_{\pm 0.2}$ &   $63.1_{\pm0.3}$     &   $\mathbf{59.1_{\pm0.2}}$ \\ \hline
    ImgNet                    & 18         & $36.5_{\pm0.1}$   & $40.9_{\pm 0.1}$ & $38.0_{\pm 0.2}$ &   $35.9_{\pm0.1}$     &   $\mathbf{35.4_{\pm0.1}}$ \\ \hline
    \end{tabular}
    \vspace{-0.12in}
    \caption{Mean validation error for SGD, E-SGD, ASO, SAM and LPF-SGD. We use ResNet-18, -50, and -101 models and CIFAR-10, CIFAR-100, TinyImageNet (tImgNet), and ImageNet (ImgNet) data sets. The models were trained with no data augmentations.}
    \label{tab:exp1}
\end{table}

In this section we compare the performance of our method with momentum SGD-based optimization strategy as well as sharpness-aware DL optimizers that encourage the recovery of flat optima. Additional training details are provided in Section~\ref{sec:Codes} (Supplement). 

\textbf{Image Classification} Here we compare LPF-SGD with  momentum SGD (mSGD)~\citep{saad1998online, POLYAK19641} as well as sharpness-aware DL optimizers, Entropy-SGD (E-SGD)~\citep{DBLP:journals/corr/ChaudhariCSL17}, Adaptive SmoothOut  with denoising~\citep{DBLP:journals/corr/abs-1805-07898} (ASO), SAM, and newly introduced adaptive variant of SAM, ASAM~\citep{kwon2021asam}. Differently than in SAM paper, in all the experiments in our work we trained our models on a single GPU worker without additional label smoothing and gradient clipping. This is done in order to obtain a clear understanding how the methods themselves perform without additional regularizations. We utilized publicly available code repositories for both model architectures and the methods that we compare LPF-SGD with. For each pair of model and data set, we use standard mSGD settings of batch size, weight decay, momentum coefficient, learning rate, learning rate schedule, and total number of epochs across all optimizers. The details of hyper-parameter selection as well as exemplary convergence curves are shown in Section~\ref{sec:IC} (Supplement). 

In the first set of our experiments we utilized image classification models based on ResNets~\citep{he2016deep}. They are prone to over-fitting when trained without any data augmentation. The sharpness based optimizers prevent overfitting by seeking flat minima. Therefore, we first compare the performance of optimizers on ResNets without performing data augmentation. We use ResNet18, 50, and 101 models and open sourced CIFAR-10 and CIFAR-100~\citep{CIFAR100}, TinyImageNet~\citep{imagenet_cvpr09}, and ImageNet~\citep{imagenet_cvpr09} data sets. We run each experiment for $5$ seeds. We report the mean validation error along with the $95\%$ confidence interval (Table~\ref{tab:exp1}). Table~\ref{tab:exp1}, as well as Figures~\ref{fig:exp1_1} and~\ref{fig:exp1_2} in the Supplement show that LPF-SGD compares favorably to all other methods in terms of the generalization performance. We also evaluate various sharpness measures for models trained with mSGD, SAM and LPF-SGD (Table~\ref{tab:exp1_sharp} in the Supplement) and show that LPF-SGD achieves the lowest value of the LPF sharpness based measure leading to the lowest validation error.

Next, we evaluate the performance of LPF-SGD on standard ImageNet training available in PyTorch~\citep{paszke2017automatic}, where we train a ResNet18 model with basic data augmentation. Table ~\ref{tab:imgnet} shows the validation error rate for various optimizers, among which our technique achieves superior performance.

\begin{table}[H]
    \vspace{0.05in}
    \centering
    \setlength\tabcolsep{2.2pt}
    \begin{tabular}{|c|c|c|c|c|}
        \hline
        Opt         & mSGD   & ASO               & SAM               & LPF-SGD \\ \hline
        Val Error   & 30.24 & $29.54_{\pm0.10}$ & $29.49_{\pm0.10}$ & $\mathbf{29.45_{\pm0.02}}$ \\ \hline
    \end{tabular}
    \vspace{-0.12in}
    \caption{Mean validation error rate with 95\% confidence interval on ImageNet data set trained with ResNet18 model and basic augmentation.}
    \label{tab:imgnet}
    \vspace{-0.05in}
\end{table}
\begin{table*}[t]
    \centering
    \setlength\tabcolsep{0.2pt}
    \renewcommand{\arraystretch}{1.1}
    \begin{tabular}{|c|c|c|c|c|c|c|c|c|c|c|c|} 
        \hline
        & & \multicolumn{5}{c|}{CIFAR-10} & \multicolumn{5}{c|}{CIFAR-100} \\ \cline{3-12}
        Model                                           & Aug       & mSGD              & E-SGD            & ASO            & SAM                               &  LPF-SGD  & mSGD              & E-SGD            & ASO            & SAM                               &  LPF-SGD                                             \\ \hline
            \multirow{3}{1.5cm}{WRN16-8}                  & B       & $4.2_{\pm 0.2}$   & $4.5_{\pm <0.1}$ & $4.2_{\pm 0.1}$ & $3.8_{\pm 0.1}$                  & $\mathbf{3.7_{\pm <0.1}}$  & $20.6_{\pm 0.2}$  & $20.7_{\pm 0.1}$  & $20.2_{\pm 0.2}$   & $19.4_{\pm 0.2}$             & $\mathbf{18.9_{\pm 0.1}}$    \\  
                                                        & B+C       & $3.9_{\pm 0.1}$   & $3.8_{\pm 0.1}$  & $3.6_{\pm 0.1}$ & $3.3_{\pm 0.1}$                  & $\mathbf{3.2_{\pm 0.1}}$   & $20.0_{\pm 0.1}$  & $20.1_{\pm 0.2}$  & $19.7_{\pm 0.1}$   & $18.9_{\pm 0.1}$             & $\mathbf{18.3_{\pm 0.1}}$    \\ 
                                                        & B+A+C     & $3.3_{\pm 0.1}$   & $3.6_{\pm 0.1}$  & $3.2_{\pm 0.1}$ & $\mathbf{2.9_{\pm 0.1}}$         & $3.1_{\pm 0.1}$            & $19.3_{\pm 0.2}$  & $19.4_{\pm 0.1}$  & $18.9_{\pm 0.1}$   & $18.1_{\pm 0.1}$             & $\mathbf{17.6_{\pm 0.1}}$    \\ \hline      
            \multirow{3}{1.5cm}{WRN28-10}                 & B       & $4.0_{\pm 0.1}$   & $4.0_{\pm 0.0}$  & $3.9_{\pm 0.1}$ & $\mathbf{3.2_{\pm 0.1}}$         & $3.5_{\pm 0.1}$            & $19.1_{\pm 0.2}$  & $19.8_{\pm 0.2}$  & $18.8_{\pm 0.1}$   & $\mathbf{17.4_{\pm <0.1}}$   & $\mathbf{17.4_{\pm 0.1}}$    \\  
                                                        & B+C       & $3.1_{\pm <0.1}$  & $3.3_{\pm 0.1}$  & $3.3_{\pm 0.1}$ & $\mathbf{2.7_{\pm 0.1}}$         & $\mathbf{2.7_{\pm <0.1}}$  & $18.3_{\pm 0.1}$  & $18.8_{\pm 0.2}$  & $18.0_{\pm 0.3}$   & $17.1_{\pm 0.1}$             & $\mathbf{16.9_{\pm 0.2}}$    \\  
                                                        & B+A+C     & $2.6_{\pm 0.1}$   & $3.0_{\pm 0.1}$  & $2.7_{\pm 0.1}$ & $\mathbf{2.3_{\pm 0.1}}$         & $2.5_{\pm 0.1}$            & $17.3_{\pm 0.2}$	& $17.6_{\pm 0.1}$  & $16.9_{\pm 0.4}$   & $\mathbf{15.9_{\pm 0.1}}$	& $\mathbf{15.9_{\pm 0.1}}$   \\ \hline  
            \multirow{3}{1.5cm}{Shake Shake (26 2x96d)}    & B      & $2.9_{\pm 0.1}$ 	& $3.3_{\pm 0.1}$  & $2.8_{\pm 0.1}$ & $2.6_{\pm 0.1}$                  & $\mathbf{2.5_{\pm <0.1}}$  & $17.1_{\pm 0.1}$  & $17.5_{\pm 0.3}$  & $17.5_{\pm 0.1}$   & $16.9_{\pm 0.1}$             & $\mathbf{16.6_{\pm 0.3}}$    \\ 
                                                        & B+C       & $2.5_{\pm <0.1}$ 	& $3.1_{\pm 0.0}$  & $2.4_{\pm 0.0}$ & $\mathbf{2.2_{\pm <0.1}}$        & $\mathbf{2.2_{\pm <0.1}}$  & $17.0_{\pm 0.4}$  & $17.3_{\pm 0.1}$  & $16.8_{\pm 0.2}$   & $16.6_{\pm 0.2}$             & $\mathbf{16.1_{\pm 0.2}}$    \\ 
                                                        & B+A+C     & $2.1_{\pm 0.1}$ 	& $3.1_{\pm 0.1}$  & $2.1_{\pm 0.1}$ & $2.0_{\pm 0.1}$ 	                & $\mathbf{2.0_{\pm 0.1}}$   & $15.7_{\pm 0.1}$  & $16.3_{\pm 0.1}$  & $15.6_{\pm 0.0}$   & $15.3_{\pm 0.2}$             & $\mathbf{15.0_{\pm <0.1}}$      \\ \hline
            \multirow{3}{1.5cm}{PyNet110 ($\alpha=270$)}  & B       & $3.7_{\pm 0.1}$   & $3.9_{\pm 0.1}$  & $3.7_{\pm 0.0}$ & $\mathbf{3.1_{\pm <0.1}}$        & $3.4_{\pm 0.1}$            & $18.4_{\pm 0.3}$  & $19.1_{\pm 0.1}$  & $18.5_{\pm 0.3}$   & $18.3_{\pm 0.2}$             & $\mathbf{18.0_{\pm 0.1}}$    \\ 
                                                        & B+C       & $2.9_{\pm 0.1}$   & $2.8_{\pm 0.1}$  & $2.9_{\pm 0.1}$ & $\mathbf{2.5_{\pm 0.2}}$         & $\mathbf{2.5_{\pm <0.1}}$  & $17.7_{\pm 0.1}$  & $18.1_{\pm 0.2}$  & $17.7_{\pm 0.4}$   & $17.4_{\pm 0.3}$             & $\mathbf{17.0_{\pm 0.1}}$    \\ 
                                                        & B+A+C     & $2.2_{\pm 0.0}$   & $2.3_{\pm 0.1}$  & $2.3_{\pm 0.0}$ & $\mathbf{2.0_{\pm <0.1}}$        & $2.1_{\pm 0.0}$            & $16.9_{\pm 0.0}$  & $15.9_{\pm <0.1}$ & $16.0_{\pm 0.1}$   & $15.8_{\pm 0.1}$             & $\mathbf{15.6_{\pm 0.1}}$    \\ \hline
            \multirow{3}{1.5cm}{PyNet272 ($\alpha=200$)}  & B       & $3.3_{\pm 0.1}$   & $3.6_{\pm 0.0}$  & $3.3_{\pm 0.1}$ & $\mathbf{2.9_{\pm 0.1}}$         & $3.3_{\pm 0.1}$            & $17.6_{\pm 0.2}$  & $18.5_{\pm 0.1}$  & $17.4_{\pm 0.1}$   & $17.0_{\pm 0.2}$             & $\mathbf{16.7_{\pm <0.1}}$   \\ 
                                                        & B+C       & $2.6_{\pm 0.0}$   & $2.7_{\pm 0.1}$  & $2.7_{\pm 0.1}$ & $2.4_{\pm 0.0}$                  & $\mathbf{2.3_{\pm 0.1}}$   & $16.9_{\pm <0.1}$ & $17.5_{\pm 0.3}$  & $16.9_{\pm 0.1}$   & $16.2_{\pm 0.2}$             & $\mathbf{15.9_{\pm <0.1}}$   \\ 
                                                        & B+A+C     & $2.1_{\pm 0.2}$   & $2.2_{\pm 0.1}$  & $2.1_{\pm 0.1}$ & $\mathbf{2.0_{\pm <0.1}}$        & $\mathbf{2.0_{\pm <0.1}}$  & $14.9_{\pm <0.1}$ & $15.0_{\pm 0.1}$  & $14.9_{\pm 0.3}$   & $\mathbf{14.7_{\pm 0.2}}$    & $\mathbf{14.7_{\pm 0.1}}$    \\ \hline
    \end{tabular}
    \vspace{-0.14in}
    \caption{Mean validation error with 95\% confidence interval. B refers to basic data augmentation, C refers to Cutout augmentation, and A refers to AutoAugmentation.}
    \label{tab:exp2}
    \vspace{-0.15in}
\end{table*}

Finally, we evaluate the performance of LPF-SGD on state-of-the-art architectures. We trained WRN16-8, WRN28-10~\citep{zagoruyko2016wide}, ShakeShake (26 2x96d)~\citep{gastaldi2017shake}, PyramidNet-110($\alpha$= 270), and PyramidNet-272($\alpha$=200)~\citep{Han_2017_CVPR} on open sourced CIFAR-10 and CIFAR-100 data sets. We also utilized three progressively increasing data augmentation schemes: i) basic (random cropping and horizontal flipping), ii) basic + cutout~\citep{devries2017improved} and iii) basic + auto-augmentation~\citep{Cubuk_2019_CVPR} + cutout~\citep{devries2017improved}. We run $5$ seeds for WRN16-8 and WRN28-10 and $2$ seeds for ShakeShake and PyramidNet models. We report the mean validation error with a $95\%$ confidence interval on the validation set (Table~\ref{tab:exp2}). Table \ref{tab:exp2} shows that LPF-SGD either outperforms or show comparable performance to SAM, and at the same time is always superior to E-SGD and ASO. As in case of ResNets, LPF-SGD clearly wins with mSGD that does not explicitly encourage the recovery of flat optima in the DL optimization landscape.
\begin{table}[h]
    \centering
    \begin{tabular}{|c|c|c|c|}
        \hline
        \multirow{2}{1.7cm}{Model}    & \multirow{2}{*}{Aug} & \multicolumn{2}{c|}{CIFAR-100}   \\ \cline{3-4}
                                    &               & ASAM                       & LPF-SGD \\ \hline
        \multirow{3}{1.7cm}{WRN 16-8} & B             & $19.2_{\pm 0.0}$           & $\mathbf{18.9_{\pm 0.1}}$  \\ \cline{2-4}
                                                    & B+C       & $\mathbf{18.3_{\pm 0.1}}$  & $\mathbf{18.3_{\pm 0.1}}$  \\ \cline{2-4}
                                                    & B+A+C   & $\mathbf{17.5_{\pm 0.2}}$  & $17.6_{\pm 0.1}$              \\ \hline
        \multirow{3}{1.7cm}{WRN 28-10}              & B             & $17.6_{\pm 0.2}$           & $\mathbf{17.4_{\pm 0.1}}$  \\ \cline{2-4}
                                                    & B       & $\mathbf{16.7_{\pm 0.1}}$  & $16.9_{\pm 0.2}$          \\ \cline{2-4}
                                                    & B+A+C   & $16.1_{\pm 0.1}$           & $\mathbf{15.9_{\pm 0.1}}$  \\ \hline
        \multirow{3}{1.7cm}{ShakeShake (26 2x96d)}      & B            &  $17.0_{\pm 0.1}$          & $\mathbf{16.6_{\pm 0.3}}$  \\ \cline{2-4}
                                                    & B+C      &  $16.7_{\pm 0.2}$          & $\mathbf{16.1_{\pm 0.2}}$  \\ \cline{2-4}
                                                    & B+A+C  &  $15.2_{\pm 0.1}$          & $\mathbf{15.0_{\pm <0.1}}$ \\ \hline
    \end{tabular}
    \vspace{-0.13in}
    \caption{\label{tab:asam} Validation error rate with $95\%$ confidence interval for ASAM and LPF-SGD.}
    \vspace{-0.16in}
\end{table}

We also present experiments capturing comparison of LPF-SGD to a newly introduced ASAM~\citep{kwon2021asam} method. For this experiment we train WRN16-8, WRN 28-10, and ShakeShake models on CIFAR-100 data set. LPF-SGD is found to be superior to ASAM as captured in Table~\ref{tab:asam}. 

\begin{table}[!ht]
\vspace{0.07in}
\centering
\setlength\tabcolsep{2.2pt}
\begin{tabular}{|c|c|c|c|}
    \hline
    Opt $\backslash$ Model                                        & WRN16-8                    & WRN28-10 & PyNet110 \\ \hline
    mSGD                                                      & 0.100 s & 0.292 s   & 0.470 s  \\ \hline
    SAM                                                      & 0.202 s                     & 0.582 s   & 0.924 s  \\ \hline
    LPF-SGD (M=2) & 0.115 s & 0.319 s   & 0.566 s  \\ \hline
    LPF-SGD (M=4) & 0.138 s & 0.380 s   & 0.679 s  \\ \hline
    LPF-SGD (M=8) & 0.196 s & 0.500 s   & 0.904 s \\ \hline
\end{tabular}
\vspace{-0.13in}
\caption{Computational time for a single iteration of mSGD, SAM, and LPF-SGD for various settings of $M$.}
\label{tab:time}
\vspace{0.05in}
\end{table}

In Table ~\ref{tab:time}, we also show the computational cost for a single iteration of mSGD, SAM, and LPF-SGD for multiple values of M, computed on a NVIDIA GTX1080 GPU. For each LPF-SGD update, we split the batch of input data into $M$ mini-batches, compute MC iterations on each mini-batch, and accumulate the gradients before making a final weight update. Therefore, the number of operations (multiplications + additions) between mSGD and LPF-SGD are preserved. LPF-SGD incurs additional cost, in comparison to SGD, due to operating on smaller batches internally (by a factor of M) and the computation of $\Sigma$ matrix. SAM, on the other hand, relies on a nested optimization scheme, which is overall slower compared to LPF-SGD, which is why LPF-SGD consistently outperforms SAM time-wise for all explored settings of $M$. Note that the computational time can be further improved as diagonal blocks of the Sigma matrix are independent and can be computed by parallel CPU threads.

\textbf{Machine Translation} Here we train a transformer model based on ~\citet{vaswani2017attention} to perform German to English translation on WMT2014 data set~\citep{bojar-etal-2014-findings}. The model was trained using Adam~\citep{DBLP:journals/corr/KingmaB14}. We compare the performance of LPF-Adam (our method) with Adam as well as Entropy-Adam (E-Adam), ASO-Adam, and SAM-Adam. Section~\ref{sec:MT} (Supplement) contains training details. In Table~\ref{tab:nlp}, we show BLEU score for both validation and test data sets and reveal that in the context of machine translation recovering flat optima with our proposed optimizer leads to the highest scores.\\

\begin{table}[!ht]
    \centering
    \begin{tabular}{|c|c|c|}
        \hline
        Optimizer    & Validation               & Test                    \\ \hline
        Adam         & $21.76_{\pm 0.26}$       & $20.97_{\pm 0.43}$      \\ \hline
        E-Adam       & $19.91_{\pm 0.03}$	    & $18.84_{\pm 0.18}$      \\ \hline
        ASO-Adam     & $21.91_{\pm 0.24}$       & $20.77_{\pm 0.21}$      \\ \hline
        SAM-Adam     & $22.06_{\pm 0.04}$       & $20.92_{\pm 0.15}$      \\ \hline
        LPF-Adam     & $\mathbf{22.10_{0.15}}$  & $\mathbf{21.14_{0.23}}$ \\ \hline
    \end{tabular}
    \vspace{-0.13in}
    \caption{BLEU scores with 95\% confidence interval for German to English translation on WMT2014 data set.}
    \label{tab:nlp}
\end{table}

Finally, the ablation studies aiming at understanding the impact of various hyper-parameters of the LPF-SGD optimizer on its performance, and the studies showing that prolonging mSGD training does not help mSGD to match LPF-SGD, as well as the studies confirming that LPF-SGD shows more consistent robustness to adversarial attacks compared to other methods are deferred to Section~\ref{sec:AS} and~\ref{sec:AR} (Supplement). 

\vspace{-0.15in}
\section{CONCLUSIONS}
\label{sec:con}
\vspace{-0.1in}

In this paper we show comprehensive empirical study investigating the connection between the flatness of the optima of the DL loss landscape and the generalization properties of DL models. We derive an algorithm, LPF-SGD, for training DL models based on the sharpness measure that best correlates with model generalization abilities. LPF-SGD compares favorably to common training strategies. Regarding societal impact of our work and extensions, new landscape-driven DL optimization tools, such as LPF-SGD, will have a strong impact on a wide range of DL applications, where better solvers translate to more efficient utilization of computational resources, i.e., new optimization strategies will maximize the performance of DL architectures. The outcomes of such research works can be leveraged by public and private entities to shift to significantly more powerful computational learning platforms.
\subsubsection*{Acknowledgements}
The authors would like to acknowledge the support of the NSF Award $\#$2041872 in sponsoring this research.

\bibliographystyle{plainnat}
\bibliography{ref}
\clearpage

\onecolumn \makesupplementtitle

\section{NETWORK BALANCING}\label{sec:bal}
We consider balanced networks, i.e., networks where norms of weights in each layer are roughly the same. In this section, we present a normalization scheme utilized in Section~\ref{sec:sharp_v_gen} to balance the network. Let $x$ be the input to the network, $\theta_{i}$ be the weight matrix of the $i^{th}$ layer, $\hat{\theta}_{i}$ denote bias matrix and $\sigma(\,\,)$ denote the relu nonlinearity. The output of a network with three layers; convolution, batch normalization and relu can be written as
\begin{equation}
    f(x) = \sigma\Big( \frac{(\theta_{1}X) - E[\theta_{1}X]}{Var(\theta_{1}X)} \theta_{2} + \hat{\theta}_{2}\Big).
\end{equation}
Let $D_{i}$ denote a diagonal normalization matrix associated with the $i^{th}$ layer. The diagonal elements of the matrix are defined as $D_{i}[j,j] = \frac{1}{||\theta_{i}^{j}||_{F} + ||\hat{\theta}_{i}^{j}||_{F}}$, where $\theta_{i}^{j}$ is the weight matrix of $j^{\text{th}}$ filter in the $i^{\text{th}}$ layer. We normalize the parameters of the network as,
\begin{equation}
    f(x) = \sigma\Big( \frac{(D_{1}W_{1}X) - E[D_{1}W_{1}X]}{Var(D_{1}W_{1}X)} D_{2}(W_{2} + B_{2})D_{2}^{-1}\Big)
\end{equation}
Note that $\hat{W}_{i} = D_{i}W_{i}(D_{i})^{-1} = W_{i}$. Since $\sigma(\lambda x) = \lambda \sigma(x)$ for $\lambda \geq 0$, we can rewrite the above equation as
\begin{equation}
    f(x) = \sigma\Big( \frac{(D_{1}W_{1}X) - E[D_{1}W_{1}X]}{Var(D_{1}W_{1}X)} D_{2}(W_{2} + B_{2})\Big)D_{2}^{-1}.
\end{equation}
We keep the multiplication with the matrix $D_{2}^{-1}$ as a constant parameter in the network but it can also be combined with the parameters of the next layer. We normalize the parameters of each layer as we move from the first layer to the last layer of the network. Figure \ref{fig:norm} shows filter wise parameter norm ($D^{-1}$) of LeNet and ResNet18 models trained on MNIST and CIFAR-10 data sets respectively. In Table~\ref{tab:norm}, we show the mean training cross entropy loss before and after normalization.
\begin{figure}[H]
    \centering
    \includegraphics[width=0.48\textwidth]{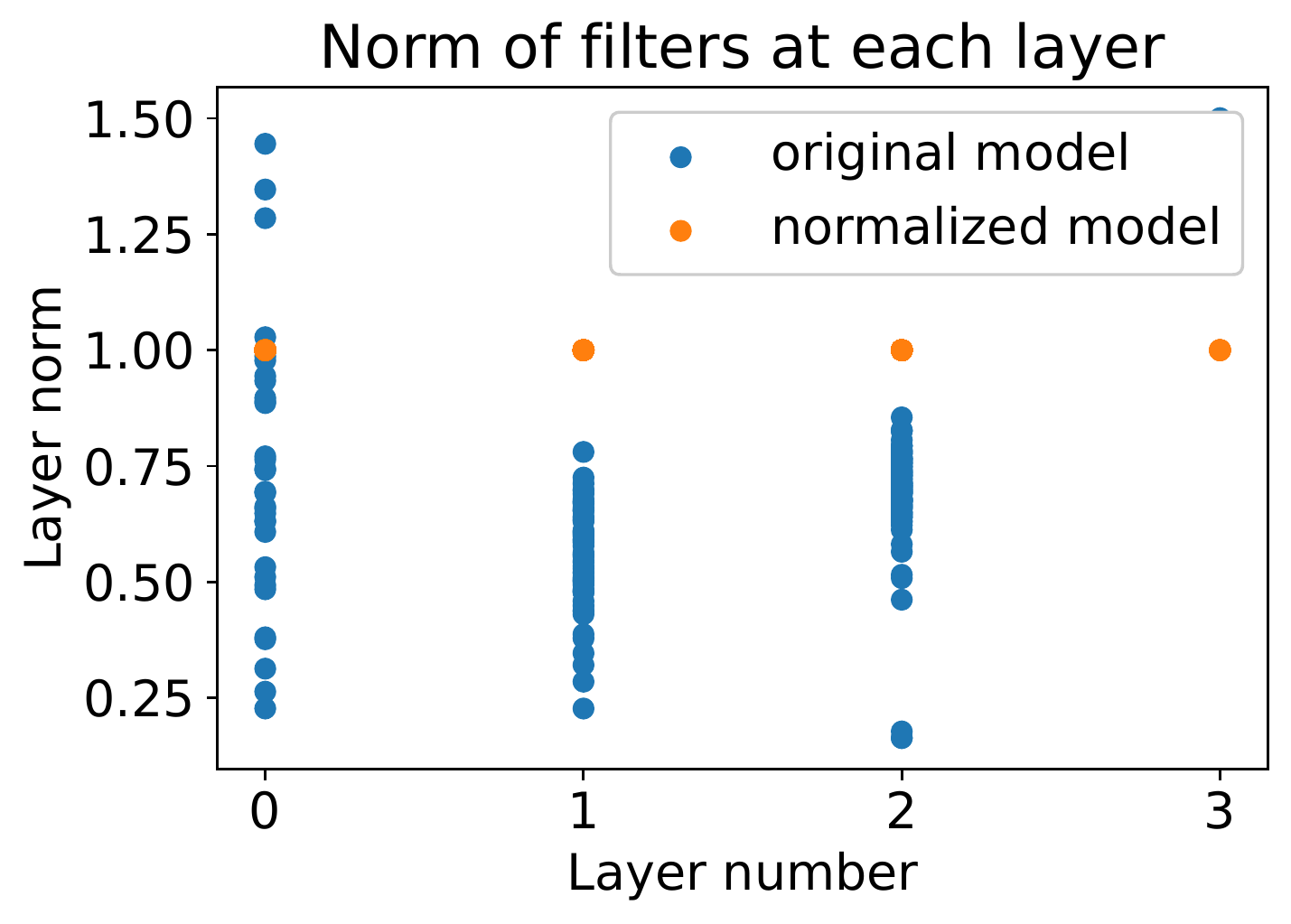}
    \includegraphics[width=0.48\textwidth]{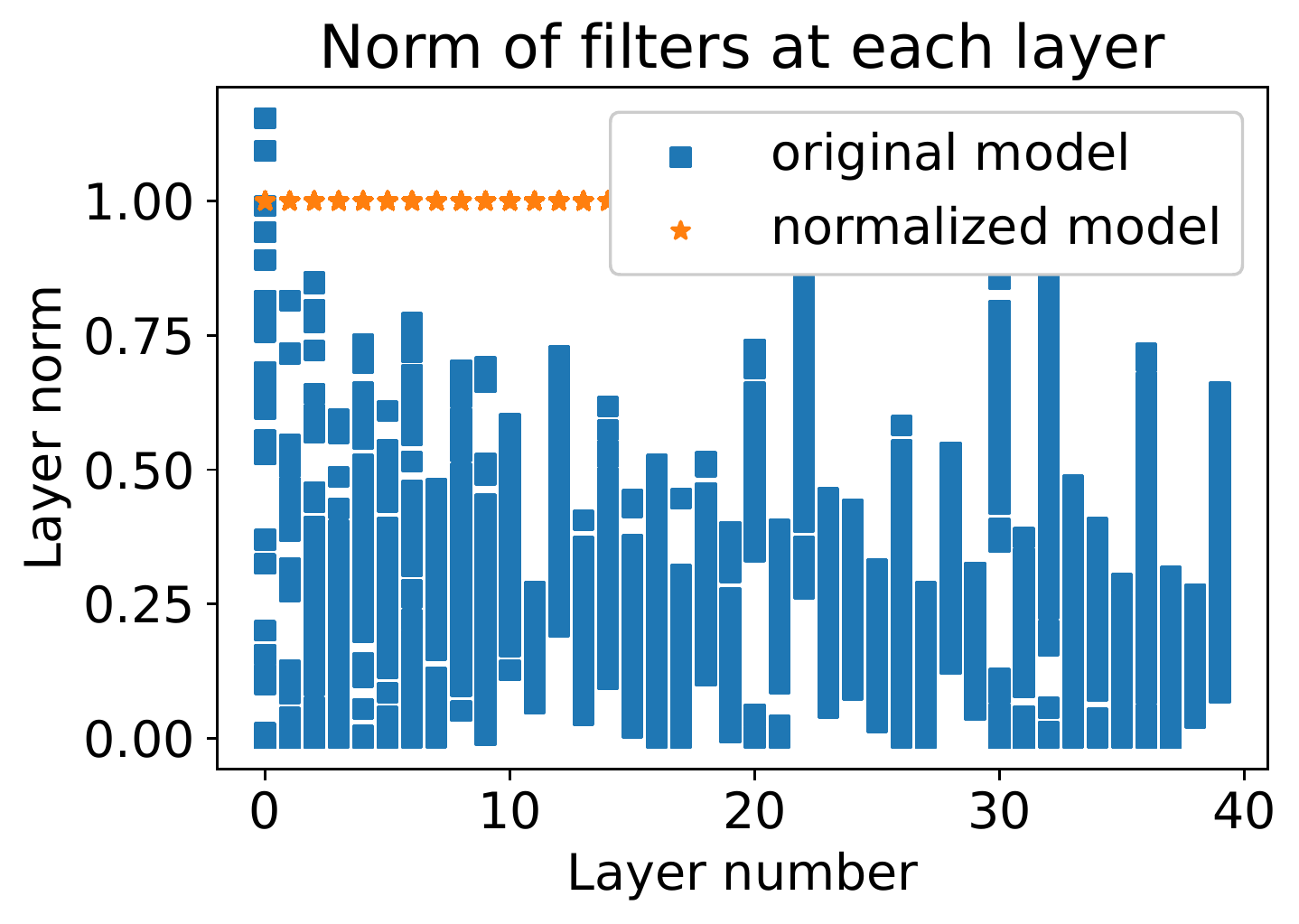}
    \caption{Norm of the filter in each layer of (\textbf{left}) LeNet and (\textbf{right}) ResNet18 networks trained on CIFAR-10 data set before and after normalization.}
    \label{fig:norm}
\end{figure}
\begin{table}[H]
    \centering
    \begin{tabular}{|c|c|c|}
        \hline
         Model & Loss before normalization & Loss after normalization \\ \hline
         LeNet & 0.0006375186720272088 & 0.0006375548382939456 \\ 
         ResNet18 & 0.0014627227121118522 & 0.0014617550651283215 \\ \hline
    \end{tabular}
    \caption{Validation loss for CIFAR-10 data set before and after normalization.}
    \label{tab:norm}
\end{table}

\section{SHARPNESS MEASURES}
\label{sec:FM}

In this section, we define various sharpness measures and the corresponding algorithms utilized in section~\ref{sec:sharp_v_gen} of the main paper to study the interplay between sharpness and generalization gap.

Let $m$ denote the number of samples in the training data set. Let $\theta$ denote the set of network parameters. $L(\theta)$ is the loss function, whose gradient is denoted as $\nabla_{\theta}L(\theta)$ and the Hessian matrix is denoted as $H=\nabla_\theta^2 L(\theta)$. Also, $\lambda_{i}(H)$ is the $i^{th}$ eigenvalue of the Hessian.  
\subsection{PAC-Bayes measure}
\begin{definition}[PAC-Bayes measure~\citep{jiang2019fantastic}]
For any $\delta > 0$ and a prior and posterior distributions over the parameters given respectively as $P =\mathcal{N}(\theta^{0}, \sigma^{2}I)$ and $Q= \mathcal{N}(\theta^{*}, \sigma^{2}I)$, with probability $1 - \delta$ the following bound holds:

\begin{equation}
  \mathbb{E}_{\theta\sim Q}[L^{true}(\theta)] \leq \mathbb{E}_{\theta\sim Q}[L(\theta)] + \sqrt{\frac{\frac{|| \theta^{*} - \theta^{0}||_{2}^{2}}{2\sigma^{2}} + ln(\frac{m}{\delta})}{2(m - 1)}},
\end{equation}
where $m$ is the size of the training data, $L^{true}$ is the true risk under unknown data distribution, and $\theta^{0}$ are the initial network parameters. The PAC-Bayes sharpness measure is then defined as:

\begin{align}
    \label{eq:sharp_meas_7}
        \mu_{PAC-Bayes} = \frac{|| \theta^{*} - \theta^{0}||_{2}^{2}}{4\sigma^{2}} + 0.5\ln\left(\frac{m}{\delta}\right),
\end{align}

where $\sigma$ is chosen such that $\mathbb{E}_{\theta\sim Q}[L(\theta)] - L(\theta^*) \leq 0.1$.
\end{definition}

\begin{algorithm}[H]
    \caption{$\mu_{PAC-Bayes}$}
    \label{alg:pac_bayes}
    \textbf{Input:} $\theta^*$: final weights, $M$: MC iterations, $\psi$: tolerance \\
    \textbf{Output:} $\sigma$
    \begin{algorithmic}
    \STATE $\sigma_{min} = \text{FLOAT\_EPSILON\_MIN}$ \\
    \STATE $\sigma_{max} = \text{FLOAT\_EPSILON\_MAX}$\\
    \WHILE{TRUE}
        \STATE $\sigma = (\sigma_{min} + \sigma_{max})) / 2$ \\
        \STATE $\hat{l} = 0$ \\
        \FOR{j = 1 to $M$} 
            \STATE $\theta = \theta^* + \mathcal{N}(0, \sigma^{2}I)$ \\
            \STATE $\hat{l} += \hat{L}(\theta)$
        \ENDFOR 
        \STATE $\hat{l} = \hat{l} / M$
        \STATE $d = \hat{l} - L(\theta^*)$
        \IF{$\epsilon - \psi \leq d \leq \epsilon + \psi$}
            \RETURN $\sigma$
        \ENDIF
        \IF{$d < \epsilon - \psi$}
            \STATE $\sigma_{min} = \sigma$
        \ELSIF{$d > \epsilon + \psi$}
            \STATE $\sigma_{max} = \sigma$
        \ENDIF
    \ENDWHILE
    \end{algorithmic}
\end{algorithm}
\subsection{$\epsilon$-sharpness}
\begin{definition} [$\epsilon$-sharpness~\citep{keskar2016largebatch}]
\label{eq:sharp_meas_6}
$\epsilon$-flatness captures the maximal range of the perturbations of the parameters that do not increase the value of the loss function by more than $\epsilon$ and is defined as the maximal value of the radius $R$ of the Euclidean ball $B_2(\theta^{*},R)$ centered at the local minimum ($\theta^{*}$) of the loss function such that $\forall_{\theta\in B_2(\theta^{*},R)} \mathcal{L}(\theta) - \mathcal{L}(\theta^{*}) < \epsilon$. $\epsilon$-sharpness is the inverse of $\epsilon$-flatness.
\end{definition}
\begin{algorithm}[H]
    \label{alg:eps_sharpness}
    \caption{$\epsilon$ - sharpness }
    \textbf{Input:} $\theta^*$: final weights, $\psi $: tolerance, $\epsilon$:  target deviation in loss  \\
    \textbf{Output:} $\epsilon$ - sharpness
    \begin{algorithmic}
        \STATE $\eta_{max}=$FLOAT\_EPSILON\_MIN
        \WHILE{TRUE}
            \STATE $\theta = \theta^* + \eta_{max} \nabla L(\theta^*)$ \\
            \STATE d = $L(\theta)  - L(\theta^{*})$ \\
            \IF{$d < \epsilon$}
                \STATE $\eta_{max} = \eta_{max}*10$ \\
            \ENDIF
        \ENDWHILE
        \STATE $\eta_{min}=$FLOAT\_EPSILON\_MIN
        \WHILE{TRUE}
            \STATE $\eta = (\eta_{max} + \eta_{min}) / 2$
            \STATE $\theta = \theta^* + \eta \nabla L(\theta^*)$ $\backslash \backslash$ step in full-data gradient direction 
            \STATE d = $L(\theta)  - L(\theta^{*})$ \\
            \IF{$\epsilon - \psi \leq d \leq \epsilon + \psi$}
                \RETURN $\frac{1}{||\theta - \theta^*||}$
            \ENDIF
            \IF{$d < \epsilon - \psi$}
                \STATE $\eta_{min} = \eta$
            \ELSIF{$d > \epsilon + \psi$}
                \STATE $\eta_{max} = \eta$
            \ENDIF
        \ENDWHILE
    \end{algorithmic}
\end{algorithm}

\subsection{Fisher Rao Norm}%
\begin{definition}[Fisher Rao Norm ~\citep{liang2019fisher}]
\label{eq:sharp_meas_5}
The Fisher-Rao Norm (FRN), under appropriate regularity conditions~\citep{liang2019fisher}, is approximated as $\theta^{*T} \mathbb{E}[\nabla^{2} L(\theta^*)] \theta^*$.
\end{definition}
FRN is calculated as $\theta^{*T}hvp(\theta^*)$ where hvp is the hessian vector product function.
\subsection{Hessian based measures ($\lambda_{max}(H)$, $Trace(H)$, $d_{eff}(H)$, and $\|H\|_F$)}
\label{alg:eig_max}
\begin{definition}
The Hessian-based sharpness measures computed at the solution $\theta^*$ are: the Frobenius norm of the Hessian ($||H||_{F}$), trace of the Hessian (Trace ($H$), the largest eigenvalue of the Hessian ($\lambda_{max}(H))$, and the effective dimensionality of the Hessian~\citep{maddox2020rethinking, mackay1992bayesian} defined as $d_{eff} = \sum_{i=1}^{n}\frac{\lambda_{i}}{\lambda_{i} + 1}$, where $\lambda_i$ is the $i^{\text{th}}$ eigenvalue of the Hessian.
\end{definition}

We compute $100$ eigenvalues of the Hessian of the loss function using Stochastic Lanczos quadrature algorithm as described in ~\citep{ghorbani2019investigation}.  $\lambda_{max}(H)$, $Trace(H)$, and $d_{eff}(H)$ can be easily estimated from the set of $100$ eigen values. Note that for any matrix $A$,  $||A||_F^2 = \mathbf{E}_{v}[||Av||_{2}^{2}]$, where $v \sim \mathcal{N}(0,I)$. Therefore, we use the algorithm the following algorithm to efficiently compute the Frobenius norm. 
\begin{algorithm}[H]
    \caption{$||H||_{F}$}
    \label{alg:fro_norm}
     \textbf{Input:} $M$: number of iterations, $hvp(v)$: Hessian-vector product \\
     \textbf{Output:} $||H||_{F}$
    \begin{algorithmic}
        \STATE $out \leftarrow 0$ \\
        \FOR{k = 1 to M}
            \STATE $v^{k} \sim \mathcal{N}(0, I)$
            \STATE out += $||hvp(v^{k})||_{2}^{2}$ 
        \ENDFOR
        \RETURN $\sqrt{out / M}$
    \end{algorithmic}
\end{algorithm}
\subsection{Shannon Entropy}
Shannon entropy is a classical measure of sharpness which does not consider sharpness of the loss function in the parameter space and therefore is fundamentally different from the rest of the measures.
\begin{definition}[Shannon Entropy~\citep{pereyra2017regularizing}]
Let $f_{\theta^*}(x)[j]$ denote the probability of $j^{th}$ class predicted by the deep learning model $f_{\theta^*}$ for input data $x$, and let $\kappa$ be the total number of classes. The Shannon entropy at $\theta^*$ is given as:
\begin{align}
    \mu_{entropy} \!=\! -\frac{1}{m}\sum_{i = 1}^{m} \sum_{j = 1}^{\kappa} f_{\theta^*}(x_{i})[j]\log(f_{\theta^*}(x_{i})[j])
    \label{eq:sharp_meas_10}
\end{align}
\end{definition}
\begin{algorithm}[H]
    \caption{$\mu_{entropy}$}
    \textbf{Input:}  $f_{\theta^*}$: trained model \\
    \textbf{Output:} Shannon Entropy
    \begin{algorithmic}
        \STATE $\text{out} \leftarrow 0$
        \FOR{i = 1 to N}
            \FOR{j = 1 to K}
                \STATE $\text{out} += f_{\theta^*}(x_{i})[j] \times log( f_{\theta^*}(x_{i})[j])$
            \ENDFOR
        \ENDFOR
        \RETURN -out / N
    \end{algorithmic}
\end{algorithm}
\subsection{Norm of the Gradient of Local Entropy}
\begin{definition}[Gradient of the Local Entropy~\citep{DBLP:journals/corr/ChaudhariCSL17}]
The sharpness of the loss landscape at the solution $\theta^*$ is computed as the norm of the gradient of the local entropy (LE). The local entropy is 

\begin{align}
    F(\theta^*, \gamma) &= \log\Big(\int_{\theta'} \exp( - L(\theta') - \frac{\gamma}{2} ||\theta^* - \theta'||_{2}^{2}) d\theta' \Big),
    \nonumber
\end{align}
where $\gamma$ is the scoping parameter and the norm of its gradient is given as
\begin{align}
    \label{eq:sharp_meas_8}
    \mu_{\text{LE}} = ||\nabla_{\theta^*}F(\theta^*, \gamma)|| = || \gamma(\theta^*-\mathbb{E}_{\theta'\sim\mathcal{G}}[\theta'])|| 
\end{align}
where $\mathcal{G}(\theta'; \theta, \gamma) \propto \exp\Big(-\Big( \frac{1}{m} \sum_{i=1}^{m} L(\theta')\Big) - \frac{\gamma}{2} ||\theta - \theta'||_{2}^{2}\Big)$ is the Gibbs distribution.
\end{definition}

It is prohibitive to compute the local entropy, as opposed to its gradient. We utilized the norm of the gradient of the local entropy to describe the sharpness at the minimum $\theta^*$, instead of using the local entropy. We utilize the EntropySGD algorithm, however instead of updating weight we compute the norm of the gradient.
\begin{algorithm}[H]
    \caption{$\mu_{LE}$}
    \textbf{Input:} $\theta^*$: final weights, $L$: Langevin iterations, $\gamma:$ scope, $\eta:$ step size, $\epsilon:$ noise level \\
    \textbf{Output:} $\mu_{LE}$
    \begin{algorithmic}
        \STATE $\theta', \mu \leftarrow \theta^*$
        \FOR{k = 1 to L}
            \STATE B $\leftarrow$ sample mini batch \\
            \STATE $g =  \nabla_{\theta'} L(\theta', B) - \gamma (\theta - \theta')$ \\
            \STATE $\theta' \leftarrow \theta' - \eta g + \sqrt{\eta} \epsilon \mathcal{N}(0, I)$ \\
            \STATE $\mu \leftarrow (1 - \alpha)\mu + \alpha \theta'$ \\
        \ENDFOR
        \RETURN $||\gamma (\theta^* - \mu)||$
    \end{algorithmic}
\end{algorithm}

\subsection{LPF based measure}

\begin{algorithm}[H]
    \caption{LPF}
    \textbf{Input:}  $\theta^*$: final weights, $\sigma$: standard deviation of Gaussian filter kernel, $M$: MC iterations \\
    \textbf{Output:} $(L \circledast K) (\theta^*)$
    \begin{algorithmic}
        \STATE $out \leftarrow 0.0$
        \FOR{k = 1 to M}
            \STATE $\tau = \mathcal{N}(0,\sigma I)$
            \STATE $out += L(\theta^* + \tau)$
        \ENDFOR
        \RETURN $out /= M$
    \end{algorithmic}
\end{algorithm}

\begin{rmk}
$\sigma$ is set to $\sigma\!=\!0.01$ in PAC Bayes measure  so that the deviation of loss is $\approx\!0.1$, as recommended by \cite{jiang2019fantastic} (see their discussion). To be consistent, $\sigma\!=\!0.01$ in LPF and $\epsilon\!=\!0.1$ in the $\epsilon$-sharpness measure.
\end{rmk}

\section{SENSITIVITY OF THE SHARPNESS MEASURES TO THE CHANGES IN THE CURVATURE OF THE SYNTHETICALLY GENERATED LANDSCAPES}
\label{sec:Sensitivity}

\begin{figure}[H]
    \centering
    \includegraphics[width=\textwidth]{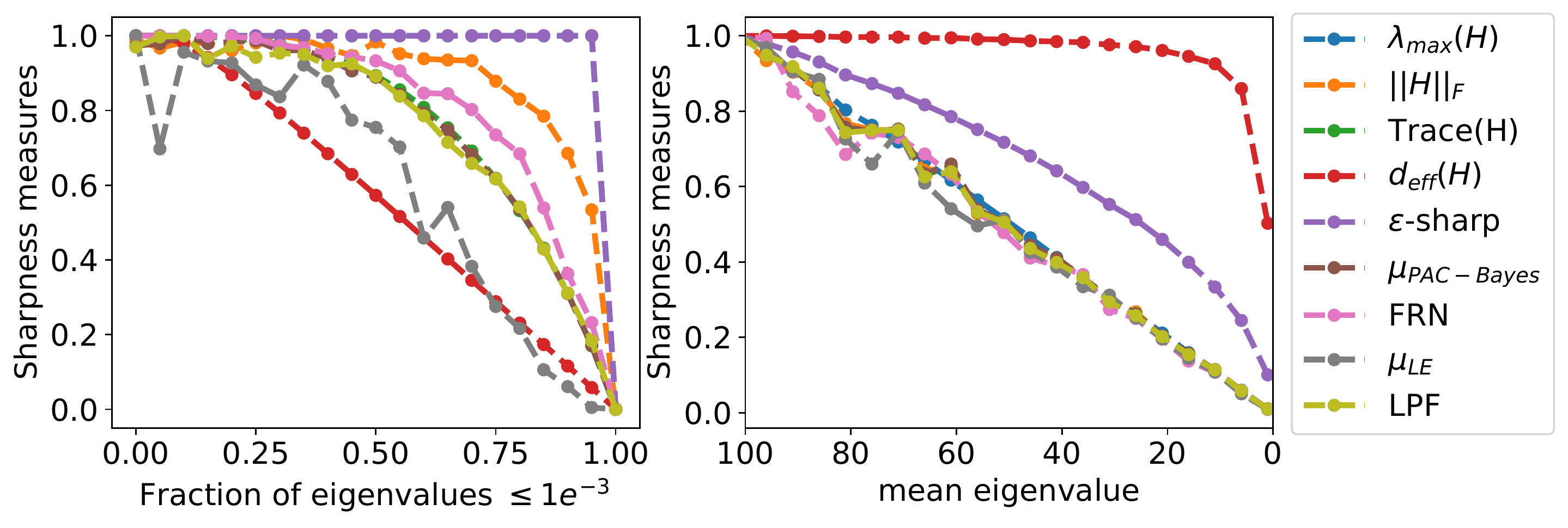}
    \caption{\textbf{Left}: The behavior of the normalized sharpness measures when the fraction of the eigenvalues of the Hessian below $1e-3$ increases from $0$ to $1$. \textbf{Right}: The behavior of the normalized sharpness measures when the mean eigenvalue of the Hessian is decreased from $100$ to $1$.}
    \label{fig:inv_cvx_func}
\end{figure}

We consider a quadratic minimization problem:
\begin{align}
    \min_{\theta} f(\theta), \:\:\:\text{where}\:\:\: f(\theta) = \frac{\theta^{T} H \theta}{2}. 
\end{align}
Note that $\nabla f = H \theta$, $\nabla^{2} f = H $ and $\theta^* = \argmin_{\theta} \frac{\theta^{T} H \theta}{2} = 0$. 

In the first experiment, we randomly sample the Hessian matrix $H$ of dimension $100$ and set its $K$ smallest eigenvalues uniformly in the interval [$1e-5$, $1e-3$]. As the value of $K$ is increased from $0$ to $100$, the loss surface becomes flatter. In the second experiment, we set the eigenvalues of Hessian $H$ uniformly as $\mathcal{U}(K - 0.10*K, K + 0.10*K)$, where K is the mean eigenvalue. Intuitively, as the value of $K$ is decreased from $100$ to $1$, the loss surface becomes flatter.

On the left plot in Figure~\ref{fig:inv_cvx_func} we show the value of the normalized sharpness measures against the fraction of eigenvalues that are $<1e-3$. Thus as we move on the $x$-axis of this plot from left to right, the number of directions along which the loss landscape is flat increases. In this case LPF is the second best measure, after $d_{eff}(H)$, where by a good measure we understand the one that is sensitive to the changes in the loss landscape. 
On the right plot in Figure~\ref{fig:inv_cvx_func} we show the value of the normalized sharpness measures against the mean eigenvalue of the Hessian. Thus as we move on the $x$-axis of this plot from left to right, the landscape along all directions becomes flatter. In this case all measure, except $\epsilon$-sharpness and $d_{eff}(H)$, are sensitive to the changes in the loss landscape. $d_{eff}(H)$ shows poor sensitivity to those changes. These experiments also well justify the choice of LPF based sharpness measure for the algorithm proposed in this paper.

\section{TRAINING DETAILS FOR SECTION~\ref{sec:sharp_v_gen}}
\label{sec:ShvGen}

\subsection{Sharpness vs Generalization (training details for Section~\ref{sec:sharpness_vs_generalization})}
\label{sec:training_details_sharp_vs_gen}
Following the experimental framework presented in ~\citep{jiang2019fantastic, dziugaite2020search}, we trained $2916$ ResNet18 models on CIFAR-10 data set by varying different model and optimizer hyper-parameters and $3$ random seeds. Each model was trained using cross entropy loss function and mSGD optimizer for $300$ epochs. The learning rate set to $0.1$ and dropped by a factor of $0.1$ at epoch $100$ and $200$. Since each model is trained with different hyper-parameters it is easy to overfit some models while under-fitting others. To mitigate this effect, we train each model until the cross entropy loss reaches the value of $\approx 0.01$. Any model that does not reach this threshold is discarded from further analysis. We compute Kendall ranking correlation coefficient between the hyper-parameters and generalization gap and report the results in Table \ref{tab:emp_order_resnet}). After convergence, we balance each network according to the normalization scheme presented in Section ~\ref{sec:bal} and compute sharpness measures using algorithms presented in Section~\ref{sec:FM}. All the models were trained on NVIDIA RTX8000, V100, and GTX1080 GPUs on our high performance computing cluster. The total computational time is $\sim 9000$ GPU hours. 

\begin{table}[!ht]
\vspace{0.1in}
    \centering
    \begin{tabular}{|c|c|c|c|c|c|c|c|}
        \hline
        Measure  & mo & width & wd & lr & bs & skip & bn \\ \hline
        Emp order &  -0.9712 & -0.6801 & -0.3135 & -0.7930 & 0.9877 & -0.2692 & -0.0955 \\ \hline
    \end{tabular}
    \caption{Ranking correlation between hyper-parameter and generalization gap. The correlation sign is consistent with our intuitive understanding.}
    \label{tab:emp_order_resnet}
\end{table}

\subsection{Sharpness vs Hyper-parameters (additional experimental results for Section~\ref{sec:sharpness_vs_generalization})}
\label{sec:correxp}
As highlighted in section ~\ref{sec:sharpness_vs_generalization}, we compute the Kendall ranking correlation between hyper-parameter and sharpness measures (Table~\ref{tab:sharp_v_hyperparam}) on ResNet18 models trained on CIFAR-10 data set as described section~\ref{sec:training_details_sharp_vs_gen}. We observe that momentum and weight decay are strongly negatively correlated to sharpness i.e increasing both hyper-parameters leads to flatter solution. It is also widely observed that increasing both these parameters also lead to lower generalization gap. Therefore, the table can provide us guidelines on how to design or modify deep architectures. This direction of research will be investigated in the future work. 

\begin{table}[!ht]
    \centering
    \begin{tabular}{|l|c|c|c|c|c|c|c|}
        \hline
        Measure                & mo & width & wd & lr & bs & skip & bn \\ \hline
        $\lambda_{\max}(H)$ & -0.891 & -0.063 & -0.291 & -0.692 & 0.981 & 0.263 & 0.996   \\
        $\|H\|_F$ & -0.930 & 0.029 & -0.474 & -0.826 & 0.994 & 0.218 & 0.996  \\
        Trace (H) & -0.942 & -0.127 & -0.381 & -0.745 & 0.984 & -0.199 & 0.987 \\
        $d_{eff}$ & -0.360 & -0.137 & -0.147 & -0.139 & 0.335 & -0.268 & 0.047 \\
        $\epsilon$-sharpness & -0.781 & 0.147 & -0.321 & -0.772 & 0.967 & 0.509 & 1.000   \\ 
        $\mu_{PAC-Bayes}$ & -0.994 & 0.981 & -0.669 & -0.971 & 0.996 & 0.322 & 0.996  \\ 
        FRN &  -0.824 & -0.226 & -0.037 & -0.545 & 0.855 & -0.605 & 1.000 \\ 
        $\mu_{entropy}$ & -0.723 & -0.174 & 0.246 & -0.352 & 0.718 & 0.613 & 0.950  \\ 
        $\mu_{LE}$ & -0.169 & 0.954 & -0.036 & -0.112 & 0.117 & 0.013 & 0.241 \\ \hline
        LPF & -0.994 & 0.874 & -0.767 & -0.934 & 0.998 & -0.543 & 0.954  \\ \hline
    \end{tabular}
    \caption{Kendall rank correlation coefficient between various sharpness measures (rows) and hyper-parameters (columns). }
    \label{tab:sharp_v_hyperparam}
\end{table}


\subsection{Training details and additional experimental results for Section~\ref{sec:sharpness_vs_generalization_under_noise}  (Sharpness versus generalization under data and label noise)}
\label{sec:noiseexp}

In order to evaluate the performance of sharpness measures to explain generalization in presence of data and label noise, we trained $10$ ResNet18 models with varying level of label noise and $20$ ResNet18 model with varying level of data noise on the CIFAR-10~\citep{CIFAR10} data set (Section~\ref{sec:sharpness_vs_generalization_under_noise} in the main paper). All models were trained for $350$ epochs using cross entropy loss and mSGD optimizer with a batch size of $128$, weight decay of $5e^{-4}$ and momentum set to $0.9$. The learning rate was set to $0.1$ and dropped by a factor of $0.1$ at epoch $150$ and $200$. The models were trained on NVIDIA RTX8000, V100, and GTX1080 GPUs on our high performance computing cluster. The total computational time is $\sim 600$ GPU hours. In Figure~\ref{fig:noise_corr}, we plot the values of normalized sharpness measures and generalization gap (averaged over $5$ seeds) for varying level of data noise. We also report the Kendall rank correlation coefficient in the figure title. 
\begin{figure}[!ht]
    \centering
    \includegraphics[width=0.95\textwidth]{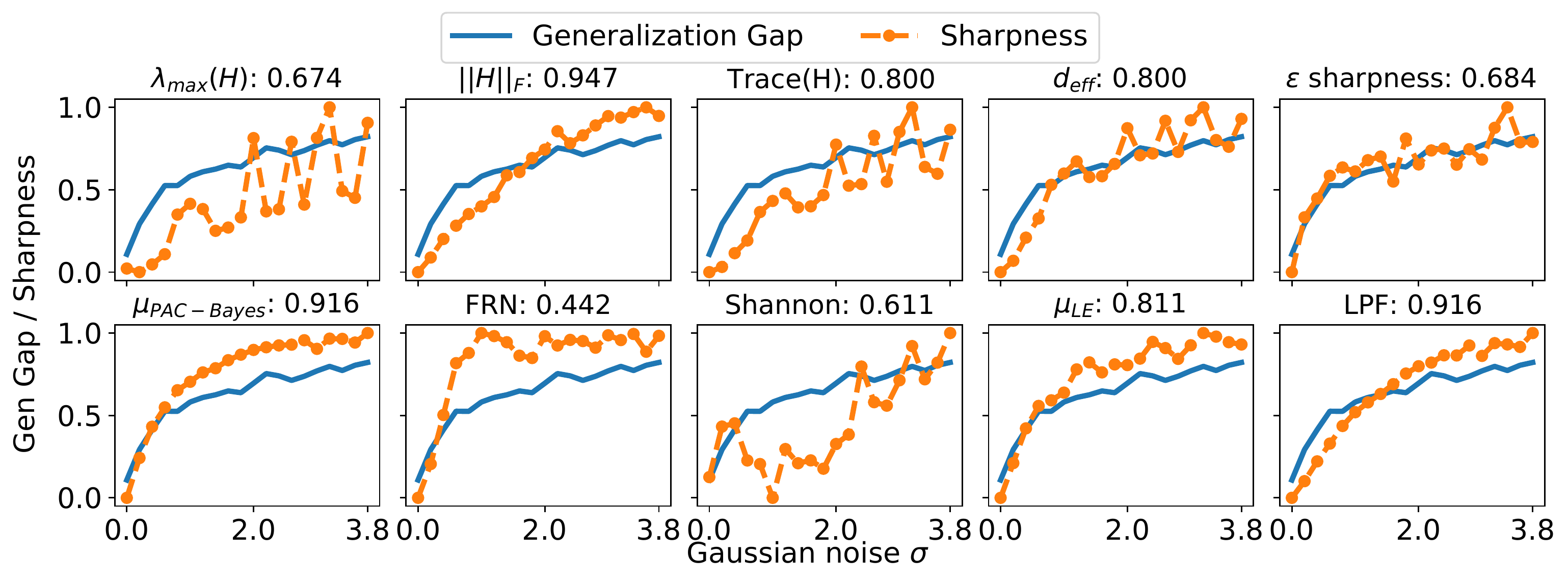}
    \caption{Normalized sharpness measures and generalization gap for varying levels of data noise. Kendall rank correlation coefficient between generalization gap and sharpness with increasing data noise are provided in the parenthesis of figure titles.}
    \label{fig:noise_corr}
\end{figure}

\subsection{Sharpness and double descent phenomenon}
\label{sec:sharpness_vs_generalization_under_double_decent}
\begin{figure}[!ht]
    \centering
    \includegraphics[width=0.95\textwidth]{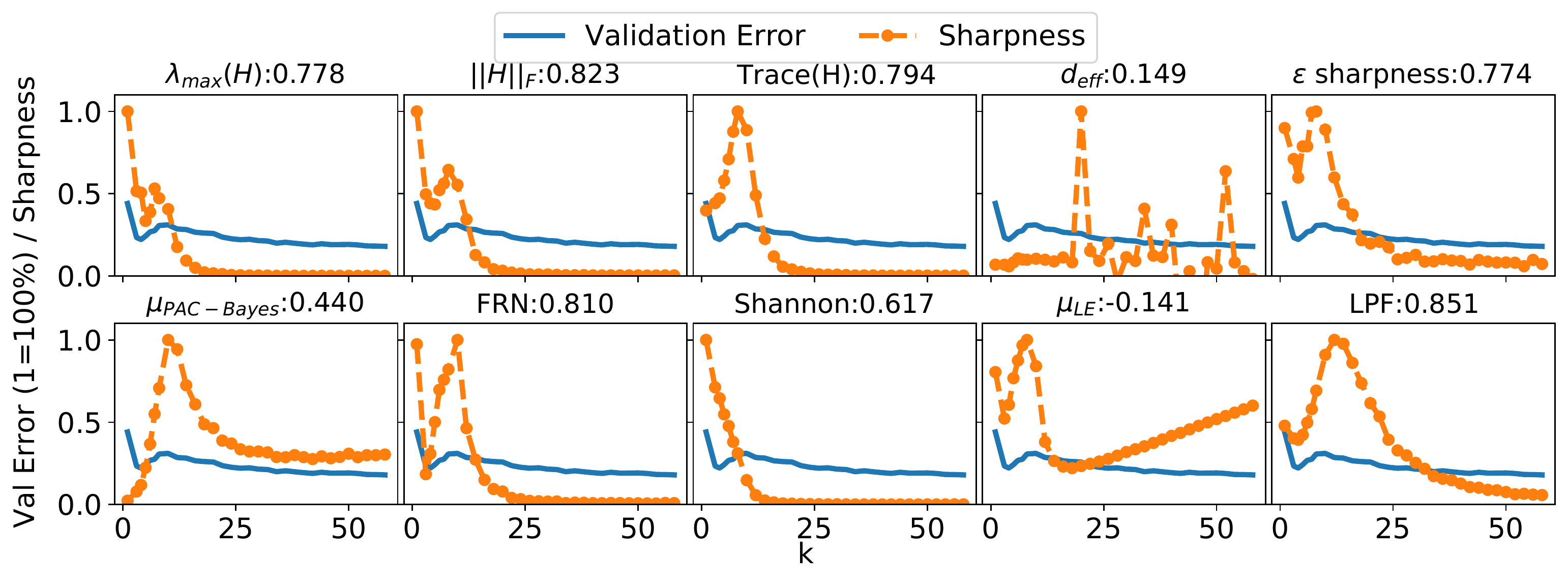}
    \caption{Normalized sharpness measures and validation error for ResNet18 when varying network width. Specifically, the widths of the consecutive layers of the network were set to $[k,2k,4k,8k]$, where $k=2,4,\dots,64$. Kendall rank correlation coefficient is provided in the titles above the figures.}
    \label{fig:resnet18_dd}
    \vspace{0.1in}
\end{figure}

As we increase the complexity of deep networks, we sometimes observe the double descent phenomenon, where increasing the network capacity may lead to the temporal increase of the test loss. We evaluate the sharpness measures against the double descent behavior for ResNet18 model. We follow the experimental framework presented in ~\citep{Nakkiran2020Deep} and set the widths of the consecutive layers of the ResNet18 model as $[k,2k,4k,8k]$. The value of k is varied in the range $[1,64]$. Note that for a standard ResNet18 model $k=64$. The models were trained on CIFAR-10 data set for $4000$ epochs with $20\%$ label noise using an Adam optimizer~\citep{DBLP:journals/corr/KingmaB14} with a batch size of $128$ and a constant learning rate set to $1e^{-4}$. All the models were trained on NVIDIA RTX8000, V100, and GTX1080 GPUs on our high performance computing cluster. The total compute time is $\sim 300$ GPU hours. In Figure~\ref{fig:resnet18_dd} we plot the validation error and normalized sharpness measures when varying network width. It can be observed that most of the flatness measures can track the double descent phenomenon, including $\lambda_{\max}(H)$, $\|H\|_F$, Trace (H), $\epsilon$-sharpness, FRN, and LPF. At the same time, LPF based measure admits the highest Kendall rank correlation coefficient from among all the flatness measures.

\section{TRAINING DETAILS FOR SECTION~\ref{sec:experiments}}\label{sec:optexp}
\subsection{Codes}
\label{sec:Codes}
We coded all our experiments in PyTorch~\citep{paszke2017automatic}. In all the experiments, we utilized the code for the SAM optimizer~\citep{SAM} available at \href{https://github.com/davda54/sam}{https://github.com/davda54/sam} that we treated as a baseline code for developing Adaptive Smooth with Denoising (ASO) and LPF-SGD. In terms of mSGD, this optimizer is included in the PyTorch environment. 
For E-SGD, we utilized public codes released by ~\citet{pittorino2021entropic} available at \href{https://openreview.net/forum?id=xjXg0bnoDmS}{https://openreview.net/forum?id=xjXg0bnoDmS}. Finally, the codes for our method, LPF-SGD, will be open-sourced and publicly released.
\subsection{Image Classification}\label{sec:IC}
\subsubsection{Image Classification: ResNets on CIFAR 10/100 data set with no data augmentations} 
\label{subsec:resnet}
Image classification models such as ResNets~\citep{he2016deep} are prone to extreme over-fitting when trained without any data augmentation. The flatness based optimizers such as SAM and E-SGD prevent model overfitting by seeking flat minimizers. Therefore, we first compare the performance of mSGD, E-SGD, ASO, SAM, and LPF-SGD on ResNet architectures without performing any data augmentation. Specifically, we trained ResNet18 model available in torchvision~\citep{paszke2017automatic} on TinyImageNet~\citep{ILSVRC} and ImageNet~\citep{ILSVRC} data sets, and a modified version of ResNet18,50,101~\citep{he2016deep} available at ~\href{https://github.com/kuangliu/pytorch-cifar}{https://github.com/kuangliu/pytorch-cifar} on CIFAR-10 and CIFAR-100~\citep{CIFAR10} data sets. The modification was small and was only done to accommodate $32 \times 32$ image sizes in CIFAR data set. We also train a LeNet~\citep{LeCun:1989:BAH:1351079.1351090} model available at ~\href{https://github.com/pytorch/examples/blob/master/mnist/main.py}{https://github.com/pytorch/examples/blob/master/mnist/main.py} on MNIST~\citep{MNIST} data set. 

We set the radius for SAM to $\rho=0.05$ and keep it fixed for the entire training procedure, as suggested by the authors (note that the authors did extreme grid search and found little to no difference when changing parameter $\rho$). For E-SGD, we performed a grid search over $\gamma_{0} = \{0.1, 0.5, 0.05\}$, $\gamma_{1} = \{0.0001, 0.0005, 0.005\}$ and $\eta = \{0.05,0.01,0.1 \}$. The parameters $(\gamma_0, \gamma_1, \eta) = (0.5,0.0001,0.05)$ were found to be the best performing parameters. The SGLD iterations ($M$) were set to $5$. The learning rate of the base optimizer was also increased by a factor of $M$. For ASO, the parameter $a$ was tuned by performing a grid search over the set $\{ 0.001, 0.005, 0.009, 0.0375, 0.07 \}$. We find $a = 0.009$ to be the best performing hyper-parameter as suggested by the authors. The hyper-parameters common to mSGD, E-SGD, ASO, SAM, and LPF-SGD optimizers are provided in the Table~\ref{tab:exp1_1}, while the individual hyper-parameters for each optimizer are provided in Table~\ref{tab:exp1_2}. Note that for E-SGD the total number of training epochs and the epoch at which we drop the learning rate are divided by $M=5$, i.e.: the total number of stochastic gradient Langevin dynamics steps per iteration. The models were trained on NVIDIA RTX8000, V100, and GTX1080 GPUs on our high performance computing cluster. The total computational time is $\sim 1500$ GPU hours. The plots showing epoch vs error curves are captured in Figure~\ref{fig:exp1_1} and Figure~\ref{fig:exp1_2}.

\begin{table}[!ht]
    \setlength\tabcolsep{2.2pt}
    \renewcommand{\arraystretch}{1.3}
    \centering
    \begin{tabular}{|c|c|c|c|c|c|c|}
        \hline
        Dataset & Model & BS & WD & MO & Epochs  & LR (Policy)\\ \hline
        MNIST & LeNet & 128 & $5e^{-4}$ & 0.9 & 150 & 0.01 (x 0.1 at ep=[50,100])  \\ 
        CIFAR10, 100 & ResNet-18, 50, 101 & 128 & $5e^{-4}$ & 0.9 & 200 & 0.1 (x 0.1 at ep=[100,120])  \\ 
        TinyImageNet & ResNet-18 & 128 & $1e^{-4}$ & 0.9 & 100 & 0.1 (x 0.1 at ep=[30, 60, 90]) \\ 
        ImageNet & ResNet-18 & 256 & $1e^{-4}$ & 0.9 & 100 & 0.1 (x 0.1 at ep=[30, 60, 90]) \\ \hline
    \end{tabular}
    \caption{Training hyper-parameters common to all optimizers used for obtaining Table~\ref{tab:exp1}. BS: batch size, WD: weight decay, and MO: momentum coefficient.}
    \vspace{10pt}
    \label{tab:exp1_1}
\end{table}

\begin{table}[!ht]
    \setlength\tabcolsep{2.2pt}
    \renewcommand{\arraystretch}{1.3}
    \centering
    \begin{tabular}{|c|c|c|c|c|c|}
        \hline
        \multirow{2}{*}{Dataset} & \multirow{2}{*}{Model} & E-SGD & ASO  & SAM & LPF-SGD \\ \cline{3-6}
        &                       & ($\gamma_0$, $\gamma_1$, $\eta$, M)& a & $\rho$  & (M, $\gamma$ (policy)) \\ \hline 
        MNIST & LeNet           & (0.5,0.0001,0.05,5) & 0.009 & 0.05                &   (1,0.001  (fixed))   \\
        CIFAR & ResNet18,50,101 & (0.5,0.0001,0.05,5) & 0.009 & 0.05                &   (1,0.002  (fixed)) \\ 
        TinyImageNet& ResNet18  & (0.5,0.0001,0.05,5) & 0.009 & 0.05                &   (1,0.001  (fixed)) \\ 
        ImageNet    & ResNet18  & (0.5,0.00001,0.05,5) & 0.009 & 0.05    &   (1,0.0005 (fixed)) \\ \hline
    \end{tabular}
    \caption{Summary of E-SGD, ASO, SAM and LPF-SGD hyper-parameters used for obtaining Table~\ref{tab:exp1}.}
    \label{tab:exp1_2}
\end{table}

Finally, after convergence, we balance our network using the normalization scheme described in section~\ref{sec:bal} and compute various sharpness measures on the best performing model on CIFAR-10 and CIFAR-100 data set for baseline mSGD and top two performing optimizers (SAM and LPF-SGD). 
Table~\ref{tab:exp1_sharp} shows normalized sharpness measures and the corresponding validation error. Note that LPF-SGD leads to a lower value of the LPF based sharpness measure and a smaller test error. 

\begin{table}[!ht]
    \vspace{0.1in}
    \setlength\tabcolsep{3pt}
    \renewcommand{\arraystretch}{1.3}
    \centering
    \begin{tabular}{|c|c|c|c|c|c|c|c|c|c|}
        \hline
        Data & Model & Opt & $||H||_{F}$ & $\epsilon$-sharp & $\mu_{PAC}$ & FRN & $\mu_{entropy}$ & LPF & val-err \\ \hline
        \multirow{9}{*}{CIFAR10}& \multirow{3}{*}{ResNet18}& mSGD    & 1.00&0.52&1.00&1.00&1.00&1.00&11.49  \\
                                                        & & SAM     & 0.34&0.36&0.70&0.45&0.61&0.40&10.00 \\
                                                        & & LPF-SGD & 0.71&1.00&0.58&0.46&0.21&0.17&9.04 \\ \cline{2-10}
        
        &\multirow{3}{*}{ResNet50} & mSGD & 1.00&0.43&1.00&1.00&1.00&1.00&10.21 \\ 
        & & SAM                          & 0.37&0.41&0.66&0.64&0.60&0.45&8.81\\ 
        & & LPF-SGD                         & 0.79&1.00&0.88&0.56&0.24&0.28&8.60\\ \cline{2-10}
        
        & \multirow{3}{*}{ResNet101} & mSGD  & 1.00&0.59&1.00&1.00&1.00&1.00&9.49 \\ 
        &                            & SAM  & 0.36&0.46&0.70&0.59&0.67&0.51&8.33 \\ 
        &                            & LPF-SGD & 0.75&1.00&0.99&0.49&0.33&0.34&8.69 \\ \hline
    
        \multirow{9}{*}{CIFAR100} & \multirow{3}{*}{ResNet18}& mSGD  & 0.18&0.64&1.00&0.29&0.65&1.00&38.29  \\ 
        &                                                    & SAM  & 0.09&0.47&0.78&0.20&0.50&0.52&36.17 \\ 
        &                                                    & LPF-SGD & 1.00&1.00&0.59&1.00&1.00&0.90&30.02  \\ \cline{2-10}
    
        & \multirow{3}{*}{ResNet50} & mSGD  & 0.54&0.50&1.00&0.62&1.00&1.00&35.55 \\ 
        &                           & SAM  & 0.18&0.54&0.78&0.27&0.64&0.46&33.15  \\
        &                           & LPF-SGD & 1.00&1.00&0.53&1.00&0.65&0.41&30.64  \\ \cline{2-10}
        
        & \multirow{3}{*}{ResNet101} & mSGD   & 1.00&0.56&1.00&1.00&0.34&1.00&32.73\\  
        &                            & SAM   & 0.46&0.43&0.64&0.55&0.25&0.41&30.70 \\  
        &                            & LPF-SGD  & 0.59&1.00&0.44&0.44&1.00&0.12&29.89 \\ \hline
        
    \end{tabular}
    \caption{Normalized sharpness measures and validation error for ResNet18,50,101 models trained on CIFAR-10 and CIFAR-100 data sets using standard SGD, SAM and LPF-SGD.}
    \label{tab:exp1_sharp}
\end{table}

\newpage

\begin{figure}[H]
    \centering
    \includegraphics[width=\textwidth]{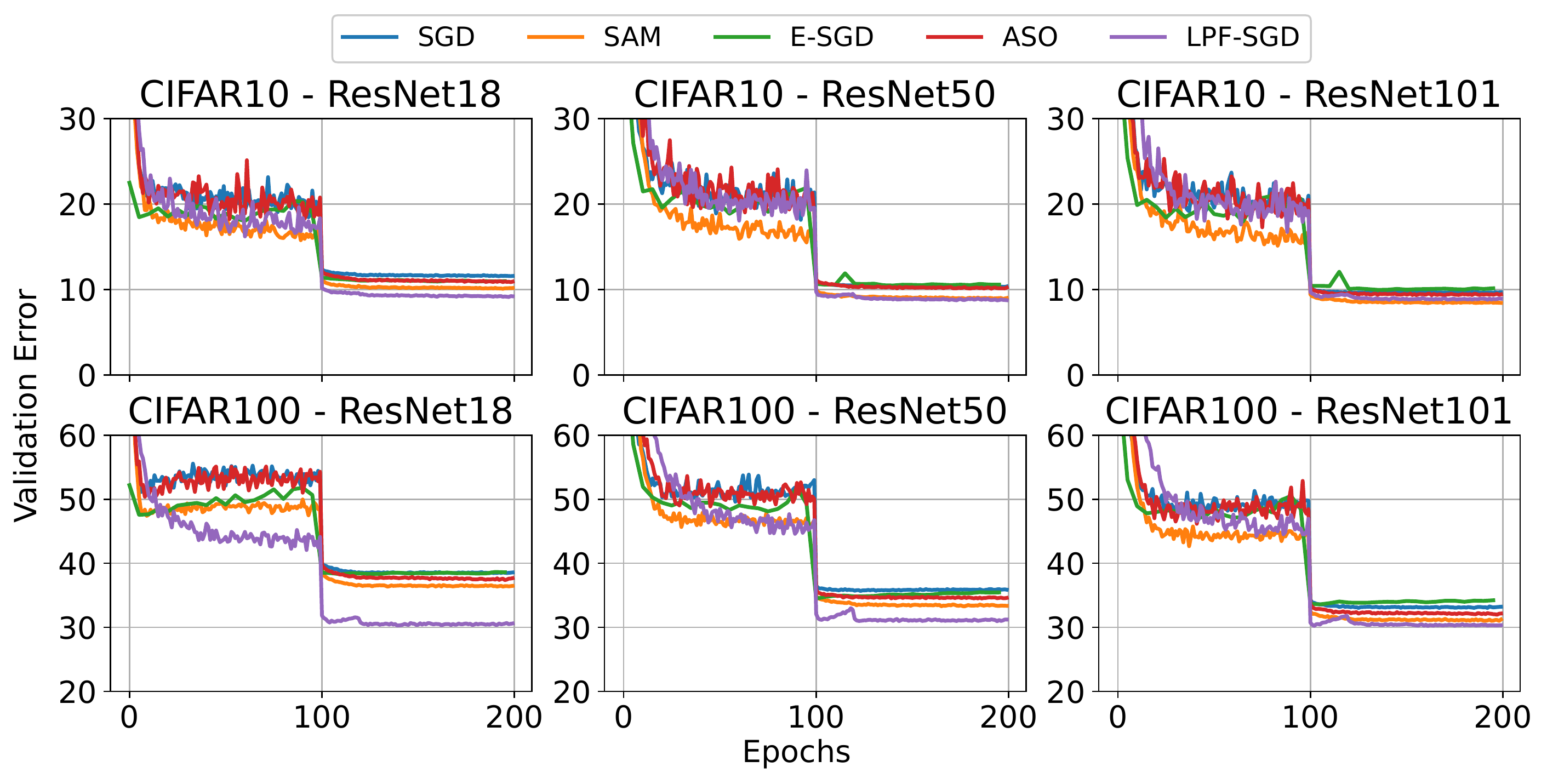}
    \caption{Validation error vs epochs for ResNet18 (left), ResNet50 (middle), and ResNet101 (right) models trained on CIFAR-10 (top) and CIFAR-100 (bottom) data sets using mSGD, SAM, and LPF-SGD. }
    \label{fig:exp1_1}
\end{figure}
\begin{figure}[H]
    \centering
    \includegraphics[width=\textwidth]{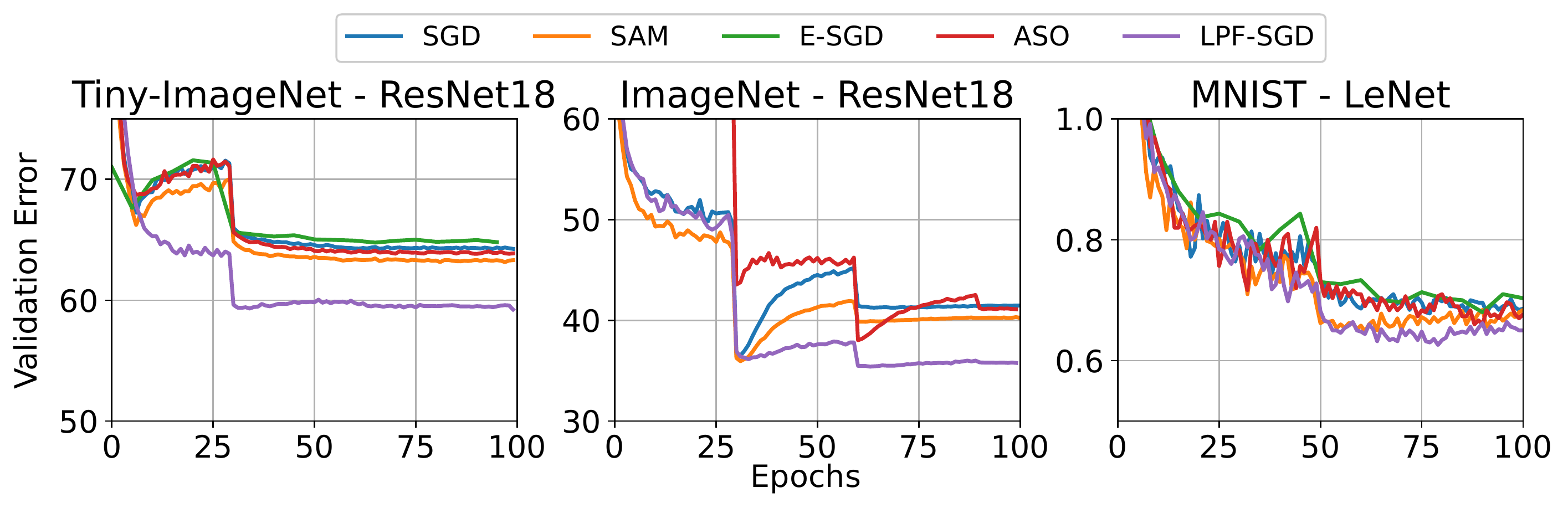}
    \caption{Validation error vs epochs for TinyImageNet (left), ImageNet (middle), and MNIST (right) data sets. TinyImage and ImageNet data sets was used to train the ResNet18 model while MNIST data set was used to train the LeNet model.}
    \label{fig:exp1_2}
\end{figure}

\subsubsection{Image Classification: ResNet18 on ImageNet data set with basic data augmentation}\label{sec:resnet2}
In the second set of experiments, we train a ResNet18 model on ImageNet data set utilizing the standard training hyper-parameters available in pytorch~\citep{paszke2017automatic}. The hyper-parameters for mSGD, ASO, SAM, and LPF-SGD optimizers are provided in the Table~\ref{tab:exp_imagenet_1}. The models were trained on NVIDIA RTX8000, V100, and GTX1080 GPUs on our high performance computing cluster. The total computational time is $\sim 1500$ GPU hours.
\begin{table}[!ht]
    \vspace{0.1in}
    \setlength\tabcolsep{2.2pt}
    \renewcommand{\arraystretch}{1.3}
    \centering
    \begin{tabular}{|c|c|c|c|c|c|c|c|c|c|}
        \hline
        \multirow{2}{*}{Dataset} & \multirow{2}{*}{Model} & \multirow{2}{*}{BS} & \multirow{2}{*}{WD} & \multirow{2}{*}{MO} & \multirow{2}{*}{Epochs}  & \multirow{2}{*}{LR (Policy)} & ASO  & SAM & LPF-SGD \\ \cline{8-10}
        &   & & & & & & a & $\rho$     & (M, $\gamma$, $\alpha$) \\ \hline 
        ImageNet & ResNet-18 & 256 & $1e^{-4}$ & 0.9 & 100 & 0.1 (x 0.1 at ep=[30, 60, 90]) & 0.009 & 0.05   & (8, 0.0001, 5)\\ \hline
    \end{tabular}
    \vspace{-0.1in}
    \caption{Training hyper-parameters used for obtaining Table~\ref{tab:imgnet}. BS: batch size, WD: weight decay, and MO: momentum coefficient.}
    \label{tab:exp_imagenet_1}
\end{table}
\subsubsection{Image Classification: Wide ResNets, ShakeShake and PyramidNets on CIFAR 10/100 data set with various data augmentations schemes}
\label{sec:widenet}

In the third set of experiments, we trained WRN16-8 and WRN28-10~\citep{zagoruyko2016wide} models available at ~\href{https://github.com/xternalz/WideResNet-pytorch}{https://github.com/xternalz/WideResNet-pytorch}, ShakeShake (26 2x96d)~\citep{gastaldi2017shake} available at ~\href{https://github.com/hysts/pytorch
_shake\_shake}{https://github.com/hysts/pytorch\_shake\_shake}, and PyramidNet-110($\alpha= 270$)~\citep{Han_2017_CVPR} and PyramidNet-272($\alpha=200$)~\citep{Han_2017_CVPR} models available at ~\href{https://github.com/dyhan0920/PyramidNet-PyTorch}{https://github.com/dyhan0920/PyramidNet-PyTorch}. 
We utilized three progressively increasing augmentation schemes: basic (random cropping and horizontal flipping), basic + cutout~\citep{devries2017improved}, and basic + cutout + auto-augmentation~\citep{Cubuk_2019_CVPR}. The cutout scheme is available at~\href{https://github.com/davda54/sam}{https://github.com/davda54/sam} and the auto-augmentation scheme is available at ~\href{https://github.com/4uiiurz1/pytorch-auto-augment}{https://github.com/4uiiurz1/pytorch-auto-augment}. 

The hyper-parameters for SAM, E-SGD, and ASO were tuned as described in the Subsection~\ref{subsec:resnet}. For LPF-SGD, we perform a grid search over the radius $\gamma_0$ and number of MC iterations ($M$) for experiments with WRN16-8 and WRN28-10 model and use the best performing hyper-parameters ($M=8$, $\gamma_0 = 5e^{-4}$, and $\alpha=15$) for the remaining experiments. Table ~\ref{tab:exp2_1} shows various hyper-parameters common to mSGD, E-SGD, ASO, SAM, and LPF-SGD optimizers, and Table~\ref{tab:exp2_2} shows individual hyper-parameters for LPF-SGD optimizer (for SAM hyperparameter $\rho$ is fixed to $0.05$, for ASO hyperparameter $a$ is fixed to $0.009$, and for E-SGD hyper-parameters ($\gamma_0$, $\gamma_1$, $\eta$, L)  are fixed to $(0.5,0.0001,0.05,5)$, and thus we do not report these hyper-parameters in the table).
\begin{table}[!ht]
    \centering
    \setlength\tabcolsep{3pt}
    \begin{tabular}{|p{2cm}|c|c|c|c|c|c|}
        \hline
        \multirow{2}{*}{Model}   &   \multirow{2}{*}{BS}  & \multirow{2}{*}{WD} & \multirow{2}{*}{MO} & \multirow{2}{*}{Epochs} & \multicolumn{2}{c|}{LR(Policy)}  \\ \cline{6-7}
                &       &    &    &        & CIFAR-10 & CIFAR-100 \\ \hline
        WRN16-8     &   128 & $5e^{-4}$ & 0.9 & 200 & \multicolumn{2}{c|}{0.1($\times$ 0.2 at [60,120,160])} \\ \hline
        WRN28-10    &   128 & $5e^{-4}$ & 0.9 & 200 & \multicolumn{2}{c|}{0.1($\times$ 0.2 at [60,120,160])} \\ \hline
        ShakeShake (26 2x96d)  &   128 & $1e^{-4}$ & 0.9 & 1800 & \multicolumn{2}{c|}{0.2(cosine decrease)} \\ \hline
        PyNet110    &   128 & $1e^{-4}$ & 0.9 & 300 &  0.1($\times$ 0.1 at [100,150]) & 0.5($\times$ 0.1 at [100,150])\\ \hline
        PyNet272    &   128 &  $1e^{-4}$ & 0.9 & 300 &  \multicolumn{2}{c|}{0.1($\times$ 0.1 at [100,150])}\\ \hline
    \end{tabular}
    \caption{Training hyper-parameters common to all optimizers used to obtain Table~\ref{tab:exp2}. BS: batch size, WD: weight decay, MO: momentum coefficient, and LR: learning rate.}
    \label{tab:exp2_1}
\end{table}
\begin{table}[!ht]
    \centering
    \setlength\tabcolsep{3pt}
    \begin{tabular}{|c|c|c|c|c|c|c|}
        \hline
        \multirow{2}{*}{Model} & \multirow{2}{*}{Aug} & \multirow{2}{*}{M} & \multicolumn{2}{c|}{CIFAR-10} & \multicolumn{2}{c|}{CIFAR-100} \\ \cline{4-7}
         & & & $\gamma_0$ & $\alpha$ (policy) & $\gamma_0$ & $\alpha$ (policy) \\ \hline
        \multirow{3}{*}{WRN16-8}    &   Basic           & 8 &  0.0005   & 15 & 0.0005 & 15 \\ \cline{2-7}
                                    &   Basic+Cut       & 8 &  0.0005   & 15 & 0.0005 & 15 \\ \cline{2-7}
                                    &   Basic+Cut+AA    & 8 &  0.0005   & 15 & 0.0005 & 15 \\ \hline
        \multirow{3}{*}{WRN28-10}   &   Basic           & 8 &  0.0005   & 35 & 0.0005 & 25 \\ \cline{2-7}
                                    &   Basic+Cut       & 8 &  0.0005    & 35 & 0.0005 & 25 \\ \cline{2-7}
                                    &   Basic+Cut+AA    & 8 &  0.0005    & 35 & 0.0007 & 15 \\ \hline
        \multirow{3}{*}{ShakeShake 26 2x96d} &   Basic           & 8 &  0.0005   & 15 & 0.0005 & 15 \\ \cline{2-7}
&   Basic+Cut       & 8 &  0.0005   & 15 & 0.0005 & 15 \\ \cline{2-7}
&   Basic+Cut+AA    & 8 &  0.0005   & 15 & 0.0005 & 15 \\ \hline
        \multirow{3}{*}{PyNet110}   &   Basic           & 8 &  0.0005   & 15 & 0.0005 & 15 \\ \cline{2-7}
&   Basic+Cut       & 8 &  0.0005   & 15 & 0.0005 & 15 \\ \cline{2-7}
&   Basic+Cut+AA    & 8 &  0.0005   & 15 & 0.0005 & 15 \\ \hline
        \multirow{3}{*}{PyNet272}   &   Basic           & 8 &  0.0005   & 15 & 0.0005 & 15 \\ \cline{2-7}
&   Basic+Cut       & 8 &  0.0005   & 15 & 0.0005 & 15 \\ \cline{2-7}
&   Basic+Cut+AA    & 8 &  0.0005   & 15 & 0.0005 & 15 \\ \hline
    \end{tabular}
    \caption{Hyper-parameters for LPF-SGD optimizer.}
    \label{tab:exp2_2}
\end{table}
In Figures~\ref{fig:exp2_1}, ~\ref{fig:exp2_2}, ~\ref{fig:exp2_3}, ~\ref{fig:exp2_4}, and ~\ref{fig:exp2_5} we provide error vs epoch curves. All the models were trained on NVIDIA RTX8000, V100, and GTX1080 GPUs on our high performance computing cluster. The total computational time is $\sim 6000$ GPU hours.

\newpage

\begin{figure}[H]
    \centering
    \includegraphics[width=\textwidth]{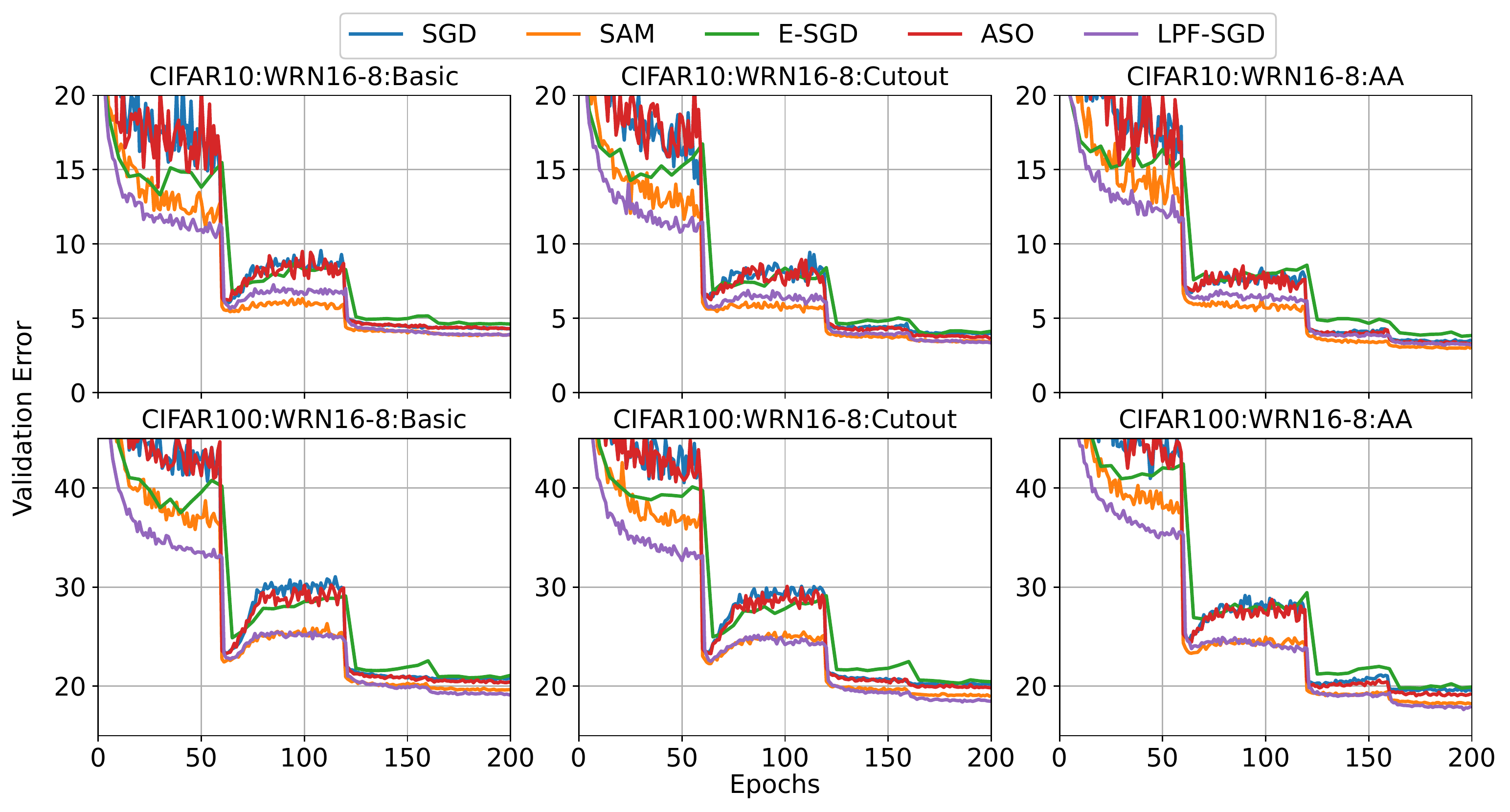}
    \caption{Validation error vs epochs for WideResNet 16-8 model trained on CIFAR-10 (top) and CIFAR-100 (bottom) data sets with Basic (left), Basic + Cutout (middle) and Basic+AutoAugmentation+Cutout (right) augmentation schemes using mSGD, SAM, E-SGD (E-SGD), ASO, and LPF-SGD optimization algorithms.}
    \label{fig:exp2_1}
\end{figure}
\begin{figure}[H]
    \centering
    \includegraphics[width=\textwidth]{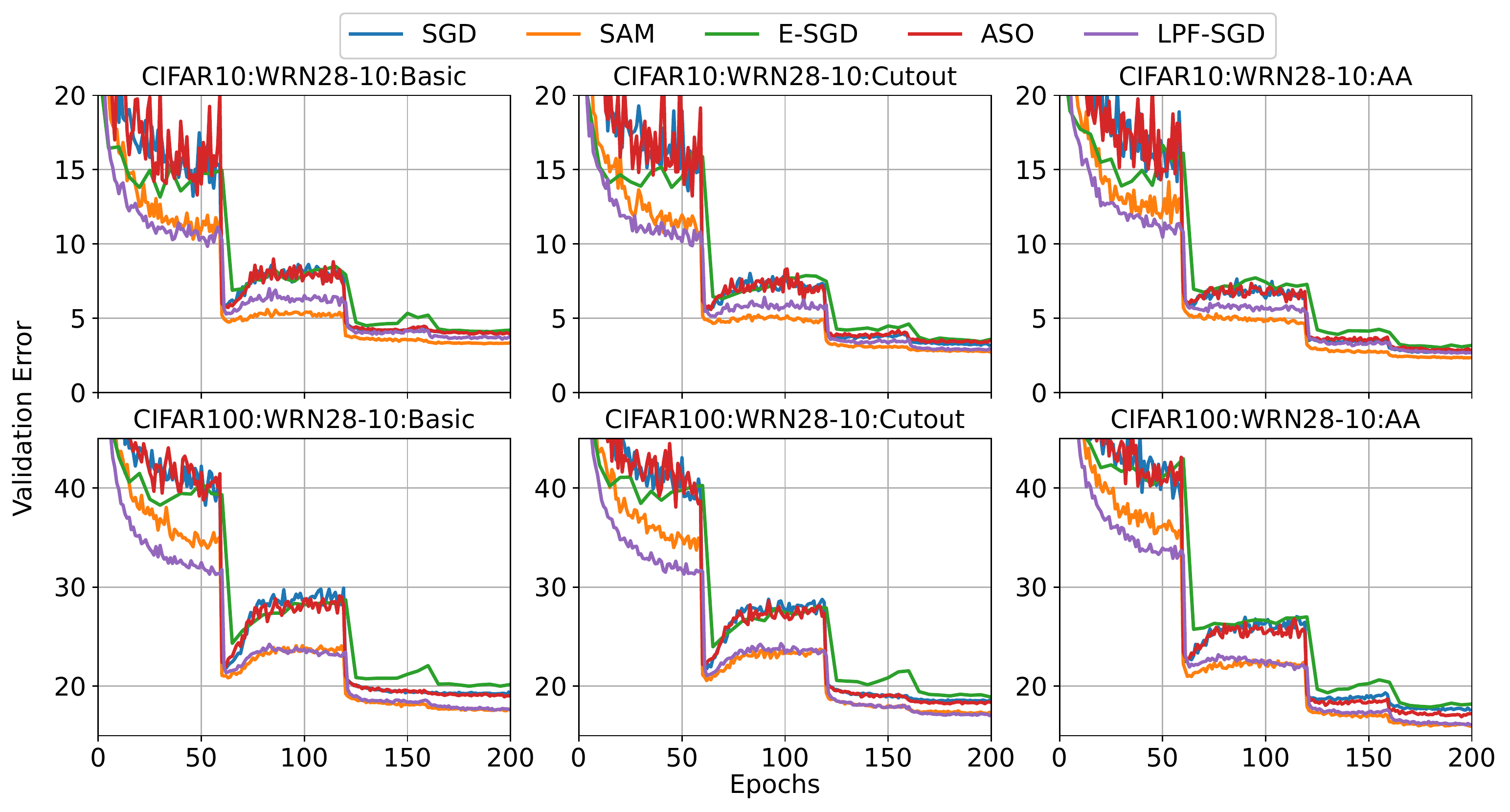}
    \caption{Validation error vs epochs for WideResNet 28-10 model trained on CIFAR-10 (top) and CIFAR-100 (bottom) data sets with Basic (left), Basic + Cutout (middle) and Basic+AutoAugmentation+Cutout (left) augmentation schemes using mSGD, SAM, E-SGD, ASO, and LPF-SGD optimization algorithms.}
    \label{fig:exp2_2}
\end{figure}
\begin{figure}[H]
    \centering
    \includegraphics[width=\textwidth]{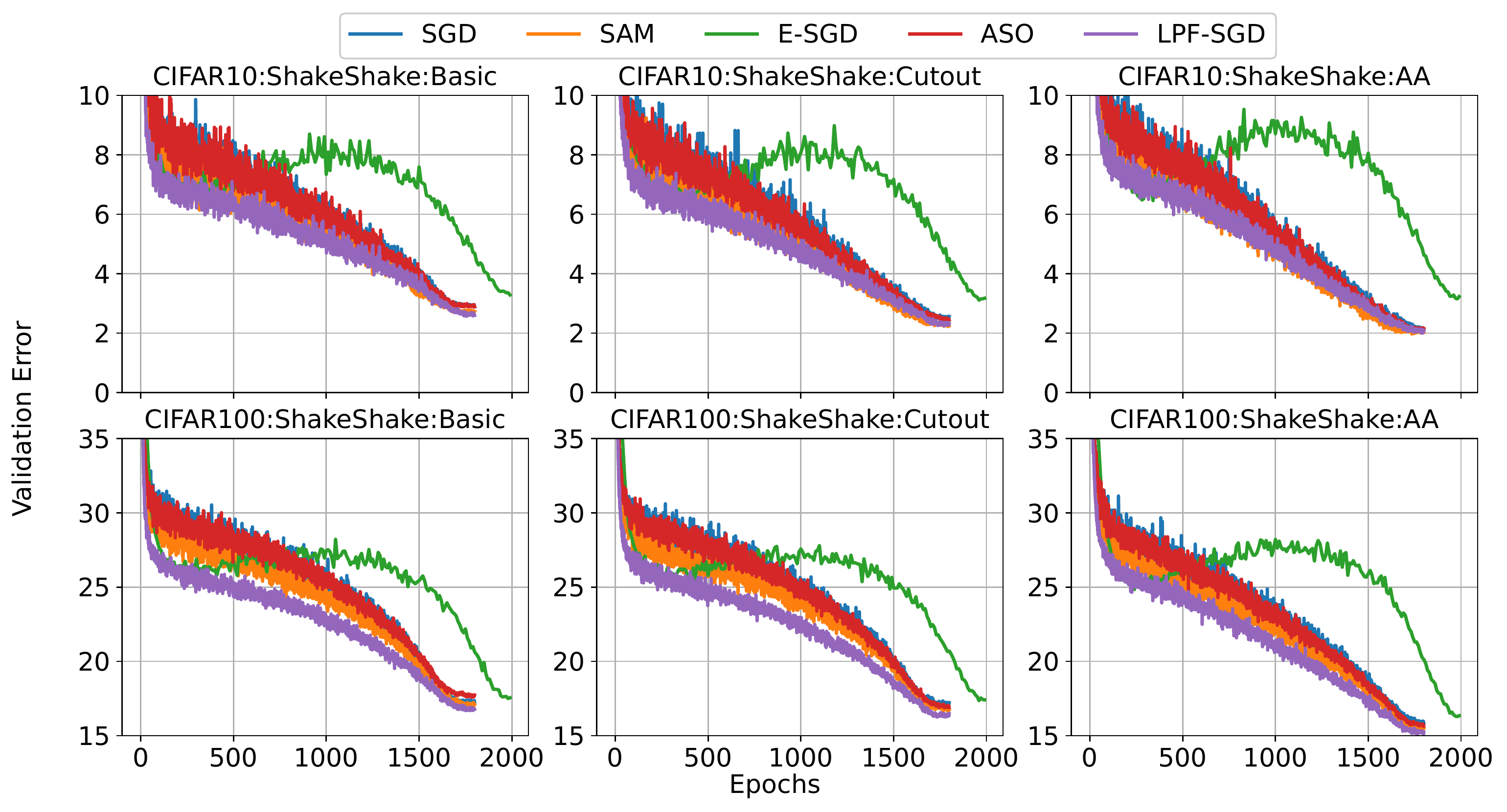}
    \caption{Validation error vs epochs for ShakeShake (26 2x96d) model trained on CIFAR-10 (top) and CIFAR-100 (bottom) data sets with Basic (left), Basic + Cutout (middle) and Basic+AutoAugmentation+Cutout (right) augmentation schemes using mSGD, SAM, E-SGD, ASO, and LPF-SGD optimization algorithms.}
    \label{fig:exp2_3}
\end{figure}
\begin{figure}[H]
    \centering
    \includegraphics[width=\textwidth]{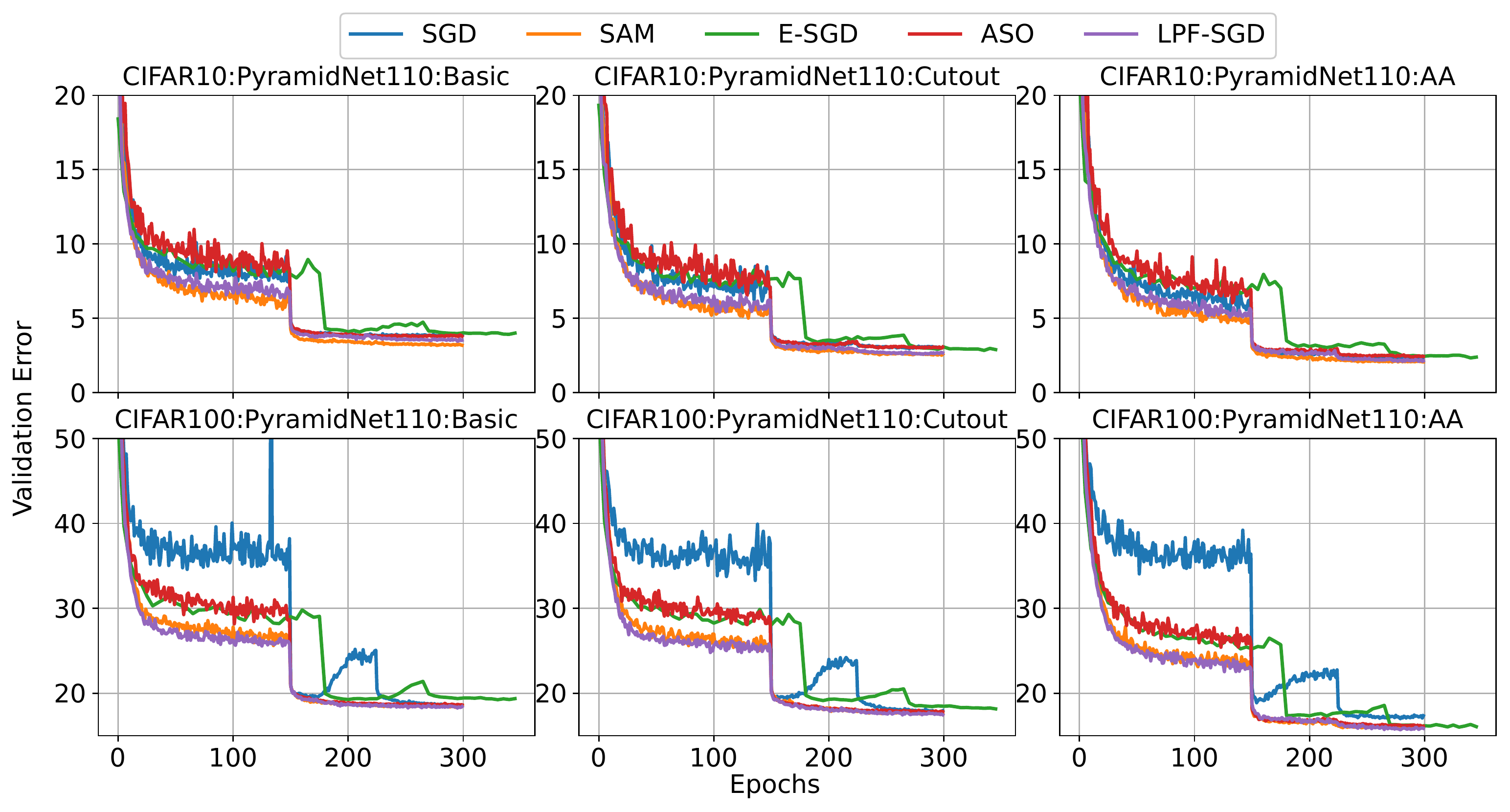}
    \caption{Validation error vs epochs for PyramidNet110 ($\alpha=270$) model trained on CIFAR-10 (top) and CIFAR-100 (bottom) data sets with Basic (left), Basic + Cutout (middle) and Basic+AutoAugmentation+Cutout (right) augmentation schemes using mSGD, SAM, E-SGD, ASO, and LPF-SGD optimization algorithms.}
    \label{fig:exp2_4}
\end{figure}
\begin{figure}[H]
    \centering
    \includegraphics[width=\textwidth]{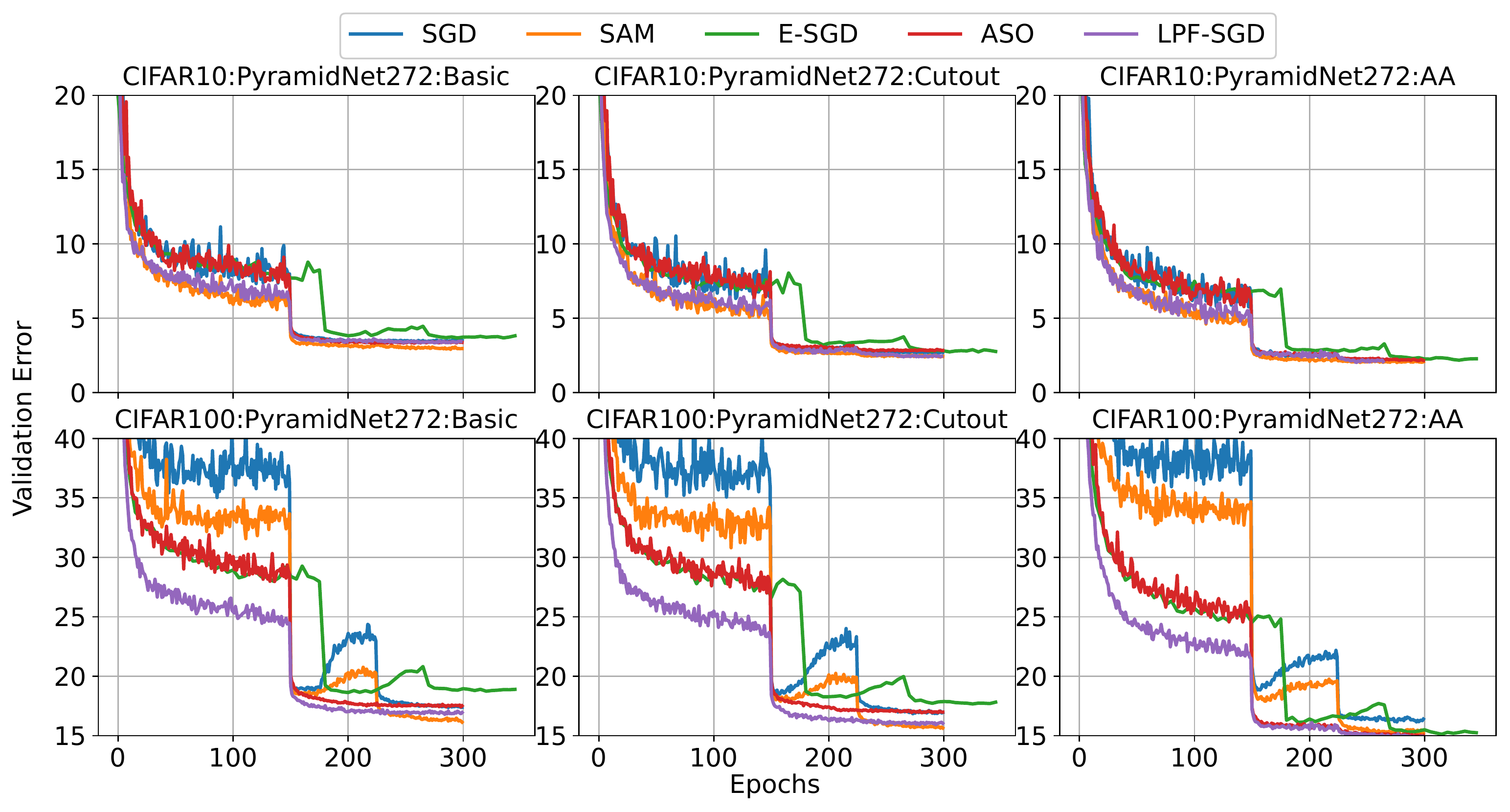}
    \caption{Validation error vs epochs for PyramidNet272 ($\alpha=200$) model trained on CIFAR-10 (top) and CIFAR-100 (bottom) data sets with Basic (left), Basic + Cutout (middle) and Basic+AutoAugmentation+Cutout (right) augmentation schemes using mSGD, SAM, E-SGD, ASO, and LPF-SGD optimization algorithms.}
    \label{fig:exp2_5}
\end{figure}

\subsubsection{Comparison with ASAM}\label{sec:asam}
The hyper-parameter $\rho$ for ASAM was tuned by performing grid search over the set $\{ 0.1, 0.5, 1.0\}$ and $\rho = 1.0$ was found to be the best performing parameter. The ASAM codes are available at \href{https://github.com/davda54/sam}{https://github.com/davda54/sam}.

\subsection{Machine Translation}\label{sec:MT}
We train a mini-transformer model based on ~\citet{vaswani2017attention} with $3$ encoder layers and $3$ decoder layers to perform German to English translation on WMT 2014 data set~\citep{bojar-etal-2014-findings}. The dropout rate of the model was set to $0.1$ while the size of the embedding layer was set to $512$. Similarly to~\citet{vaswani2017attention}, we utilize byte-pair encoding and construct a common vocabulary of $32,000$ tokens. The model was trained using Adam~\citep{DBLP:journals/corr/KingmaB14} as the base optimizer, natively available in pytorch. For all other optimizers, we modify the codes described in section ~\ref{sec:Codes} to change the base optimizer from SGD to Adam. The hyper-parameters common to Adam, E-Adam, ASO-Adam, SAM-Adam, and LPF-Adam are provided in the Table~\ref{tab:exp_mt1}.  The hyper-parameter $\rho$ in SAM-Adam was tuned by performing a grid search over the set $\{0.01, 0.1, 0.25, 0.5, 1.0\}$ and $\rho=0.01$ was found to be the best performing hyper-parameter. For E-Adam, we perform grid search over $\gamma_0 = \{0.1, 0.5, 0.05\}$, $\gamma_1=\{0.0001, 0.0005,0.005\}$, and $\eta=\{0.05,0.01,0.1\}$. $(g_{0}, g_{1}, \eta) = (0.5, 0.0001, 0.1)$ were found to be the best performing hyper-parameters. The SGLD iterations ($M$) were set to $5$ and the learning rate of the base optimizer (Adam) was also increased by a factor of $M$. The smooth out parameter $a$ in ASO-Adam was tuned by performing grid search over the set $\{ 0.001, 0.005, 0.009, 0.05 \}$ and $a = 0.009$ was found to be the best performing hyper-parameter. Finally, the LPF-Adam parameters were tuned by performing grid search over the set $\gamma = \{ 0.001, 0.0001, 0.0005 \}$ and $\gamma=0.0005$ was found to be the best performing hyper-parameters. The value of $M$ was set to $8$ and $\alpha$ was set to $15$. 

\begin{table}[!ht]
    \vspace{0.1in}
    \setlength\tabcolsep{2.2pt}
    \renewcommand{\arraystretch}{1.3}
    \centering
    \begin{tabular}{|c|c|c|c|c|c|c|}
        \hline
        Dataset & Model & BS & WD & $(\beta_{1},\beta_{2})$ & Itr  & LR (Policy) \\ \hline
        WMT2014 & Mini-Trans & 256 & $1e^{-4}$ & $(0.9,0.99)$ & 100k & $0.0005$ ($\times$ 0.1 ReduceLROnPlateau) \\ \hline
    \end{tabular}
    \vspace{-0.1in}
    \caption{Training hyper-parameters common to all optimizers used for obtaining results in Table~\ref{tab:nlp}. BS: batch size, WD: weight decay.}
    \label{tab:exp_mt1}
\end{table}

\section{ABLATION STUDIES}
\label{sec:AS}
In order to understand the impact of various hyper-parameters of the LPF-SGD optimizer on its performance, we perform three sets of ablation studies. First, we study the impact of the number of MC iterations ($M$). Specifically, we train WRN16-8, WRN28-10, and Pyramidnet-110 models on CIFAR-10 and CIFAR-100 data sets with various different data augmentation techniques using LPF-SGD optimizer while varying the number of MC iterations (M) over the set $\{ 1, 2, 4, 8\}$. For $M
>8$ the computational cost is $>2\times$SGD which may not be viable for empirical purposes. Table ~\ref{tab:ablation_study1} shows that increasing M leads to the drop in the test error, as expected, though the differences in performance between different values of M are not drastic.

\begin{table}[!ht]
    \centering
    \setlength\tabcolsep{1.7pt}
    \renewcommand{\arraystretch}{1.0}
    \begin{tabular}{|c|c|c|c|c|c|c|c|c|c|} 
        \hline
        & & \multicolumn{4}{c|}{CIFAR-10} & \multicolumn{4}{c|}{CIFAR-100} \\ \cline{3-10}
        \multirow{3}{1.5cm}{Model}  & \multirow{3}{2cm}{Augmentation} & \multicolumn{4}{c|}{LPF-SGD ($M$)}       &  \multicolumn{4}{c|}{LPF-SGD ($M$)} \\ \cline{3-10}
        & & M=1 & M=2 	& M=4 	& M=8 	& M=1 	& M=2 	& M=4 	& M=8 \\\hline
        \multirow{3}{1.5cm}{WRN16-8}    & Basic                   & $4.1_{\pm0.1}$	&	$4.0_{\pm0.1}$	&	$3.8_{\pm0.1}$	& $\mathbf{3.7_{\pm <0.1}}$                 &	$19.3_{\pm0.2}$	&	$19.1_{\pm0.2}$	&	$19.1_{\pm0.1}$ & $\mathbf{18.9_{\pm 0.1}}$ \\ 
        & Basic+Cut                                             & $3.6_{\pm0.0}$	&	$3.5_{\pm<0.1}$	&	$3.4_{\pm0.1}$	& $\mathbf{3.2_{\pm 0.1}}$                  &	$18.9_{\pm0.2}$	&	$18.7_{\pm0.2}$	&	$18.5_{\pm0.2}$ & $\mathbf{18.3_{\pm 0.1}}$ \\ 
        & Basic+AA+Cut                                          & $3.1_{\pm0.1}$	&	$3.2_{\pm0.0}$	&	$\mathbf{3.1_{\pm<0.1}}$	& $\mathbf{3.1_{\pm 0.1}}$      &	$17.9_{\pm0.2}$	&	$17.8_{\pm0.0}$	&	$17.6_{\pm0.2}$ & $\mathbf{17.6_{\pm 0.1}}$ \\ \hline
        \multirow{3}{1.5cm}{WRN28-10}    & Basic                  & $3.8_{\pm0.1}$	&	$3.7_{\pm0.0}$	&	$3.6_{\pm0.1}$	& $\mathbf{3.5_{\pm 0.1}}$                  &	$17.7_{\pm0.1}$	&	$17.7_{\pm0.2}$	&	$17.5_{\pm0.1}$	& $\mathbf{17.4_{\pm 0.1}}$ \\ 
        & Basic+Cut                                             & $3.1_{\pm<0.1}$	&	$3.0_{\pm<0.1}$	&	$3.0_{\pm0.1}$	& $\mathbf{2.7_{\pm <0.1}}$                 &	$17.1_{\pm0.1}$	&	$17.2_{\pm0.2}$	&	$16.9_{\pm0.2}$	& $\mathbf{16.9_{\pm 0.2}}$ \\ 
        & Basic+AA+Cut                                          & $2.5_{\pm0.1}$	&	$2.6_{\pm<0.1}$	&	$2.6_{\pm<0.1}$	& $\mathbf{2.5_{\pm 0.1}}$                  &	$15.9_{\pm0.1}$ &	$16.2_{\pm0.2}$ &	$15.8_{\pm0.2}$ & $\mathbf{15.9_{\pm 0.1}}$ \\ \hline
        \multirow{3}{1.5cm}{PyNet110 ($\alpha=270$)}    & Basic   & $3.7_{\pm0.1}$	&	$3.7_{\pm<0.1}$	&	$3.6_{\pm0.1}$	& $\mathbf{3.4_{\pm 0.1}}$                  &	$18.5_{\pm0.2}$	&	$18.5_{\pm0.3}$	&	$\mathbf{18.1_{\pm0.1}}$	& $18.2_{\pm 0.3}$ \\ 
        & Basic+Cut                                             & $2.9_{\pm0.1}$	&	$2.9_{\pm0.1}$	&	$2.9_{\pm0.1}$	& $\mathbf{2.5_{\pm <0.1}}$                 &	$17.8_{\pm0.4}$	&	$17.8_{\pm0.2}$	&	$\mathbf{17.2_{\pm0.1}}$	& $17.3_{\pm 0.2}$ \\ 
        & Basic+AA+Cut                                          & $2.4_{\pm0.1}$	&	$2.4_{\pm0.1}$	&	$2.5_{\pm<0.1}$	& $\mathbf{2.1_{\pm 0.0}}$                  &	$15.8_{\pm0.1}$	&	$15.7_{\pm0.1}$	&	$15.8_{\pm0.2}$	& $\mathbf{15.6_{\pm 0.1}}$ \\ \hline
    \end{tabular}
    \vspace{-0.1in}
    \caption{Validation error rate with 95\% confidence interval for LPF-SGD optimizer trained with different values of MC iterations (M).}
    \label{tab:ablation_study1}
    \vspace{0.1in}
\end{table}

\begin{table}[!ht]
    \centering
    \renewcommand{\arraystretch}{1.1}
    \begin{tabular}{|c|c|c|c|c|c|} 
        \hline
        & & \multicolumn{2}{c|}{CIFAR-10} & \multicolumn{2}{c|}{CIFAR-100} \\ \cline{3-6}
        \multirow{3}{1.5cm}{Model}  & \multirow{3}{*}{Aug} & \multicolumn{2}{c|}{LPF-SGD ($\gamma$ pol)}       &  \multicolumn{2}{c|}{LPF-SGD ($\gamma$ pol)} \\ \cline{3-6}
        & & Fixed & Cosine & Fixed & Cosine \\\hline
        \multirow{3}{1.5cm}{WRN16-8}                    & Basic             & $3.9_{\pm0.1}$ & $\mathbf{3.7_{\pm<0.1}}$ & $20.1_{\pm0.1}$ & $\mathbf{18.9_{\pm 0.1}}$   \\ 
                                                        & Basic+Cut           & $3.5_{\pm0.1}$ & $\mathbf{3.2_{\pm0.1}}$  & $19.4_{\pm0.2}$ & $\mathbf{18.3_{\pm 0.1}}$   \\ 
                                                        & Basic+AA+Cut         & $3.1_{\pm0.1}$ & $\mathbf{3.1_{\pm0.1}}$  & $19.0_{\pm0.1}$ & $\mathbf{17.6_{\pm 0.1}}$   \\ \hline
        \multirow{3}{1.5cm}{WRN28-10}                     & Basic             & $3.8_{\pm0.1}$ & $\mathbf{3.5_{\pm0.1}}$  & $18.6_{\pm0.1}$ & $\mathbf{17.4_{\pm 0.1}}$   \\ 
                                                        & Basic+Cut           & $2.9_{\pm0.1}$ & $\mathbf{2.7_{\pm<0.1}}$ & $18.0_{\pm0.2}$ & $\mathbf{16.9_{\pm 0.2}}$   \\ 
                                                        & Basic+AA+Cut         & $2.6_{\pm0.1}$ & $\mathbf{2.5_{\pm0.1}}$  & $17.1_{\pm0.1}$ & $\mathbf{15.9_{\pm 0.1}}$   \\ \hline
        \multirow{3}{1.5cm}{PyNet110 ($\alpha=270$)}      & Basic             & $3.6_{\pm0.1}$ & $\mathbf{3.4_{\pm0.1}}$  & $\mathbf{18.0_{\pm0.1}}$ & $18.2_{\pm 0.3}$   \\ 
                                                        & Basic+Cut           & $2.8_{\pm0.1}$ & $\mathbf{2.5_{\pm0.1}}$  & $\mathbf{17.0_{\pm0.1}}$ & $17.3_{\pm 0.2}$   \\ 
                                                        & Basic+AA+Cut         & $2.3_{\pm0.1}$ & $\mathbf{2.1_{\pm<0.1}}$ & $15.8_{\pm0.3}$ & $\mathbf{15.6_{\pm 0.1}}$   \\ \hline
    \end{tabular}
    \vspace{-0.1in}
    \caption{Validation error rate with 95\% confidence interval for LPF-SGD optimizer trained with different policy of $\gamma$ hyper-parameter.}
    \label{tab:ablation_study2}
     \vspace{0.1in}
\end{table}

In the second set of experiments, we study the impact of $\gamma$ policy on the test error rate. Specifically, we consider two policies for setting up $\gamma$ parameter, a fixed value as well as cosine policy given in Equation~\ref{eq:gamma_policy} ($\gamma$ parameter is chosen via hyper-parameter search). Similarly to the previous study, we train WRN16-8, WRN28-10, and Pyramidnet-110 models on CIFAR-10 and CIFAR-100 data sets with various different data augmentation techniques. Table ~\ref{tab:ablation_study2} shows that cosine policy is superior and corresponds to progressively increasing the area of the loss surface that is explored (as we converge to the flat valley).

In the third set of experiments, we study the impact of two different kinds of co-variance matrix ($\Sigma$) on the test error rate. Specifically, we train WRN16-8 and WRN28-10 models on CIFAR-10 and CIFAR-100 data sets with isotropic co-variance (identity) matrix and compare its performance with the parameter dependent anisotropic approach as in LPF-SGD (note that the parameter-dependent anisotropic strategy is equivalent to using isotropic covariance on a balanced network). Table~\ref{tab:iso_aniso} captures the results and reveals that anisotropic strategy is superior to isotropic in terms of the final performance. 

Finally, in order to verify whether the performance gain of the proposed method can be achieved by prolonging the number of training epochs for the mSGD method, we trained WRN-16-8 model and WRN-28-10 model on CIFAR-10 and CIFAR-100 data set for $400$ epochs (instead of the current setting of $200$ epochs). In Table ~\ref{tab:exp2_supp}, we show that even with prolonged training, the mSGD is not able to match the performance of LPF-SGD.

\begin{table}[h]
    \renewcommand{\arraystretch}{1.1}
    \centering
    \begin{tabular}{|c|c|c|c|c|c|}
    \hline
    \multirow{2}{1.5cm}{Model}      & \multirow{2}{*}{Aug} & \multicolumn{2}{c|}{CIFAR-10} & \multicolumn{2}{c|}{CIFAR-100} \\ \cline{3-6}
                                    &           & Iso              & An-Iso     & Iso              & An-Iso       \\ \hline
    \multirow{3}{1.5cm}{WRN 16-8}   & Basic            & $3.8_{\pm 0.0}$  &  $\mathbf{3.7_{\pm <0.1}}$  &  $19.7_{\pm 0.2}$ & $\mathbf{18.9_{\pm 0.1}}$ \\ \cline{2-6}
                                    & Basic+Cut        & $3.3_{\pm 0.0}$  &  $\mathbf{3.2_{\pm 0.1}}$   &  $19.0_{\pm 0.2}$ & $\mathbf{18.3_{\pm 0.1}}$ \\ \cline{2-6}
                                    & Basic+AA+Cut     & $3.3_{\pm 0.1}$  &  $\mathbf{3.1_{\pm 0.1}}$   &  $18.1_{\pm 0.2}$ & $\mathbf{17.6_{\pm 0.1}}$ \\ \hline
    \multirow{3}{1.5cm}{WRN 28-10}  & Basic            & $3.7_{\pm 0.1}$  &  $\mathbf{3.5_{\pm 0.1}}$   &  $18.1_{\pm 0.1}$ & $\mathbf{17.4_{\pm 0.1}}$ \\ \cline{2-6}
                                    & Basic+Cut        & $2.9_{\pm 0.0}$  &  $\mathbf{2.7_{\pm <0.1}}$  &  $17.6_{\pm 0.1}$ & $\mathbf{16.9_{\pm 0.2}}$ \\ \cline{2-6}
                                    & Basic+AA+Cut     & $2.8_{\pm 0.1}$  &  $\mathbf{2.5_{\pm 0.1}}$   &  $16.5_{\pm 0.1}$ & $\mathbf{15.9_{\pm 0.1}}$ \\ \hline
    \end{tabular}
    \vspace{-0.1in}
    \caption{Validation error rate with 95\% confidence interval for LPF-SGD optimizer trained with isotropic and parameter dependent anisotropic covariance matrix.}
    \label{tab:iso_aniso}
    \vspace{0.1in}
\end{table}

\begin{table}[h]
    \centering
    \setlength\tabcolsep{2.2pt}
    \renewcommand{\arraystretch}{1.1}
    \begin{tabular}{|c|c|c|c|c|c|c|c|} 
        \hline
        \multirow{2}{*}{Model} & \multirow{2}{*}{Aug}   &  \multicolumn{3}{c|}{CIFAR-10} & \multicolumn{3}{c|}{CIFAR-100} \\ \cline{3-8}
        & & \begin{tabular}{@{}c@{}}mSGD \\ (200 epochs)\end{tabular} & \begin{tabular}{@{}c@{}}mSGD \\ (400 epochs)\end{tabular} &  \begin{tabular}{@{}c@{}}LPF-SGD \\ (200 epochs)\end{tabular} & \begin{tabular}{@{}c@{}}mSGD \\ (200 epochs)\end{tabular} & \begin{tabular}{@{}c@{}}mSGD \\ (400 epochs)\end{tabular} &  \begin{tabular}{@{}c@{}}LPF-SGD \\ (200 epochs)\end{tabular} \\ \hline
         \multirow{3}{2cm}{WRN16-8}    & B        & $4.2_{\pm 0.2}$   & $4.1_{\pm 0.0}$             & $\mathbf{3.7_{\pm <0.1}}$ & $20.6_{\pm 0.2}$  & $20.4_{\pm 0.3}$  & $\mathbf{18.9_{\pm 0.1}}$ \\ 
                                       & B+C      & $3.9_{\pm 0.1}$   & $3.6_{\pm 0.1}$             & $\mathbf{3.2_{\pm 0.1}}$  & $20.0_{\pm 0.1}$  & $19.8_{\pm 0.1}$  & $\mathbf{18.3_{\pm 0.1}}$ \\  
                                       & B+A+C    & $3.3_{\pm 0.1}$   & $\mathbf{3.1_{\pm 0.1}}$    & $\mathbf{3.1_{\pm 0.1}}$  & $19.3_{\pm 0.2}$  & $19.0_{\pm 0.2}$  & $\mathbf{17.6_{\pm 0.1}}$ \\ \hline      
         \multirow{3}{2cm}{WRN28-10}   & B        & $4.0_{\pm 0.1}$   & $3.7_{\pm 0.1}$             & $\mathbf{3.5_{\pm 0.1}}$  & $19.1_{\pm 0.2}$  & $18.5_{\pm 0.1}$  & $\mathbf{17.4_{\pm 0.1}}$ \\  
                                       & B+C      & $3.1_{\pm <0.1}$  & $2.9_{\pm 0.2}$             & $\mathbf{2.7_{\pm <0.1}}$ & $18.3_{\pm 0.1}$  & $17.6_{\pm 0.1}$  & $\mathbf{16.9_{\pm 0.2}}$ \\ 
                                       & B+A+C    & $2.6_{\pm 0.1}$   & $\mathbf{2.4_{\pm 0.0}}$    & $2.5_{\pm 0.1}$           & $17.3_{\pm 0.2}$  & $17.0_{\pm 0.2}$  & $\mathbf{15.9_{\pm 0.1}}$ \\ \hline  

    \end{tabular}
    \vspace{-0.1in}
    \caption{Mean validation error with 95\% confidence interval. B refers to basic data augmentation, C refers to Cutout augmentation, and A refers to AutoAugmentation.}
    \label{tab:exp2_supp}
\end{table}

\section{ADVERSARIAL ROBUSTNESS}
\label{sec:AR}

In this section, we evaluate the performance of models trained with various flatness based optimizers under different adversarial attacks that cab be found in the standard foolbox library~\citep{rauber2017foolbox}. Specifically, we evaluate the robust error rate (Table~\ref{tab:adversarial}) of WRN16-8 and WRN28-10 models trained using mSGD, E-SGD, ASO, SAM, and LPF-SGD on CIFAR-10 and CIFAR-100 data sets with different data augmentation schemes under simultaneous four different adversarial attacks: L2 Fast Gradient Method~\citep{DBLP:journals/corr/GoodfellowSS14}, Linf Projected Gradient Descent (PGD)~\citep{madry2017towards}, Linf Additive Uniform Noise Attack (AUNA)~\citep{inci2018deepcloak}, and Linf Deep Fool Attack~\citep{moosavi2016deepfool}) with standard step size of $\epsilon=8/255$. 

\begin{table}[!ht]
    \centering
    \setlength\tabcolsep{2pt}
    \renewcommand{\arraystretch}{1.1}
    \begin{tabular}{|c|c|c|c|c|c|c|c|c|c|c|c|} 
        \hline
        & & \multicolumn{5}{c|}{CIFAR-10} & \multicolumn{5}{c|}{CIFAR-100} \\ \cline{3-12}
        Model                     & Aug  & mSGD               & E-SGD & ASO   & SAM                       & LPF-SGD & mSGD               & E-SGD & ASO   & SAM                       & LPF-SGD                  \\ \hline
        \multirow{3}{1cm}{WRN 16-8}    & B         & $10.3_{\pm0.1}$          & $\mathbf{9.0_{\pm0.3}}$   & $10.4_{\pm0.1}$ & $9.9_{\pm0.3}$           & $9.7_{\pm0.2}$            & $29.5_{\pm0.3}$          & $\mathbf{28.6_{\pm0.3}}$  & $29.4_{\pm0.4}$ & $28.9_{\pm0.2}$          & $\mathbf{28.6_{\pm0.3}}$     \\ 
                                        & B+C    & $10.8_{\pm0.2}$          & $\mathbf{9.4_{\pm0.5}}$   & $10.8_{\pm0.3}$ & $9.9_{\pm0.1}$           & $9.9_{\pm0.2}$               & $29.1_{\pm0.1}$          & $28.7_{\pm0.4}$           & $29.3_{\pm0.1}$ & $28.7_{\pm0.4}$          & $\mathbf{28.3_{\pm0.3}}$     \\ 
                                        & B+A+C  & $10.9_{\pm0.3}$          & $9.9_{\pm0.3}$            & $10.2_{\pm0.4}$ & $\mathbf{9.4_{\pm0.1}}$  & $10.2_{\pm0.1}$              & $30.4_{\pm0.2}$          & $29.3_{\pm0.2}$           & $30.0_{\pm0.3}$ & $29.6_{\pm0.2}$          & $\mathbf{28.7_{\pm0.4}}$     \\ \hline
        \multirow{3}{1cm}{WRN 28-10}   & B         & $9.0_{\pm0.2}$           & $9.5_{\pm0.2}$            & $8.6_{\pm0.3}$  & $\mathbf{8.1_{\pm0.1}}$  & $8.6_{\pm0.1}$            & $26.4_{\pm0.2}$          & $\mathbf{25.8_{\pm0.2}}$  & $26.9_{\pm0.1}$ & $\mathbf{25.8_{\pm0.2}}$ & $26.7_{\pm0.1}$              \\ 
                                        & B+C     & $8.3_{\pm0.3}$           & $8.6_{\pm0.1}$            & $8.4_{\pm0.1}$  & $8.4_{\pm0.2}$           & $\mathbf{7.7_{\pm0.4}}$     & $26.2_{\pm0.1}$          & $25.6_{\pm0.2}$           & $26.2_{\pm0.3}$ & $25.8_{\pm0.2}$          & $\mathbf{25.4_{\pm0.3}}$     \\ 
                                        & B+A+C  & $8.3_{\pm0.3}$           & $8.1_{\pm0.1}$            & $8.2_{\pm0.3}$  & $\mathbf{7.9_{\pm0.2}}$  & $\mathbf{7.9_{\pm0.3}}$      & $26.7_{\pm0.2}$          & $\mathbf{25.6_{\pm0.2}}$  & $26.9_{\pm0.6}$ & $25.8_{\pm0.3}$          & $26.0_{\pm0.2}$              \\ \hline
    \end{tabular}
    \caption{Robust error rate with 95\% confidence interval under simultaneous four different adversarial attacks: L2 Fast gradient method (FGM),  Linf Projected gradient descent (PGD) attack, Linf Additive uniform noise attack, and Linf Deep fool attack with standard step size $\epsilon=8 / 255$. B refers to basic data augmentation, C refers to Cutout augmentation, and A refers to AutoAugmentation.}
    \label{tab:adversarial}
\end{table}

\section{THEORETICAL PROOFS}
\label{sec:proofs}

\subsection{Proof for Theorem 1}
Proof in this section is inspired by the analysis in~\citep{duchi2012randomized}.
\begin{lma}\label{lma:1}
Let $l_o(\theta;\xi)$ be $\alpha$ Lipschitz continuous with respect to $l_2$-norm. Let variable $Z$ be distributed according to the distribution $\mu$. Then
\begin{align}\label{eq:lm1_1}
    \norm{\nabla l_\mu(x;\xi)-\nabla l_\mu(y;\xi)}=&\mathbb{E}_{Z\sim\mu}[\nabla l_\mu(x+Z;\xi)-\nabla l_\mu(y+Z;\xi)]\\\notag
    \leq& \alpha\int |\mu(z-x)-\mu(z-y)|dz.
\end{align}
If distribution $\mu$ is rotationally symmetric and non-increasing, the bound is tight and can be attained by the function
\begin{align*}
    l_o(\theta;\xi)=\alpha\frac{\norm{x}^2+\norm{y}^2}{\norm{x-y}}\left|<\frac{x-y}{\norm{x}^2+\norm{y}^2},\theta>-\frac{1}{2}\right|.
\end{align*}
\end{lma}
\begin{proof}
Let $Z$ be the random variable satisfies distribution $\mu$.
\begin{align*}
    &\mathbb{E}_{Z\sim\mu}[\nabla l_\mu(x+Z;\xi)-\nabla l_\mu(y+Z;\xi)]\\
    =&\int \nabla l_\mu(x+z;\xi)\mu(z)dz-\int \nabla l_\mu(y;\xi)\mu(z)dz\\
    =&\int \nabla l_\mu(x;\xi)\mu(z)dz-\int \nabla l_\mu(y;\xi)\mu(z)dz\\
    =&\int_{I_>}\nabla l_o(z)[\mu(z-x)-\mu(z-y)]dz-\int_{I_<}\nabla l_0(z)[\mu(z-y)-\mu(z-x)]dz
\end{align*}
where
\begin{align*}
    I_{>} =& \{z\in\mathbb{R}^d|\mu(z-x)>\mu(z-y)\},\\
    I_{<} =& \{z\in\mathbb{R}^d|\mu(z-x)<\mu(z-y)\}.
\end{align*}
Obviously,
\begin{align*}
    &\norm{\mathbb{E}_{Z\sim\mu}[\nabla l_\mu(x+Z;\xi)-\nabla l_\mu(y+Z;\xi)]}\\
    \leq&\sup_{z\in I_>\cup I_<}\norm{\nabla l_o(z)}\left|\int_{I_>}[\mu(z-x)-\mu(z-y)]dz-\int_{I_<}l(z)[\mu(z-y)-\mu(z-x)]dz\right|\\
    \leq& \alpha\left|\int_{I_>}[\mu(z-x)-\mu(z-y)]dz-\int_{I_<}l(z)[\mu(z-y)-\mu(z-x)]dz\right|\\
    =&\alpha\int |\mu(z-x)-\mu(z-y)|dz.
\end{align*}
We already prove the inequality \ref{eq:lm1_1}. We are going to show that the bound is tight and could be attained. Since $\mu$ is an rotationally symmetric and non-increasing, the set $I_>$ could be rewritten as
\begin{align*}
    I_{>} =& \{z\in\mathbb{R}^d|\mu(z-x)>\mu(z-y)\}\\
    =&\{z\in\mathbb{R}^d|\norm{z-x}^2>\norm{z-y}^2\}\\
    =&\{z\in\mathbb{R}^d|\langle z,x-y\rangle>\frac{1}{2}(\norm{x}^2+\norm{y}^2)\},
\end{align*}
similarly,
\begin{align*}
    I_{<} =\{z\in\mathbb{R}^d|\langle z,x-y\rangle<\frac{1}{2}(\norm{x}^2+\norm{y}^2)\}.
\end{align*}
For given $x,y$, define function $l_o$ as
\begin{align*}
    l_o(\theta;\xi)=\alpha\frac{\norm{x}^2+\norm{y}^2}{\norm{x-y}}\left|<\frac{x-y}{\norm{x}^2+\norm{y}^2},\theta>-\frac{1}{2}\right|.
\end{align*}
Therefore, the gradient of function $f$ is
\begin{equation}
\nabla l_o(\theta;\xi)=\left\{
\begin{aligned}
\alpha\frac{x-y}{\norm{x-y}} &&&\text{if}\langle \theta,x-y\rangle>\frac{1}{2}(\norm{x}^2+\norm{y}^2)\\
-\alpha\frac{x-y}{\norm{x-y}} &&&\text{if}\langle \theta,x-y\rangle<\frac{1}{2}(\norm{x}^2+\norm{y}^2)
\end{aligned}
\right.
\end{equation}
Hence, 
\begin{align*}
    &\norm{\mathbb{E}_{Z\sim\mu}[\nabla l_\mu(x+Z;\xi)-\nabla l_\mu(y+Z;\xi)]}\\
    =&\norm{\int_{I_>}\nabla l_o(z)[\mu(z-x)-\mu(z-y)]dz-\int_{I_<}\nabla l_o(z)[\mu(z-y)-\mu(z-x)]dz}\\
    =&\norm{\int_{I_>}\alpha\frac{x-y}{\norm{x-y}}[\mu(z-x)-\mu(z-y)]dz+\int_{I_<}\alpha\frac{x-y}{\norm{x-y}}[\mu(z-y)-\mu(z-x)]dz}\\
    =&\norm{\alpha\frac{x-y}{\norm{x-y}}\int|\mu(z-x)-\mu(z-y)|dz}\\
    =&\alpha\int|\mu(z-x)-\mu(z-y)|dz\norm{\frac{x-y}{\norm{x-y}}}\\
    =&\alpha\int|\mu(z-x)-\mu(z-y)|dz
\end{align*}
We already show that the equality holds for given function $l_o$. Therefore the bound is tight.
\end{proof}


\begin{repthm}{thm:gau}
Let $\mu$ be the $\mathcal{N}(0,\sigma^2I_{d\times d})$ distribution. Assume the differentiable loss function $l_o(\theta;\xi):\mathbb{R}^d\rightarrow \mathbb{R}$ is $\alpha$-Lipschitz continuous and $\beta$-smooth with respect to $l_2$-norm. The smoothed loss function $l_\mu(\theta;\xi)$ is defined as (\ref{eq:f_mu}). Then the following properties hold:
\begin{itemize}
  \item[i)]$l_\mu$ is $\alpha$-Lipschitz continuous.
  \item[ii)]$l_\mu$ is continuously differentiable; moreover, its gradient is $\min\{\frac{\alpha}{\sigma},\beta\}$-Lipschitz continuous, i.e., $f_\mu$ is $\min\{\frac{\alpha}{\sigma},\beta\}$-smooth.
  \item[iii)] If $l_o$ is convex, $l_o(\theta;\xi)\leq l_\mu(\theta;\xi)\leq l_o(\theta;\xi)+\alpha\sigma\sqrt{d}$.
\end{itemize}
In addition, for each bound i)-iii), there exists a function $l_o$ such that the bound cannot be improved by more than a constant factor.
\end{repthm}

\begin{proof}We are going the prove the properties one by one.
\begin{itemize}
    \item[i)]Since $\nabla l_\mu(\theta;\xi)=\mathbb{E}_{Z\sim\mu}[\nabla l_o(\theta+Z;\xi)]$, we have
    \begin{align*}
        \norm{\nabla l_\mu(\theta;\xi)}=\norm{\mathbb{E}_{Z\sim\mu}[\nabla l_o(\theta+Z;\xi)]}
        \leq\mathbb{E}_{Z\sim\mu}[\norm{\nabla l_o(\theta+Z;\xi)}]\leq \alpha.
    \end{align*}
    Therefore, $l_\mu$ is $\alpha$-Lipschitz continuous. To prove the bound is tight, we define
    $$l_o(\theta;\xi)=\frac{1}{2}v^T\theta,$$
    where $v\in\mathbb{R}^d$ is a scalar. Hence, we have
    $$l_\mu(\theta;\xi)=\mathbb{E}_{Z\sim\mu}[l_o(\theta+Z;\xi)]=\mathbb{E}_{Z\sim\mu}[\frac{1}{2}v^T(\theta-Z)]=\frac{1}{2}v^T\theta=l_o(\theta;\xi).$$
    Both $l_o$ and smoothed $l_\mu$ have the gradient $v$ and $l_\mu$ is exactly $\alpha$-Lipschitz.
    
    \item[ii)] The proof scheme for this part is organized as follow: Firstly we show that $l_\mu$ is $\frac{\alpha}{\sigma}$-smooth and the bound can not be improved by more than a constant factor. Then we show that $l_\mu$ is $\beta$-smooth and the bound can not be improved by more than a constant factor as well. In all, we could draw the conclusion that $l_\mu$ is $\min\{\frac{\alpha}{\sigma},\beta\}$-smooth and the bound is tight.
    
    By Lemma \ref{lma:1}, for $\forall x,y\in\mathbb{R}^n$,
    \begin{align}\label{ieq:norm_I2}
        \norm{\nabla l_\mu(x;\xi)-\nabla l_\mu(y;\xi)}
    \leq \alpha\underbrace{\int |\mu(z-x)-\mu(z-y)|dz}_{I_2}.
    \end{align}
    Denote the integral as $I_2$. We follow a technique used in~\citep{lakshmanan2008decentralized, duchi2012randomized}. Since $\mu(z-x)\geq\mu(z-y)$ is equivalent to $\norm{z-x}\geq\norm{z-y}$,
    \begin{align*}
        I_2=&\int |\mu(z-x)-\mu(z-y)|dz\\
        =&\int_{z:\norm{z-x}\geq\norm{z-y}}[\mu(z-x)-\mu(z-y)]dz+\int_{z:\norm{z-x}\leq\norm{z-y}}[\mu(z-y)-\mu(z-x)]dz\\
        =&2\int_{z:\norm{z-x}\leq\norm{z-y}}[\mu(z-x)-\mu(z-y)]dz\\
        =&2\int_{z:\norm{z-x}\leq\norm{z-y}}\mu(z-x)dz-2\int_{z:\norm{z-x}\leq\norm{z-y}}\mu(z-y)dz.
    \end{align*}
    Denote $u=z-x$ for $\mu(z-x)$ term and $u=z-y$ for $\mu(z-y)$ term, we have
    \begin{align*}
        I_2=&2\int_{z:\norm{u}\leq\norm{u-(x-y)}}\mu(u)dz-2\int_{z:\norm{u}\geq\norm{u-(x-y)}}\mu(u)dz\\
        =&2\mathbb{P}_{Z\sim\mu}(\norm{Z}\leq\norm{Z-(x-y)})-2\mathbb{P}_{Z\sim\mu}(\norm{Z}\geq\norm{Z-(x-y)}).
    \end{align*}
    Obviously,
    \begin{align*}
        &\mathbb{P}_{Z\sim\mu}(\norm{Z}\leq\norm{Z-(x-y)})\\
        =&\mathbb{P}_{Z\sim\mu}(\norm{Z}^2\leq\norm{Z-(x-y)}^2)\\
        =&\mathbb{P}_{Z\sim\mu}(2\langle z, x-y\rangle\leq\norm{x-y}^2)\\
        =&\mathbb{P}_{Z\sim\mu}(2\langle z, \frac{x-y}{\norm{x-y}}\rangle\leq\norm{x-y}),
    \end{align*}
    $\frac{x-y}{\norm{x-y}}$ has norm 1 and $Z\sim \mathcal{N}(0,\sigma^2 I)$ implies $\langle z, \frac{x-y}{\norm{x-y}}\rangle\sim \mathcal{N}(0,\sigma^2I)$. Hence, we have
    \begin{align*}
        &\mathbb{P}_{Z\sim\mu}(\norm{Z}\leq\norm{Z-(x-y)})\\
        =&\mathbb{P}_{Z\sim\mu}(\langle z, \frac{x-y}{\norm{x-y}}\rangle\leq\frac{\norm{x-y}}{2})\\
        =&\int_{-\infty}^{\frac{\norm{x-y}}{2}}\frac{1}{\sqrt{2\pi\sigma^2}}\exp(-\frac{u^2}{2\sigma^2})du.
    \end{align*}
    Similarly,
    \begin{align*}
        &\mathbb{P}_{Z\sim\mu}(\norm{Z}\geq\norm{Z-(x-y)})\\
        =&\mathbb{P}_{Z\sim\mu}(\langle z, \frac{x-y}{\norm{x-y}}\rangle\geq\frac{\norm{x-y}}{2})\\
        =&\int_{\frac{\norm{x-y}}{2}}^{+\infty}\frac{1}{\sqrt{2\pi\sigma^2}}\exp(-\frac{u^2}{2\sigma^2})du.
    \end{align*}
    Therefore,
    \begin{align}
        I_2&=2\int_{-\infty}^{\frac{\norm{x-y}}{2}}\frac{1}{\sqrt{2\pi\sigma^2}}\exp(-\frac{u^2}{2\sigma^2})du-2\int_{\frac{\norm{x-y}}{2}}^{+\infty}\frac{1}{\sqrt{2\pi\sigma^2}}\exp(-\frac{u^2}{2\sigma^2})du\notag\\
        &=2\int_{-\frac{\norm{x-y}}{2}}^{\frac{\norm{x-y}}{2}}\frac{1}{\sqrt{2\pi\sigma^2}}\exp(-\frac{u^2}{2\sigma^2})du\notag\\
        &\leq\frac{\sqrt{2}\norm{x-y}}{\sigma\sqrt{\pi}}\label{eq:I2}
    \end{align}
    In conclusion, combine formula (\ref{ieq:norm_I2}) and (\ref{eq:I2}) we have
    \begin{align*}
        \norm{\nabla l_\mu(x;\xi)-\nabla l_\mu(y;\xi)}
    \leq \alpha\frac{\sqrt{2}\norm{x-y}}{\sigma\sqrt{\pi}}\leq\frac{\alpha}{\sigma}\norm{x-y}.
    \end{align*}
    We finish proving that $l_\mu$ is $\frac{\alpha}{\sigma}$-smooth. We are going to show the bound is tight. For any given $x,y$, define function $l_o$ as
    $$l_o(\theta;\xi)=\alpha\frac{\norm{x}^2+\norm{y}^2}{\norm{x-y}}\left|<\frac{x-y}{\norm{x}^2+\norm{y}^2},\theta>-\frac{1}{2}\right|,$$  
    Uniform Lemma \ref{lma:1} and former proof, we know that
    \begin{align}
        \norm{\nabla l_\mu(x;\xi)-\nabla l_\mu(y;\xi)}
        =&\alpha\int |\mu(z-x)-\mu(z-y)|dz\notag\\
        =&\frac{\sqrt{2}\alpha}{\sigma \sqrt{\pi}}\int_{-\frac{\norm{x-y}}{2}}^{\frac{\norm{x-y}}{2}}\exp(-\frac{u^2}{2\sigma^2})d
    \end{align}
    Because
    $$\frac{\sqrt{2}\alpha}{\sigma\sqrt{\pi}}\exp(-\frac{\norm{x-y}^2}{8\sigma^2})\norm{x-y}
    \leq\frac{\sqrt{2}\alpha}{\sigma \sqrt{\pi}}\int_{-\frac{\norm{x-y}}{2}}^{\frac{\norm{x-y}}{2}}\exp(-\frac{u^2}{2\sigma^2})d\leq\frac{\sqrt{2}\alpha}{\sigma\sqrt{\pi}}\norm{x-y}$$
    Obviously, taking $x,y$ such that $\norm{x-y}\leq2\sqrt{2}\sigma$, 
    $$\frac{\sqrt{2}\alpha}{e\sigma\sqrt{\pi}}\norm{x-y}\leq\norm{\nabla l_\mu(x;\xi)-\nabla l_\mu(y;\xi)}\leq\frac{\sqrt{2}\alpha}{\sigma\sqrt{\pi}}\norm{x-y}$$
    we could conclude the Lipschitz bound for $\nabla l_\mu$ cannot be improved by more than a constant factor.
    
    Then we are going to show smooth objective $l_\mu$ is $\beta$ smooth and the bound is tight.
    \begin{align*}
        \norm{\nabla l_\mu(x;\xi)-\nabla l_\mu(y;\xi)}=&\norm{\nabla\mathbb{E}_{Z\sim\mu}[l_o(x+Z)]-\nabla\mathbb{E}_{Z\sim\mu}[l_o(x+Z)]}\\
        =&\norm{\mathbb{E}_{Z\sim\mu}[\nabla l_o(x+Z)-\nabla l_o(y+Z)]}\\
        =&\norm{\int [\nabla l_o(x+Z)-\nabla l_o(y+Z)]\mu(z)dz}\\
        \leq&\int \norm{\nabla l_o(x+Z)-\nabla l_o(y+Z)}\mu(z)dz\\
        \leq&\int \beta\norm{(x+z)-(y+z)}\mu(z)dz\\
        =&\beta\norm{x-y}\int\mu(z)dz\\
        =&\beta\norm{x-y}
    \end{align*}
    Therefore, $l_\mu$ is $\beta$-smooth. Then we are going to show the bound is tight and cannot be improved. Define $\alpha$ Lipschitz continuous and $\beta$-smooth function $l_o:\mathbb{R}^d\to\mathbb{R}$ as
    \begin{align*}
       l_o(\theta;\xi)=\frac{1}{2}\beta\norm{w}^2\quad\quad \theta\in B(0,\frac{\alpha}{\beta}). 
    \end{align*}
    Hence, we have
    \begin{align*}
        \norm{\nabla l_\mu(x;\xi)-\nabla l_\mu(y;\xi)}=&\norm{\int (\beta x-\beta y)\mu(z)dz}\\
        =&\norm{\beta(x-y)\int \mu(z) dz}\\
        =&\beta\norm{x-y}.
    \end{align*}
    Therefore. $l_\mu$ is exactly $\beta$-smooth.
    
    \item[iii)] By Jensen's inequality, for left hand side:
    \begin{align*}
        l_\mu(\theta;\xi)=\mathbb{E}_{Z\sim\mu}[l_o(\theta+Z;\xi)]\geq l_o(\theta+\mathbb{E}_{Z\sim\mu}[Z];\xi)=l_o(\theta;\xi).
    \end{align*}
    For the tightness proof, defining $l_o(\theta;\xi)=\frac{1}{2}v^T\theta$ for $v\in\mathbb{R}^d$ leads to $l_\mu=l_o$. 
    
    For right hand side:
     \begin{align*}
        l_\mu(\theta;\xi)=&\mathbb{E}_{Z\sim\mu}[l_o(\theta+Z;\xi)]\\
        \leq&l_o(\theta;\xi)+\alpha\mathbb{E}_{Z\sim\mu}[\norm{Z}]\quad\quad (\alpha\text{-Lipchitz continuous})\\
        \leq&l_o(\theta;\xi)+\alpha\sqrt{\mathbb{E}[\norm{Z}^2]}\quad\quad (\frac{\norm{Z}^2}{\sigma^2}\sim\mathcal{X}^2(d))\\
        =&l_o(\theta;\xi)+\alpha\sigma\sqrt{d}.
    \end{align*}
    For the tightness proof, taking $l_o(\theta;\xi)=\alpha\norm{\theta}$.Since $l_\mu(\theta;\xi)\geq c \alpha\sigma\sqrt{d}$ for some constant $c$. Therefore, the bound cannot be improved by more than a constant factor.
\end{itemize}
\end{proof}

\subsection{Brief justification for Theorem \ref{prop:gen}}
\label{sec:genjust}

Note that former works also use the ratio of upper-bounds to approximate the ratio of generalization error. For example, in~\citep{DBLP:journals/corr/ChaudhariCSL17}, they use the ratio of upper-bounds to compare the generalization error of Entropy-SGD and Vanilla-SGD (see their Equation 9). Also, in~\citep{sokolic2017generalization} they use the ratio of upper-bounds (see their Equation 15) to compare the generalization error of the model with and without invariant method. Finally, comparing the properties, including convergence rate or generalization error, of learning algorithms using $\mathcal{O}$ notation is common in the field.

\subsection{Proof for Theorem \ref{prop:gen}}

We first consider the generalization error in the context of the original loss $L$, and then we analyze smoothed loss function $L \circledast K$. The true loss is defined as
\begin{align}\label{eq:L_true}
    L^{true}(\theta):=\mathbb{E}_{\xi\sim D}l(\theta;\xi).
\end{align}
 where $l$ is an arbitrary loss function (i.e., $l_o$ for SGD case and $l_{\mu}$ for LPF-SGD case). Since the distribution $\mathcal{D}$ is unknown, we replace the true loss by the empirical loss given as
 \begin{align}\label{eq:L}
     L^{\mathcal{S}}(\theta):= \frac{1}{m}\sum_{i=1}^ml(\theta;\xi).
 \end{align}
 In order to bound the generalization error $\epsilon_g$, we consider the following stability bound.

\begin{D1}[$\epsilon_s$-uniform stability \citep{hardt2016train}]
Let $\mathcal{S}$ and $\mathcal{S}'$ denote two data sets from input data distribution $\mathcal{D}$ such that $\mathcal{S}$ and $\mathcal{S}'$ differ in at most one example. Algorithm $A$ is $\epsilon_s$-uniformly stable if and only if for all data sets $\mathcal{S}$ and $\mathcal{S}'$ we have
\begin{align}
    \sup_\xi\mathbb{E}[l(A(\mathcal{S});\xi)-l(A(\mathcal{S}');\xi)]\leq\epsilon_s.
\end{align}
\end{D1}

The following theorem, proposed in \citep{hardt2016train}, implies that the generalization error could be bounded using the uniform stability bound. 

\begin{repthm}{thm:gen1}
     If $A$ is an $\epsilon_s$-uniformly stable algorithm, then the generalization error (the gap between the true risk and the empirical risk) of $A$ is upper-bounded by the stability factor $\epsilon_s$:
    \begin{align}
        \epsilon_g:=\mathbb{E}_{\mathcal{S},A}[L^{true}(A(\mathcal{S}))-L^{\mathcal{S}}(A(\mathcal{S}))]\leq\epsilon_s
    \end{align}
\end{repthm}
Denote the original true loss and empirical loss as:
\begin{align*}
    L^{true}_o(\theta):=\mathbb{E}_{\xi\sim D}l_o(\theta;\xi)\quad\text{and}\quad
    L^{\mathcal{S}}_o(\theta):= \frac{1}{m}\sum_{i=1}^ml_o(\theta;\xi).
\end{align*}
Denote the stability gap and generalization error of original loss function as $\epsilon_s^o$ and $\epsilon_g^o$, respectively. Theorem~\ref{thm:gen2} bounds links the stability with Lipschitz factor $\alpha$, smoothing factor $\beta$, and number of iterations $T$ of SGD. Its proof can be found in~\citep{hardt2016train}.

\begin{thm}[Uniform stability of SGD~\citep{hardt2016train}]
\label{thm:gen2}
Assume that $l_o(\theta;\xi)\in[0,1]$ is a $\alpha$-Lipschitz and $\beta$-smooth loss function for every example $\xi$. Suppose that we rum SGD for $T$ steps with monotonically non-increasing step size $\eta_t\leq c/t$. Then SGD is uniformly stable with the stability factor $\epsilon_s^o$ satisfying:
     \begin{align}
         \epsilon_s^o\leq\frac{1+1/\beta c}{n-1}(2c\alpha^2)^{\frac{1}{\beta c+1}}T^{\frac{\beta c}{\beta c+1}}.
     \end{align}
\end{thm}

Now, we have already bound the stability gap $\epsilon_s^{o}$ on original loss. Then we will move onto the stability gap $\epsilon_s^{\mu}$ of loss for Gaussian LPF kernel smoothed loss function. Let $\mu$ be distribution $\mathcal{N}(0,\sigma^2 I)$. By the definition of Gaussian LPF (Definition~\ref{def:lpf}), the true loss and the empirical loss with respect to the Gaussian LPF smoothed function are
\begin{align}
    &L_\mu^{true}(\theta):=(L^{true}_o \circledast K) (\theta)= \int_{-\infty}^{\infty} L^{true}_o(\theta-\tau) \mu(\tau) d\tau=\mathbb{E}_{Z\sim\mu}[L^{true}_o(\theta+Z)],\\
    &L_\mu^{\mathcal{S}}(\theta)\quad:=(L^{\mathcal{S}}_o \circledast K) (\theta) = \int_{-\infty}^{\infty} L^{\mathcal{S}}_o(\theta-\tau) \mu(\tau) d\tau=\mathbb{E}_{Z\sim\mu}[L^{\mathcal{S}}_o(\theta+Z)],
\end{align}
where K is the Gaussian LPF kernel satisfies distribution $\mu$ and Z is a random variable satisfies distribution $\mu$. Since $L^{true}_o(\theta):=\mathbb{E}_{\xi\sim D}l_o(\theta;\xi)$ and $L^{\mathcal{S}}_o(\theta):= \frac{1}{m}\sum_{i=1}^ml_o(\theta;\xi)$, $L_\mu^{true}$ and $L_\mu$ could be rewritten as
\begin{align}
    &L_{\mu}^{true}(\theta)\!=\!\!\!\int_{-\infty}^{\infty} \!\!\!\!\mathbb{E}_{\xi\sim D}[l_o(\theta-\tau;\xi)]\mu(\tau) d\tau\!=\!\mathbb{E}_{\xi\sim D}\left[\!\int_{-\infty}^{\infty}\!\!\!l_o(\theta-\tau;\xi)\mu(\tau) d\tau\right]\!=\!\mathbb{E}_{\xi\sim D}\left[l_\mu(\theta;\xi)\right]\label{eq:Ltrue_mu}\\
    &L_{\mu}^{\mathcal{S}}(\theta)\!=\!\int_{-\infty}^{\infty} \!\!\frac{1}{m}\sum_{i=1}^ml_o(\theta;\xi)\mu(\tau) d\tau\!=\!\frac{1}{m}\sum_{i=1}^m\left[\!\int_{-\infty}^{\infty}\!\!l_o(\theta-\tau;\xi_i)\mu(\tau) d\tau\right]=\frac{1}{m}\sum_{i=1}^ml_\mu(\theta;\xi_i).\label{eq:L_mu}
\end{align}
Compare formulas (\ref{eq:Ltrue_mu}-\ref{eq:L_mu}) with (\ref{eq:L_true}-\ref{eq:L}). We could conclude that the true and empirical Gaussian LPF smoothed loss function ($L^{true}_\mu$ and $L^{\mathcal{S}}_\mu$) is exactly the formula of original true and empirical loss ($L^{true}_o$ and $L^{\mathcal{S}}_o$) by replacing $l_o(\theta;\xi)$ with smoothed $l_\mu(\theta;\xi)$. Since LPF-SGD is exactly performing SGD iteration on Gaussian LPF smoothed loss function instead of original loss, the generalization error and stability gap of LPF-SGD also satisfies Theorem \ref{thm:gen1} and Theorem \ref{thm:gen2} after replacing $l_o$ with $l_\mu$.

In section \ref{subsec:gau} we analyze the change of Lipschitz continuous and smooth properties of the objective function after Gaussian LPF smoothing. Therefore, by Theorem \ref{thm:gau}, $l_\mu$ is $\alpha$-Lipschitz continuous and $\min\{\frac{\alpha}{\sigma},\beta\}$-smooth. Define $\hat{\beta}=\min\{\frac{\alpha}{\sigma},\beta\}$, then we could bound the stability gap for LPF-SGD as
\begin{align*}
    \epsilon_s^\mu\leq\frac{1+1/\hat{\beta} c}{n-1}(2c\alpha^2)^{\frac{1}{\beta c+1}}T^{\frac{\beta c}{\beta c+1}}.
\end{align*}

Combine Theorem \ref{thm:gau} with Theorem \ref{thm:gen2} we could have the following proposition.

\begin{repthm}{prop:gen}[Generalization error (GE) bound of LPF-SGD]
Assume that $l_o(\theta;\xi)\in[0,1]$ is a $\alpha$-Lipschitz and $\beta$-smooth loss function for every example $\xi$. 
    Suppose that we run SGD and LPF-SGD for $T$ steps with non-increasing learning rate $\eta_t\leq c/t$. Denote the generalization error bound (GE bound) of SGD and LPF-SGD as $\hat{\epsilon}_g^o$ and $\hat{\epsilon}_g^{\mu}$, respectively. Then the ratio of GE bound is
    \vspace{-0.05in}
    \begin{equation}
\rho=\frac{\hat{\epsilon}_g^{\mu}}{\hat{\epsilon}_g^o}=\frac{1-p}{1-\hat{p}}\left(\frac{2c\alpha}{T}\right)^{\hat{p}-p}=O(\frac{1}{T^{\hat{p}-p}}),
    \end{equation}
    \vspace{-0.15in}
    
    where $p=\frac{1}{\beta c+1}$, $\hat{p}=\frac{1}{\min\{\frac{\alpha}{\sigma},\beta\}c+1}$.
    
    Finally, the following two properties hold:
    \vspace{-0.1in}
    \begin{itemize}
        \item[i)] If  $\sigma>\frac{\alpha}{\beta}$ and $T\gg2c\alpha^2\left(\frac{1-p}{1-\hat{p}}\right)^{\frac{1}{\hat{p}-p}}$, $\rho\ll 1$.
        \vspace{-0.05in}
        \item[ii)] If $\sigma>\frac{\alpha}{\beta}$ and $T>2c\alpha^2\exp(\frac{2}{1-p})$, increasing $\sigma$ leads to a smaller $\rho$.
    \end{itemize}
    \vspace{-0.1in}
\end{repthm}

\begin{proof}
 For easy notation, denote $\hat{\beta}=\min\{\frac{\alpha}{\sigma},\beta\}$, $\epsilon_S^o$ and $\epsilon_s^\mu$ are stability gaps of SGD and LPF-SGD, respectively. 
 From Theorem \ref{thm:gen2} and based on the facts that $l_o$ is $\alpha$-Lipschitz continuous and $\beta$-smooth and that smoothed objective $l_\mu$ is $\alpha$-Lipschitz continuous and $\min\{\frac{\alpha}{\sigma},\beta\}$-smooth, the upper bounds for the stability gaps are
\begin{align*}
    \epsilon_s^o&\leq\frac{1+1/\beta c}{n-1}(2c\alpha^2)^{\frac{1}{\beta c+1}}T^{\frac{\beta c}{\beta c+1}},\\
    \epsilon_s^\mu&\leq\frac{1+1/\hat{\beta} c}{n-1}(2c\alpha^2)^{\frac{1}{\hat{\beta} c+1}}T^{\frac{\hat{\beta} c}{\hat{\beta} c+1}}.
\end{align*}
Denote $p=\frac{1}{\beta c+1}$, $\hat{p}=\frac{1}{\hat{\beta}c+1}$, the bound could be rewritten as
\begin{align*}
    \epsilon_s^o&\leq\frac{1 }{(n-1)(1-p)}(2c\alpha^2)^{p}T^{1-p},\\
    \epsilon_s^\mu&\leq\frac{1 }{(n-1)(1-\hat{p})}(2c\alpha^2)^{\hat{p}}T^{1-\hat{p}}.
\end{align*}
By Theorem \ref{thm:gen1}, the generalization errors $\epsilon_{g}^o$ and $\epsilon_g^\mu$ are bounded by stability gaps $\epsilon_s^o$ and $\epsilon_s^{\mu}$:
\begin{align*}
    \epsilon_{g}^o&\leq \epsilon_{s}^o\leq\frac{1 }{(n-1)(1-p)}(2c\alpha^2)^{p}T^{1-p},\\
    \epsilon_g^\mu&\leq\epsilon_s^\mu\leq\frac{1 }{(n-1)(1-\hat{p})}(2c\alpha^2)^{\hat{p}}T^{1-\hat{p}}.
\end{align*}
Therefore, the GE bound for SGD and LPF-SGD could be written as
\begin{align*}
    \hat{\epsilon}_{g}^o = \frac{1 }{(n-1)(1-p)}(2c\alpha^2)^{p}T^{1-p} \text{ and }
    \hat{\epsilon}_g^\mu = \frac{1 }{(n-1)(1-\hat{p})}(2c\alpha^2)^{\hat{p}}T^{1-\hat{p}},
\end{align*}
respectively. Then the ratio of GE bound is
\begin{align*}
    \rho=\frac{\hat{\epsilon}_g^{\mu}}{\hat{\epsilon}_g^o}=\frac{1-p}{1-\hat{p}}\left(\frac{2c\alpha}{T}\right)^{\hat{p}-p}=O(\frac{1}{T^{\hat{p}-p}}).
\end{align*}

\begin{itemize}
    \item[i)] When $\sigma>\frac{\alpha}{\beta}$ and $T>2c\alpha^2\left(\frac{1-p}{1-\hat{p}}\right)^{\frac{1}{\hat{p}-p}}$, $\rho=\frac{1-p}{1-\hat{p}}(\frac{2c\alpha^2}{T})^{\hat{p}-p}< 1$. Therefore, if $\sigma>\frac{\alpha}{\beta}$ and $T\gg 2c\alpha^2\left(\frac{1-p}{1-\hat{p}}\right)^{\frac{1}{\hat{p}-p}}$, $\rho\ll1$ and property i) holds.
    \item[ii)] Denote $x:=\hat{p}-p$, the reciprocal of approximated ratio could be re-written as
    \begin{align*}
        \frac{1}{\rho}\approx(1-\frac{\hat{p}-p}{1-p})(\frac{T}{2c\alpha^2})^{\hat{p}-p}=(1-\frac{x}{1-p})(\frac{T}{2c\alpha^2})^{x}
    \end{align*}
    Define function $h(x) =(1-ax)b^x$, where $a=\frac{1}{1-p}$ and $b=\frac{T}{2c\alpha^2}$. Compute the derivative of function $h$:
    \begin{align*}
        h'(x)=(-ax\ln b-a+\ln b)b^x
    \end{align*}
    $$h'(x_0)=0\Longleftrightarrow x_0=\frac{lnb-a}{a\ln b}$$
    When $x\leq\frac{lnb-a}{a\ln b}$, $h'(x)\geq0$. Otherwise $h'(x)<0$. Since $0<p<\hat{p}<1$, obviously the domain of function $h$ is in the interval $[0,1]$. If $\frac{lnb-a}{a\ln b}>1$, the function $h$ is increasing in its domain. Which means that if the difference between $\hat{p}$ and $p$ increase, the reciprocal of approximated ratio of stability gap $\frac{1}{\rho}$ increase, which is equivalent to the approximate ratio $\rho$ of stability gap decrease. Because
    $$\frac{\ln b-a}{a\ln b}>1\Longleftrightarrow \ln b>\frac{a}{1-a}\Longleftrightarrow T>2c\alpha^2e^{-p}.$$
    In all, we could conclude if $T>2c\alpha^2e^{-p}$, $\hat{p}-p$ increase leads to the approximate ratio of stability gap $\rho$ decrease.
    
    What's more, we are going to analysis the relation between Gaussian filter factor $\sigma$ and the difference $\hat{p}-p$. Since
    $$p=\frac{1}{\beta c+1},\quad\hat{p}=\frac{1}{\hat{\beta}c+1},$$
    
    where $\hat{\beta}=\min\{\frac{\alpha}{\sigma},\beta\}$. $\hat{p}-p$ increase is equivalent to $\hat{\beta}$ decrease. When the Gaussian factor $\sigma$ is large enough ($\frac{\alpha}{\sigma}<\beta$), the smoother factor $\hat{\beta}$ for function $l_\mu$ is exactly $\frac{\alpha}{\sigma}$. Moreover, increasing the factor $\sigma$ leads to the decrease of $\hat{\beta}$. 
    
    Due to the analysis above, if $T>2c\alpha^2e^{-p}$ and $\frac{\alpha}{\sigma}<\beta$, increasing $\sigma$ will cause the approximate ratio $\rho$ to decrease and the generalization error will be smaller. We finish the proof for condition ii).
\end{itemize}
\end{proof}

\subsection{ Non-scalar covariance version}\label{subsec:non-qua}
In this section, we analysis the case when the covariance $\Sigma = \gamma*diag(||\theta_{1}||, ||\theta_{2}|| \cdots ||\theta_{k}||)$ for Gaussian kernel K is no longer a scalar diagonal matrix. For easy notation, we denoted $\Sigma=diag(\sigma_1^2,\cdots,\sigma_d^2)$ where $\sigma_i^2=\gamma*||\theta_{i}||$.

\begin{thm}\label{thm:gau2}
Let $\mu$ be the $\mathcal{N}(0,\Sigma)$ distribution, where $\Sigma=diag(\sigma_1^2,\cdots,\sigma_d^2)\in\mathbb{R}^{d\times d}$ is diagonal. Denote $\sigma_-^2=\min\{\sigma_1^2,\cdots,\sigma_d^2\}$. Assume the differentiable loss function $l_o(\theta;\xi):\mathbb{R}^d\rightarrow \mathbb{R}$ is $\alpha$-Lipschitz continuous and $\beta$-smooth with respect to $l_2$-norm. The smoothed loss function $l_\mu(\theta;\xi)$ is defined as (\ref{eq:f_mu}). Then the following properties hold:
\begin{itemize}
  \item[i)]$l_\mu$ is $\alpha$-Lipschitz continuous.
  \item[ii)]$l_\mu$ is continuously differentiable; moreover, its gradient is $\min\{\frac{\alpha}{\sigma_-},\beta\}$-Lipschitz continuous, i.e. $l_\mu$ is $\min\{\frac{\alpha}{\sigma_-},\beta\}$-smooth.
  \item[iii)] If l is convex, $l_\mu(\theta;\xi)=l(\theta;\xi)+\alpha\sqrt{tr(\Sigma)}=l(\theta;\xi)+\alpha\sqrt{\sum_{i=1}^d\sigma_i^2}$.
\end{itemize}
In addition, for bound i) and iii), there exists a function $l$ such that the bound cannot be improved by more than a constant factor.
\end{thm}
\begin{proof}
We are going to prove the properties one by one.
\begin{itemize}
\item[i)]The proof for properties i) is exactly the same as Theorem \ref{thm:gau}.
\item[ii)]As is shown in the proof for Theorem \ref{thm:gau}, firstly, we need to first address that $l_\mu$ is $\frac{\alpha}{\sigma}$-smooth. then show that $l_\mu$ is $\beta$-smooth. Since, the proof for second part remains the same as what in Theorem \ref{thm:gau}. We will focus on demonstrating $l_\mu$ is $\frac{\alpha}{\sigma}$-smooth.

By Lemma \ref{lma:1}, for $\forall x,y\in\mathbb{R}^n$,
    \begin{align}\label{ieq:norm_I2_2}
        \norm{\nabla l_\mu(x;\xi)-\nabla l_\mu(y;\xi)}
    \leq \alpha\underbrace{\int |\mu(z-x)-\mu(z-y)|dz}_{I_2}.
    \end{align}
    Denoted the integral as $I_2$. We follow the technique in~\citep{lakshmanan2008decentralized} and~\citep{duchi2012randomized}. Since $\mu(z-x)\geq\mu(z-y)$ is equivalent to $\norm{z-x}\geq\norm{z-y}$,
    \begin{align*}
        I_2=&\int |\mu(z-x)-\mu(z-y)|dz\\
        =&\int_{z:\norm{z-x}\geq\norm{z-y}}[\mu(z-x)-\mu(z-y)]dz+\int_{z:\norm{z-x}\leq\norm{z-y}}[\mu(z-y)-\mu(z-x)]dz\\
        =&2\int_{z:\norm{z-x}\leq\norm{z-y}}[\mu(z-x)-\mu(z-y)]dz\\
        =&2\int_{z:\norm{z-x}\leq\norm{z-y}}\mu(z-x)dz-2\int_{z:\norm{z-x}\leq\norm{z-y}}\mu(z-y)dz.
    \end{align*}
    Denote $u=z-x$ for $\mu(z-x)$ term and $u=z-y$ for $\mu(z-y)$ term, we have
    \begin{align*}
        I_2=&2\int_{z:\norm{u}\leq\norm{u-(x-y)}}\mu(u)dz-2\int_{z:\norm{u}\geq\norm{u-(x-y)}}\mu(u)dz\\
        =&2\mathbb{P}_{Z\sim\mu}(\norm{Z}\leq\norm{Z-(x-y)})-2\mathbb{P}_{Z\sim\mu}(\norm{Z}\geq\norm{Z-(x-y)}).
    \end{align*}
    Obviously,
    \begin{align*}
        &\mathbb{P}_{Z\sim\mu}(\norm{Z}\leq\norm{Z-(x-y)})\\
        =&\mathbb{P}_{Z\sim\mu}(\norm{Z}^2\leq\norm{Z-(x-y)}^2)\\
        =&\mathbb{P}_{Z\sim\mu}(2\langle z, x-y\rangle\leq\norm{x-y}^2)\\
        =&\mathbb{P}_{Z\sim\mu}(2\langle z, \frac{x-y}{\norm{x-y}}\rangle\leq\norm{x-y}),
    \end{align*}
    Denote $p=\frac{x-y}{\norm{x-y}}\in\mathbb{R}^{d\times d}$, $\frac{x-y}{\norm{x-y}}$ has norm 1 implies $\sum_{i=1}^dp_i^2=1$. Since $Z\sim \mathcal{N}(0,\Sigma)$ and $\Sigma=diag(\sigma_1^2,\cdots,\sigma_d^2)$, each element in vector $Z$ satisfies $z_i\sim \mathcal{N}(0,\sigma_i^2)$. Hence, we have
    $$\langle z, \frac{x-y}{\norm{x-y}}\rangle=\sum_{i=1}^d p_i z_i\sim \mathcal{N}(0,\sum_{i=1}^dp_i^2\sigma_i^2).$$
    Denote $\sigma^2=\sum_{i=1}^dp_i^2\sigma_i^2$, $\sigma_+^2=\max\{\sigma_1^2,\cdots,\sigma_d^2\}$ and $\sigma_-^2=\min\{\sigma_1^2,\cdots,\sigma_d^2\}$. Because $\sum_{i=1}^dp_i^2=1$, it is easy to know
    \begin{align}
        \langle z, \frac{x-y}{\norm{x-y}}\rangle\sim \mathcal{N}(0,\sigma^2),\quad\text{where } \sigma_-^2\leq\sigma^2\leq\sigma_+^2.
    \end{align}
    Hence, we have
    \begin{align*}
        &\mathbb{P}_{Z\sim\mu}(\norm{Z}\leq\norm{Z-(x-y)})\\
        =&\mathbb{P}_{Z\sim\mu}(\langle z, \frac{x-y}{\norm{x-y}}\rangle\leq\frac{\norm{x-y}}{2})\\
        =&\int_{-\infty}^{\frac{\norm{x-y}}{2}}\frac{1}{\sqrt{2\pi\sigma^2}}\exp(-\frac{u^2}{2\sigma^2})du.
    \end{align*}
    Similarly,
    \begin{align*}
        &\mathbb{P}_{Z\sim\mu}(\norm{Z}\geq\norm{Z-(x-y)})\\
        =&\mathbb{P}_{Z\sim\mu}(\langle z, \frac{x-y}{\norm{x-y}}\rangle\geq\frac{\norm{x-y}}{2})\\
        =&\int_{\frac{\norm{x-y}}{2}}^{+\infty}\frac{1}{\sqrt{2\pi\sigma^2}}\exp(-\frac{u^2}{2\sigma^2})du.
    \end{align*}
    Therefore,
    \begin{align}
        I_2&=2\int_{-\infty}^{\frac{\norm{x-y}}{2}}\frac{1}{\sqrt{2\pi\sigma^2}}\exp(-\frac{u^2}{2\sigma^2})du-2\int_{\frac{\norm{x-y}}{2}}^{+\infty}\frac{1}{\sqrt{2\pi\sigma^2}}\exp(-\frac{u^2}{2\sigma^2})du\notag\\
        &=2\int_{-\frac{\norm{x-y}}{2}}^{\frac{\norm{x-y}}{2}}\frac{1}{\sqrt{2\pi\sigma^2}}\exp(-\frac{u^2}{2\sigma^2})du\notag\\
        &\leq\frac{\sqrt{2}\norm{x-y}}{\sigma\sqrt{\pi}}\leq\frac{\sqrt{2}\norm{x-y}}{\sigma_-\sqrt{\pi}}\label{eq:I2_2}
    \end{align}
    In conclusion, combine formula (\ref{ieq:norm_I2_2}) and (\ref{eq:I2_2}) we have
    \begin{align*}
        \norm{\nabla l_\mu(x;\xi)-\nabla l_\mu(y;\xi)}
    \leq \alpha\frac{\sqrt{2}\norm{x-y}}{\sigma_-\sqrt{\pi}}\leq\frac{\alpha}{\sigma_-}\norm{x-y}.
    \end{align*}
    We finish proving that $l_\mu$ is $\frac{\alpha}{\sigma_-}$-smooth. Since covariance matrix $\Sigma$ for distribution $\mu$ is no longer a scalar matix and $\mu$ is not rotationally symmetric, the bound can no longer be achieved.
    
    \item[iii)] By Jensen's inequality, for left hand side:
    \begin{align*}
        l_\mu(\theta;\xi)=\mathbb{E}_{Z\sim\mu}[l_o(\theta+Z;\xi)]\geq l_o(\theta+\mathbb{E}_{Z\sim\mu}[Z];\xi)=l_o(\theta;\xi).
    \end{align*}
    For the tightness proof, defining $l_o(\theta;\xi)=\frac{1}{2}v^T\theta$ for $v\in\mathbb{R}^d$ leads to $l_\mu=l_o$. 
    
    For right hand side:
     \begin{align*}
        l_\mu(\theta;\xi)=&\mathbb{E}_{Z\sim\mu}[l_o(\theta+Z;\xi)]\\
        \leq&l_o(\theta;\xi)+\alpha\mathbb{E}_{Z\sim\mu}[\norm{Z}]\quad\quad (\alpha\text{-Lipchitz continuous})\\
        \leq&l_o(\theta;\xi)+\alpha\sqrt{\mathbb{E}[\norm{Z}^2]}.
    \end{align*}
    Letting $C^TC=\Sigma$ and $V\sim \mathcal{N}(0,I)$, because $Z\sim \mathcal{N}(0,\Sigma)$, we have
    \begin{align*}
        \mathbb{E}[\norm{Z}^2]=\mathbb{E}[\norm{CV}^2]=\mathbb{E}[V^TC^TCV]=tr(C^TC\mathbb{E}[V^TV])=tr(\Sigma).
    \end{align*}
    Therefore,
    \begin{align*}
        l_\mu(\theta;\xi)=l_o(\theta;\xi)+\alpha\sqrt{tr(\Sigma)}=l_o(\theta;\xi)+\alpha\sqrt{\sum_{i=1}^d\sigma_i^2}.
    \end{align*}
    For the tightness proof, taking $l_o(\theta;\xi)=\alpha\norm{\theta}$. Since $l_\mu(\theta;\xi)\geq c \alpha \sqrt{tr(\Sigma)}$ for some constant $c$. Therefore, the bound cannot be improved by more than a constant factor.
\end{itemize}
\end{proof}

\begin{prop}\label{thm:gen2_2}
    Let $\mu$ be the $\mathcal{N}(0,\Sigma)$ distribution, where $\Sigma=diag(\sigma_1^2,\cdots,\sigma_d^2)\in\mathbb{R}^{d\times d}$ is diagonal. Denote $\sigma_-^2=\norm{\Sigma}_{\infty}=\min\{\sigma_1^2,\cdots,\sigma_d^2\}$. Assume loss function $l_o(\theta;\xi)\!:\!\mathbb{R}^d\!\rightarrow \!\mathbb{R}$ is $\alpha$-Lipschitz and $\beta$-smooth. The smoothed loss function $l_\mu$ is defined as (\ref{eq:f_mu}). Suppose we execute SGD and LPF-SGD for T steps with non-increasing learning rate $\eta_t\leq c/t$. Denote the generalization error bound (GE bound) of SGD and LPF-SGD as $\hat{\epsilon}_g^o$ and $\hat{\epsilon}_g^{\mu}$, respectively. Then the ratio of GE bound is
    \vspace{-0.05in}
    \begin{equation}
\rho=\frac{\hat{\epsilon}_g^{\mu}}{\hat{\epsilon}_g^o}=\frac{1-p}{1-\hat{p}}\left(\frac{2c\alpha}{T}\right)^{\hat{p}-p}=O(\frac{1}{T^{\hat{p}-p}}),
    \end{equation}
    \vspace{-0.15in}
    
    where $p=\frac{1}{\beta c+1}$, $\hat{p}=\frac{1}{\min\{\frac{\alpha}{\sigma},\beta\}c+1}$.
    
    Finally, the following two properties hold:
    \vspace{-0.1in}
    \begin{itemize}
        \item[i)] If  $\sigma_->\frac{\alpha}{\beta}$ and $T\gg2c\alpha^2\left(\frac{1-p}{1-\hat{p}}\right)^{\frac{1}{\hat{p}-p}}$, $\rho\ll 1$.
        \vspace{-0.05in}
        \item[ii)] If $\sigma_->\frac{\alpha}{\beta}$ and $T>2c\alpha^2e^{-p}$, increasing $\sigma_-$ leads to a smaller $\rho$.
    \end{itemize}
    \vspace{-0.1in}
\end{prop}
\begin{proof}
 By Theorem \ref{thm:gau2}, the smoothed loss function $l_\mu$ is $\alpha$-Lipschitz continuous and $\min\{\frac{\alpha}{\sigma_-},\beta\}$-smooth. This gives as equivalency to Theorem \ref{thm:gau} after substituting $\sigma_-$ for $\sigma$. Therefore, 
 proof of Theorem~\ref{thm:gen2_2} is exactly the same as that of Proposition~\ref{prop:gen} after performing this substitution and therefore will be omitted.

\end{proof}



\end{document}